\documentclass[letterpaper]{article} 
\usepackage{aaai24}  
\usepackage{times}  
\usepackage{helvet}  
\usepackage{courier}  
\usepackage[hyphens]{url}  
\usepackage{graphicx} 
\usepackage{natbib}  
\usepackage{caption} 
\usepackage{bibentry}
\usepackage{etoolbox}
\providetoggle{beginFlag}
\settoggle{beginFlag}{true}
\usepackage[ruled,vlined,linesnumbered]{algorithm2e}

\usepackage{newfloat}
\usepackage{listings}
\DeclareCaptionStyle{ruled}{labelfont=normalfont,labelsep=colon,strut=off} 
\usepackage[T1]{fontenc}    
\usepackage{url}            
\usepackage{booktabs}       
\usepackage{amsfonts}       
\usepackage{amsmath}

\usepackage{amssymb}
\usepackage{graphicx}
\usepackage{txfonts}
\usepackage{lipsum}
\usepackage{threeparttable}
\usepackage{nicefrac}       
\usepackage{microtype}      
\usepackage[table]{xcolor}         
\usepackage{verbatim}
\usepackage{multirow}
\usepackage{diagbox}

\usepackage{siunitx}
\usepackage{bbding}
\usepackage{float}
\usepackage{etoolbox}
\usepackage {subcaption}
\makeatletter
\patchcmd{\@algocf@start}
{-1.5em}
{0pt}
{}{}
\makeatother
\def\HiLi{\leavevmode\rlap{\hbox to \hsize{\color{gray!30}\leaders\hrule height .6\baselineskip depth .5ex\hfill}}}
\usepackage{array}
\newcolumntype{L}[1]{>{\raggedright\let\newline\\\arraybackslash\hspace{0pt}}m{#1}}
\newcolumntype{C}[1]{>{\centering\let\newline\\\arraybackslash\hspace{0pt}}m{#1}}
\newcolumntype{R}[1]{>{\raggedleft\let\newline\\\arraybackslash\hspace{0pt}}m{#1}}

\usepackage{algcompatible}

\definecolor{light-gray}{gray}{0.8}
\usepackage[utf8]{inputenc} 
\usepackage{booktabs}       
\usepackage{amsfonts}       
\usepackage{amsmath}
\usepackage{amssymb}
\usepackage{subcaption}
\usepackage{dutchcal}
\usepackage{amsthm}
\usepackage{txfonts}
\usepackage{pifont}
\usepackage{lipsum}
\usepackage{xfrac}       
\usepackage{microtype}      
\usepackage{xcolor}         
\usepackage{multirow}
\usepackage{siunitx}
\usepackage{bbding}
\usepackage{array}

\usepackage{diagbox}
\usepackage[switch]{lineno}
\newtheorem{proposition}{Proposition}

\newtheorem{definition}{Definition}

\newtheorem{theorem}{Theorem}

\usepackage[capitalise]{cleveref}

\newcommand\Rset{{\mathbb{R}}}
\newcommand\Nset{{\mathbb{N}}}
\newcommand\probm{{\mathbb{P}}}
\newcommand\expv{{\mathbb{E}}}
\newcommand\calR{{\mathcal{R}}}

\newcommand\calL{{\mathcal{L}}}

\title{Robustness Verification of Deep Reinforcement Learning Based Control Systems using Reward Martingales}

\author {
    Dapeng Zhi\textsuperscript{\rm 1},
    Peixin Wang\textsuperscript{\rm 2}\equalcontrib,
    Cheng Chen\textsuperscript{\rm 1},
    Min Zhang\textsuperscript{\rm 1}\equalcontrib
}
\affiliations {
    \textsuperscript{\rm 1}Shanghai Key Laboratory for Trustworthy Computing, East China Normal University  \\
    \textsuperscript{\rm 2}University of Oxford\\
    \texttt{zhi.dapeng@163.com, peixin.wang@cs.ox.ac.uk, \{chchen,zhangmin\}@sei.ecnu.edu.cn}
}

\begin{document}

\maketitle

\begin{abstract}
Deep Reinforcement Learning (DRL) has gained prominence as an effective approach for control systems. However, its practical deployment is impeded by state perturbations that can severely impact system performance. Addressing this critical challenge requires robustness verification about system performance, which involves tackling two quantitative questions: (i) how to establish guaranteed bounds for expected cumulative rewards, and (ii) how to determine tail bounds for cumulative rewards.
In this work, we present the first approach for robustness verification of DRL-based control systems by introducing   \emph{reward martingales}, which offer a rigorous mathematical foundation to characterize the impact of state perturbations on system performance in terms of cumulative rewards. Our verified results provide provably quantitative certificates for the two questions.
We then show that reward martingales can be implemented and trained via neural networks, against different types of control policies. Experimental results demonstrate that our certified bounds tightly enclose simulation outcomes on various DRL-based control systems, indicating the effectiveness and generality of the proposed approach. 
\end{abstract}

\vspace{-3mm}
\section{Introduction}

Deep Reinforcement Learning (DRL) is gaining widespread adoption in various control systems, including safety-critical ones like power systems~\citep{power_system_2,power_system_3} and traffic signal controllers~\cite{light_control_1,light_control_3}.
As these systems collect state information via sensors,  uncertainties inevitably originate from sensor errors, equipment inaccuracy, or even adversarial attacks~\cite{SAMDP,power_system_1,vulnerability_DRL}. 
In real-world scenarios, the robustness guarantee of their performance is of utmost importance when they are subjected to reasonable environmental perturbations and adversarial attacks. Failing to do so could lead to critical errors and a significant decline in performance, which may cause fatal consequences in safety-critical applications. 

A DRL-based control system's robustness is usually reflected in its performance variation, i.e., the cumulative rewards,  when the system is perturbed \cite{CORL}. The robustness verification refers to answering two quantitative questions: (i) how to establish guaranteed bounds for expected cumulative rewards, and (ii) how to determine tail bounds for cumulative rewards. However, the verification is a very challenging task. Firstly, DRL-based control systems are complex cyber-physical systems, making formal verification difficult  \citep{deshmukh2019formal}. Secondly, the inclusion of opaque AI models like Deep Neural Networks (DNNs) adds complexity to the problem \citep{larsen2022formal}. Thirdly, performance is measured statistically rather than by analytical calculations, lacking theoretical guarantees.

In this work, we propose a novel approach for formally verifying the robustness of DRL-based control systems.
By leveraging the concept of \emph{martingales} from probabilistic programming \citep{term_martingales,cost_martingales}, we establish provable upper and lower bounds for the expected cumulative rewards of DRL-based control systems under state perturbations. Specifically, we define  \emph{upper reward supermartingales (URS)} and \emph{lower reward submartingales (LRS)} and prove they provide theoretical guarantees in the system's certified reward range. Moreover, we extend our analysis to encompass tail bounds of rewards, utilizing a combination of martingales and Hoeffding's inequality~\cite{hoeffding1994probability}.
This refined approach can derive upper bounds for tail probabilities that show how system performance deviates from some predefined threshold, offering a more comprehensive understanding of the system's robustness.

We further show that reward martingales can be efficiently implemented and trained as  DNNs, as ranking martingales are trained  \cite{stab_martingales}, against various control policies. Given a DRL-based control system, we define a corresponding loss function and train a DNN repeatedly until the DNN satisfies the conditions of being a reward martingale or timeout. 
We identify that computing expected rewards is the difficult part in checking whether a trained DNN is a reward martingale and it varies in the training approaches of 
policies. If policies are implemented by DNNs on infinite and continuous state space, we take advantage of the over-approximation-based method \cite{stab_martingales}. We also propose an analytical method to compute expected values precisely when policies are trained on discretized abstract state space in recently emerging approaches \cite{trainify,li2022neural,drews2020proving}.

We intensively evaluate the effectiveness of our approach on four classic control problems, namely, MountainCar, CartPole, B1, and B2. Through rigorous quantitative robustness verification using our proposed method, we assess the performance of the corresponding DRL-based control systems. To demonstrate our approach's  effectiveness, we compare the verified lower and upper bounds, and tail bounds with the performance achieved through simulations under the same settings. Encouragingly, our experimental results demonstrate that our verified bounds tightly enclose the simulation outcomes.

In summary, this work makes three major contributions:
\begin{itemize}
		
	\item We introduce reward martingales and prove that they analytically  characterize both reward bounds and tail bounds to  the performance of perturbed DRL-based control systems, rendering us the first robustness verification approach to those systems.

    \item We show that	reward martingales can be represented and efficiently trained in the form of deep neural networks and propose corresponding validation approaches for policies trained by two different approaches. 
	
    \item We intensively evaluate our approach on four classic control problems with  control policies under two different training approaches, demonstrating the effectiveness and generality of the proposed approach.
\end{itemize}

\section{Related Work}

\paragraph{Qualitative Verification of DRL-based Control Systems} Formal verification of DRL-based control systems has received increasing attention for safety assurance in recent years. \citet{trainify} proposed a CEGAR-driven training
and verification framework that guarantees that the trained systems  satisfy the properties predefined in ACTL formulas.
\citet{verify_DRL_1} developed formal models of controllers executing under uncertainty and proposed new verification techniques based on abstract interpretation. 
\citet{verify_DRL_2} provided a new formulation for the safety properties to ensure that the agent always makes rational decisions. 
However, these works mainly focus on qualitative verification for specific properties but lack quantitative guarantees.

\paragraph{Robust Training of DRL Systems} Several attempts are made to improve DRL systems' robustness by means of formal verification~\cite{RADIAL,policy_smoothing}. 
For instance, \citet{CORL}  proposed to compute guaranteed lower bounds on state-action values to determine the optimal action under a worst-case deviation in input space.
\citet{RADIAL} designed adversarial loss functions by leveraging existing formal verification bounds w.r.t. neural network robustness.
\citet{SAMDP} studied the fundamental properties of state-adversarial Markov decision processes and developed a theoretically principled policy regularization.
These approaches focused on robust training rather than verification, and they have to rely on simulation to demonstrate the effectiveness of their approaches in robustness improvement.

\paragraph{Quantitative Verification in Stochastic Control Systems via Martingales} Robustness verification of  control systems is essentially a  quantitative verification problem, which provides certified guarantees to systems' quantitative properties such as stabilization time. 
Some studies emerged in this direction. \citet{stab_martingales} considered the problem of formally verifying almost-sure (a.s.) asymptotic stability in discrete-time nonlinear stochastic control systems and presented an approach for general nonlinear stochastic control problems with two aspects: using ranking supermartingales (RSMs) to certify a.s. asymptotic stability and presenting a method for learning neural network RSMs. 
\citet{reach_avoid_martingale} studied the problem of learning controllers for discrete-time non-linear stochastic dynamical systems with formal reach-avoid guarantees by combining and generalizing stability and safety guarantees with a tolerable probability threshold over the infinite time horizon in general Lipschitz continuous systems.  However, these works do not consider system robustness, a non-trivial property of DRL-based control systems.

\section{Preliminaries}\label{sec:prelim}

\paragraph{DRL-Based Control Systems}
In this work, we consider  DRL-based control systems where the control policies are implemented by  neural networks (NNs)  and suppose the networks are trained. 
Formally, a DRL-based control system is a tuple $M=(S,S_0,S_g,A,\pi,f,R)$, where $S\subseteq \Rset^n$ is the set of system states (possibly infinite), $S_0\subseteq S$ (resp. $S_g\subseteq S$) is the set of initial (resp. terminal)  states and $S_0\cap S_g=\emptyset$, $A$ is the set of actions, $\pi:S\rightarrow \Rset$ is the trained policy implemented by a neural network, $f:S\times A\rightarrow S$ is the system dynamics, and $R:S\times A\times S\rightarrow \Rset$ is the reward function. \footnote{Here we focus on deterministic system dynamics and policies, and leave the analysis of probabilistic ones as future work.}

A trained DRL-based control system $M$ is a decision-making system that continuously interacts with the environment. 
At each time step $t\in\Nset_0$, it observes a state $s_t$ and feeds $s_t$ into its planted NN to compute the optimal action $a_t=\pi(s_t)$ that shall be taken. Action $a_t$ is then performed, which transits $s_t$ into the next state $s_{t+1}=f(s_t,a_t)$ via the system dynamics $f$ and earns some reward $r_{t+1} = R(s_t,a_t,s_{t+1})$. Given an  initial state $s_0 \in S_0$, a sequence of states generated during interaction is called an \emph{episode}, denoted as ${\cal e}=\{s_t\}_{t\in\Nset_0}$.

\paragraph{State Perturbations}
During  interaction with environments,
the observed states of systems may be perturbed and actions are computed based on the perturbed states. Formally, an observed state at time $t$ is $\hat{s}_t:=s_t+\delta_t$ where $\delta_t\sim \mu$ and $\mu$ is a probability distribution over $\Rset^n$. 
Due to perturbation, the actual successor state is 
$s_{t+1}:=f(s_t,\hat{a}_t)$ with $\hat{a}_t:=\pi(\hat{s}_t)$ and the reward is $r_{t+1}:=R(s_t,\hat{a}_t,s_{t+1})$. Note that the successor state and reward are calculated according to the actual state and the action on the perturbed state, and this update is common~\cite{SAMDP}. We then denote a DRL-based control system $M$ perturbed by the noise distribution $\mu$ as $M_\mu=(S,S_0,S_g,A,\pi,f,R,\mu)$.

\paragraph{Probability Space} Given an $M_\mu$, for each $s_0\in S_0$, there exists a \emph{probability space} $(\Omega_{s_0},\mathcal{F}_{s_0},\probm_{s_0})$ 
such that $\Omega_{s_0}$ is the set of all episodes that start from $s_0$ by the environmental interaction, $\mathcal{F}_{s_0}$ is a $\sigma$-algebra over $\Omega_{s_0}$ (i.e., a collection of subsets of $\Omega_{s_0}$ that contains the empty set $\emptyset$ and is closed under complementation and countable union), and $\probm_{s_0}:\mathcal{F}\rightarrow [0,1]$ is a probability measure on $\mathcal{F}_{s_0}$. We also denote 
the expectation operator in this probability space by $\expv_{s_0}$.

\paragraph{Termination Time} When states are perturbed, actions may become sub-optimal, which may cause non-termination or premature termination, i.e., the system never or untimely reaches a terminal state. 
Thus, prerequisites for studying the robustness of DRL-based control systems are to guarantee the system is terminating and know its \emph{termination time}.  Intuitively, the termination time of an episode is the number of steps it takes for the episode to reach the terminal set or $\infty$ if it never reaches $S_g$.

Formally, the \emph{termination time} of an $M_\mu$ is a random variable defined on episodes as  $T(\{s_t\}_{t\in\Nset_0}):=\mathrm{min}\ \{t\in\Nset_0\mid s_t\in S_g\}$. We define $\mathrm{min}\ \emptyset=\infty$. 
A control system is \emph{finitely terminating} if
it has finite expected termination time over all episodes, i.e., $\expv_{s_0}[T]<\infty$ for all states $s_0\in S_0$. Besides, a system has the \emph{concentration property} if there exist two constants $a,b>0$ such that for sufficiently large $n\in\Nset$, we have $\probm_{s_0}(T> n)\le a\cdot \mathrm{exp}(-b\cdot n)$ for all states $s_0\in S_0$, i.e. if the probability that the system executes $n$ steps or more decreases exponentially as $n$ grows.

\section{Problem Formulation}

\paragraph{Model Assumptions} Given a DRL-based control system, it is assumed  that its state space $S$ is compact in the Euclidean topology of $\Rset^n$, its system dynamics $f$ and trained policy $\pi$ are Lipschitz continuous. This assumption is common in control theory \cite{reach_avoid_martingale}. Besides, we further assume that once the system state enters $S_g$, it will stop and no more actions will be taken, i.e., for any $s_T\in S_g$, $s_{T+1}=s_T$. The control systems of interest are assumed to be finitely terminating, which can be checked by the stability verification approach  \cite{stab_martingales}. For perturbation, we  assume that a noise distribution $\mu$ either has bounded support
or is a product of independent univariate distributions.

\begin{definition}[Cumulative Rewards]
    Given an  $M_\mu$ with termination time $T$, its  \emph{cumulative reward} is a random variable defined on episodes as $\calR(e):=\sum_{t=0}^{T} r_t$,
	where $e=\{s_t\}_{t=0}^{T}$ is an episode and $r_t$ is the step-wise reward of $e$ that is determined by the reward function $R(\cdot)$ in $M_\mu$ with $r_0\equiv 0$.	
\end{definition}
Intuitively, the cumulative reward is the sum of all step-wise rewards until the system reaches a terminal state. It is a random variable and varies from different episodes.

\paragraph{Robustness Verification Problems of $M_\mu$} Given an $M_\mu$ and an initial state $s_0\in S_0$, we are interested in the following two robustness problems:
\begin{enumerate}
	\item What are the upper and lower bounds for $\expv_{s_0}[\calR]$?
	\item Given a reward $c$, what is the tail bound of $\probm_{s_0}(\calR\ge c)$ (resp. $\probm_{s_0}(\calR\le c)$) 
	if $c$  is larger (resp. smaller) than the upper (resp. lower) bound of $\expv_{s_0}[\calR]$? 
\end{enumerate}
The first problem concerns certified upper and lower bounds of expected cumulative rewards when systems are perturbed. The second problem considers two cases. Provided a cumulative reward $c$ that is greater than the upper bound of the expected cumulative reward $\expv_{s_0}[\calR]$, we are interested in the tail probability that a system can achieve a reward greater than $c$. The dual problem is to compute the tail probability that the system can achieve a cumulative reward lower than $c$, when $c$ is less than the lower bound of $\expv_{s_0}[\calR]$. A higher tail probability implies worse robustness because it indicates a higher probability that the reward gets out of the certified range of expected cumulative rewards.

\section{Reward Martingales and the Fundamentals}
In this section, we present our theoretical results about the two robustness verification problems by introducing the notion of \textit{reward martingales}. It is the foundation of  reducing the robustness verification problems of perturbed DRL-based control systems to the analysis of a stochastic process.

In the following, we fix a perturbed DRL-based control system $M_\mu=(S,S_0,S_g,A,\pi,f,R,\mu)$ and  denote the difference $S\setminus S_g$ by $\overline{S_g}$ for  the set of non-terminal  states in $S$.

To define reward martingales, we first need the notion of pre-expectation of functions. 
Given a function $h(\cdot)$, the pre-expectation $pre_h(\cdot)$ of $h(\cdot)$ is the reward of the current step plus the expected value of $h(\cdot)$ in the next step of the system.

\begin{definition}[Pre-Expectation]
\label{def:pre-ex}
Given an $M_\mu$ and a function $h:S\to \Rset$, the pre-expectation of $h$ is a  function $pre_h:S\to\Rset$, such that:
\begin{align}
pre_h(s) =\left\{\begin{array}{ll}
h(s) & \text{if } s\in S_g\\
r+\expv_{\delta\sim \mu}[h(f(s,\pi(s+\delta)))] & \text{if } s\in \overline{S_g}
\end{array}\right.\notag 
\end{align}
where, $r=R(s,\pi(s),f(s,\pi(s)))$ is the reward of performing action $\pi(s)$ in state $s$. 
\end{definition}

We next define the notion of reward martingales. First, we begin with the definition of URS which can be served as an upper bound for the expected cumulative reward of $M_\mu$.

\begin{definition}[Upper Reward Supermartingales, \textbf{URS}]
\label{def:URS}
Given an $M_\mu$, a function $h: S \rightarrow {\mathbb R}$ is an upper reward supermartingale (URS) of $M_\mu$ if there exist  $K,K'\in\Rset$ such that:
\begin{align}
\forall s\in S_g, K\le h(s)\le K'\tag{Boundedness}\\
\forall s\in \overline{S_g}, pre_h(s)\le h(s)\tag{Decreasing Pre-Expectation}
\end{align}

\end{definition}
Intuitively, the first condition says that the values of the
URS at terminal states should always be bounded, and the second condition specifies that for all non-terminal  states, the
pre-expectation is no more than the value of the URS itself.

Similar to the definition of URS (\Cref{def:URS}), we define LRS as follows and will employ it as a lower bound for the expected cumulative reward of $M_\mu$. 
\begin{definition}[Lower Reward Submartingales, \textbf{LRS}]
\label{def:LRS}
Given an $M_\mu$, a function $h:S\to \Rset$ is a lower reward submartingale (LRS) of $M_\mu$  if there exist  $K,K'\in\Rset$ such that:
\begin{align}
\forall s\in S_g.K\le h(s)\le K'\tag{Boundedness}\\
\forall s\in \overline{S_g}. pre_h(s)\ge h(s)\tag{Increasing Pre-Expectation}
\end{align}
\end{definition}

Compared with the definition of URS, the only difference is that 
the second condition of LRS specifies that the pre-expectation is no less than the value of the LRS itself at all non-terminal  states. We call $K,K'$  the bounds of $h$ if $h$ is a URS or LRS.

\begin{definition}[Difference-boundedness]\label{def:diff-bound}
Given an $M_\mu$ and a function $h:S\to \Rset$, $h$ is  difference-bounded if 
there exists $m\in \Rset$ such that for any state $s\in S$, $|h(f(s,\pi(s)))-h(s)|\le m$.
\end{definition}

Based on the URS and LRS, we have \Cref{thm:bounds}, stating that there must exist upper and lower bounds of the expected cumulative rewards of the perturbed system when we can calculate URS and LRS for the system.  
\begin{theorem}[Bounds for Expected Cumulative Rewards]\label{thm:bounds}
	Suppose an $M_\mu$ has a difference-bounded URS (resp. LRS) $h$ and $K,K'\in\Rset$ are the bounds of $h$. For each state $s_0\in S_0$, we have 
	\begin{align}
	&\expv_{s_0}[\calR]\le h(s_0)-K.\tag{Upper Bound}\\
	 (\textit{resp.}\  &\expv_{s_0}[\calR]\ge h(s_0)-K') \tag{Lower Bound}
	\end{align} 
\end{theorem}

\begin{proof}[Proof Sketch]
 For upper bounds, we define the stochastic process  $\{X_n\}_{n=0}^{\infty}$ as $X_n:= h( \overline{s}_n)$, where $h$ is an URS and $\overline{s}_n$ is a random (vector) variable representing value(s) of the state at the $n$-th step of an episode. Furthermore, we construct the stochastic
process $\{Y_n\}_{n=0}^{\infty}$ such that $Y_n := X_n + \textstyle\sum_{i=0}^{n} r_i$, where $r_i$ is the reward of the $i$-th step.
Let $T$ be termination time of $M_\mu$. We prove that $\{Y_n\}_{n=0}^{\infty}$ satisfies the prerequisites of the Optional Stopping Theorem (OST)~\cite{williams1991}. 
This proof depends on the assumption that $M_\mu$ is finitely terminating and $h$ is difference-bounded. Then by applying OST, we have that $\expv[Y_T]\le \expv[Y_0]$. By the boundedness condition  in \cref{def:URS}, we obtain that $\calR=\sum_{i=0}^T r_i=Y_T-X_T\le Y_T-K$. Finally, we conclude that $\expv_{s_0}[\calR]\le \expv[Y_T]-K\le \expv[X_0]-K=h(\overline{s}_0)-K$. The proof of lower bounds is similar. 
\end{proof}

Lastly in this section, we present our fundamental results about tail bounds of cumulative rewards, with the aid of reward martingales, as formulated by the following theorem. 

\begin{theorem}[Tail Bounds for Cumulative Rewards]\label{thm:tail-bounds}
Suppose that an $M_\mu$ has the concentration property and a difference-bounded URS (\emph{resp.} LRS) $h$ with bounds  $K,K'\in\Rset$. Given an initial state $s_0\in S_0$, if a reward $c>h(s_0)-K$ (resp. $c<h(s_0)-K'$), we have 
\begin{align}
&\probm_{s_0}(\calR\ge c)\le \alpha+ \beta\cdot \mathrm{exp}(-\gamma\cdot c^2 ) \label{eq:tail-1}\\ (\emph{resp.}\ & \probm_{s_0}(\calR\le c)\le \alpha+ \beta\cdot \mathrm{exp}(-\gamma\cdot c^2 ) \label{eq:tail-2}
\end{align} 
where, $\alpha,\beta,\gamma$ are positive constants derived from  $M_\mu$, the concentration property and $h$, respectively.	
\end{theorem}

To prove~\cref{eq:tail-1}, we construct a stochastic process $\{Y_n\}_{n=0}^{\infty}$ such that $Y_n:=h(\overline{s}_n)+\sum_{i=0}^n r_i$ where $h$ is an URS, $\overline{s}_n$ and $r_i$ are defined as those in the proof sketch of \cref{thm:bounds}. 
By \cref{def:URS}, we prove that $\{Y_n\}_{n=0}^{\infty}$ is a supermartingale. 
 Then by the difference-bounded property of $h$ (\cref{def:diff-bound}), we derive the upper bound of $\probm_{s_0}(\calR\ge c)$ by the concentration property of $M_\mu$ and Hoeffding's Inequality on Martingales~\cite{hoeffding1994probability}. \cref{eq:tail-2} is obtained in the same manner.

\section{Neural Network-Based  Reward Martingales}
In this section, we present our method for training and validating neural network martingales including URS and LRS. 
Martingales are not necessarily polynomial functions and can be as complex as deep neural networks, as shown by the pioneering works \cite{RSM_PP,stab_martingales,reach_avoid_martingale,survey_fan}.   
Likewise, we show that reward martingales can be also achieved in the form of DNNs.

Our method consists of two modules that alternate
within a loop: training and validating. In each loop iteration, we train a candidate reward martingale in the form of
a neural network which is then passed to the validation. 
If the validation result is false, we compute a set of counterexamples for future training. This iteration is repeated until a trained candidate is validated or a given timeout is reached. The whole process is sketched in ~\Cref{alg:train_verify}.

\SetAlgoSkip{SkipBeforeAndAfter}
\begin{algorithm}[t]
	\SetKwData{Left}{left}\SetKwData{This}{this}\SetKwData{Up}{up}
	\SetKwFunction{Union}{Union}\SetKwFunction{FindCompress}{FindCompress}
	\SetKwInOut{Input}{input}\SetKwInOut{Output}{output}
	\SetVlineSkip{1mm}
	\SetNlSkip{1mm}

	\caption{Sketch of Reward Martingale Training.}
	\label{alg:train_verify}
	\Input{Perturbed System $M_\mu$,  Granularity $\tau$, Refinement step length $\xi$.}
	\Output{Trained DNN $h$ or UNKNOWN.}
        $\tilde{S} \leftarrow Discretize(S, \tau)$; \\
	$h \leftarrow Initialize(\cdot)$; \\
	\While{timeout not reached}{
		$h$$~\leftarrow~$$ \mathit{Train}(h,\tilde{S})$ \tcp*[r]{\mbox{Train a candidate.}}
            \eIf(\tcp*[h]{Validate.}){\cref{eq:bd-check}$\wedge$ (\cref{formula:verify_LRS}$\vee$ \cref{formula:verify_URS})}
            {Return $h$ \tcp*[r]{Return $h$ once it is valid.}}
            {$\tau \leftarrow \tau-\xi$  \tcp*[r]{Make $\tau$ smaller.}
            $\tilde{S} \leftarrow \mathit{Discretize}(S, \tau)\cup \mathit{Counterexamples}(h)$;\\} 
        {Return UNKNOWN.}
	}
\end{algorithm}

\subsection{Training Candidate Reward Martingales}

The training phase involves two important steps, i.e., training data construction and loss function definition. 

 \vskip 1mm
 \paragraph{Discretizing Training Data} Since the state space $S$ is possibly continuous and infinite, to boost the training 
we choose a finite set of states and then train reward martingale candidates on it.
This can be achieved by discretizing the state space $S$ and constructing a \emph{discretization} $\tilde{S}\subseteq S$ such that for each $s \in S$, there is a $\tilde{s} \in \tilde{S}$ with $||s-\tilde{s}||_1 < \tau $,
where $ \tau>0$ is called the granularity of $\tilde{S}$. As $S$ is compact and thus bounded, this discretization can be computed by simply picking vertices of a grid with sufficiently small cells. 
For the training after validation failure, $\tilde{S}$ is constructed on a set of counterexamples and a new finite set of states triggered by a smaller  $\tau$.
Once the discretization $\tilde{S}$ is obtained, we construct three finite sets $S_{\mathrm{C1}}:=\tilde{S}\cap S_g$,  $S_{\mathrm{C2}}:=\tilde{S}\cap \overline{S_g}$ 
and $S_{\mathrm{C3}}:=\tilde{S}\cap S_0$ 
used for the training process.

\vskip 1mm
\paragraph{Loss Functions of URS}
A candidate URS is initialized as a neural network $h_\theta$ w.r.t. the network parameter $\theta$. Then $h_\theta$ is learned by minimizing the following loss function:
\begin{equation}
\calL_{URS}(\theta):=k_1 \cdot \calL_{C1}(\theta)+k_2 \cdot \calL_{C2}(\theta)+k_3 \cdot \calL_{C3}(\theta)
\end{equation}
where $k_i, \ i=1,2,3$ are the algorithm parameters balancing the loss terms.

The first loss term is defined via the boundedness condition of URS in~\cref{def:URS} as:
\begin{eqnarray}
\calL_{C1}(\theta)=\frac{1}{|S_{\mathrm{C1}}|} \sum\limits_{s\in S_{\mathrm{C1}}} \left(\mathop{\mathrm{max}}\{h_\theta(s)-K',0\}+ \mathop{\mathrm{max}} \{K-h_\theta(s),0\}\right)
\end{eqnarray}
Intuitively, a loss will incur if either $h_\theta(s)$ is not bounded from above by $K'$ or below by $K$ for any $s\in S_{\mathrm{C1}}$.

The second loss term is defined via the decreasing pre-expectation condition of URS in~\cref{def:URS} as: 	
\begin{eqnarray}
\calL_{C2}(\theta)=\frac{1}{|S_{\mathrm{C2}}|}\sum\limits_{s\in S_{\mathrm{C2}}} \left(  \mathop{\mathrm{max}}\{ \sum\limits_{s'\in \mathcal{D}_s}\frac{h_\theta(s')}{N} -h_\theta(s)+\zeta ,0\}  \right),
\end{eqnarray}
where for each $s\in S_{\mathrm{C2}}$, $\mathcal{D}_s$ is the set of its successor states such that $\mathcal{D}_s:=\{s'\mid s'=f(s,\pi(s+\delta_i)), \delta_i\sim \mu, i\in [1,N]  \}$,
$N>0$ is the sample number of successor states. Note that $h_\theta$ is a neural network, so it is intractable to directly compute the closed form of its expectation. Instead, we use the mean of $h_\theta(\cdot)$ at the $N$ successor states to approximate the expected value $\expv_{\delta\sim \mu}[h(f(s,\pi(s+\delta)))]$ for each $s\in S_{\mathrm{C2}}$, 
and $\zeta$  to tighten the decreasing pre-expectation 
condition. Details will be explained in \Cref{theorem:verify_RM}.

The third loss term is the regularization term used to assure the tightness of upper bounds from URS:
\begin{eqnarray}
\calL_{C3}(\theta):=\frac{1}{|S_{\mathrm{C3}}|}\sum\limits_{s\in S_{\mathrm{C3}}} \left( \mathop{\mathrm{max}}\{h_\theta(s) - \overline{u},0\}\right)
\end{eqnarray}
where $\overline{u}$ is a  hyper-parameter enforcing the upper bounds always under some tolerable thresholds, which makes the upper bounds as tight as possible. 

\vskip 1mm
\paragraph{Loss Functions of LRS} Like URS, a candidate neural network LRS $h_\theta$ w.r.t. the parameter $\theta$ is learned by minimizing the loss function
\begin{equation}
\calL_{LRS}(\theta):=k_1 \cdot \calL_{C1'}(\theta)+k_2 \cdot \calL_{C2'}(\theta)+k_3 \cdot \calL_{C3'}(\theta)
\end{equation}
where $k_i, \ i=1,2,3$ are hyperparameters balancing the loss terms.  $\calL_{C1'},\calL_{C2'}$
are defined based on the LRS conditions in \cref{def:LRS}, while $\calL_{C3'}$ is the regularization term used to assure the tightness of lower bounds from LRS:
\begin{linenomath*}
\begin{equation}
\begin{split}
     & \calL_{C1'}(\theta)=\frac{1}{|S_{\mathrm{C1}}|} \sum\limits_{s\in S_{\mathrm{C1}}} \left(\mathop{\mathrm{max}}\{h_\theta(s)-K',0\}+ \mathop{\mathrm{max}} \{K-h_\theta(s),0\}\right), \\
     & \calL_{C2'}(\theta)=\frac{1}{|S_{\mathrm{C2}}|}\sum\limits_{s\in S_{\mathrm{C2}}} \left(  \mathop{\mathrm{max}}\{ h_\theta(s)-  \sum\limits_{s'\in\mathcal{D}_s}\frac{h_\theta(s')}{N}  -\zeta' ,0\}  \right),\\
     & \calL_{C3'}(\theta)=\frac{1}{|S_{\mathrm{C3}}|}\sum\limits_{s\in S_{\mathrm{C3}}} \left( \mathop{\mathrm{max}}\{  \underline{l} -h_\theta(s),0\}\right) ,
\end{split}
\notag \end{equation}
\vskip -3ex 
\end{linenomath*}
\noindent where $\zeta'$ is used to make the increasing pre-expectation condition stricter (see~\cref{theorem:verify_RM}), and $\underline{l}$ is a  hyper-parameter that enforces the lower bounds are as tight as possible by incentivizing not exceeding some tolerable thresholds.

\subsection{Reward Martingale Validation}
A candidate URS (resp. LRS) can be validated if it meets the conditions in ~\Cref{def:URS} (resp. ~\Cref{def:LRS}). 
Because candidate URS and LRS are neural networks, they are Lipschitz continuous~\cite{RuanHK18}. 
Thus, the difference-bounded condition (\cref{def:diff-bound}) is  satisfied straightforwardly. 
For the boundedness condition, we can check  
\begin{align}\label{eq:bd-check}
   \mathop{\mathrm{inf}}\limits_{s\in S_g} h(s)\ge K \text{ and } \mathop{\mathrm{sup}}\limits_{s\in S_g} h(s)\le K'
\end{align}
using the interval bound propagation approach  ~\cite{Interval_Bound_Propagation,interval_pr}. 
When a state $s\in S_g$ violates  \cref{eq:bd-check}, it is treated as a counterexample and added to $\tilde{S}$ for future training.

For the decreasing and increasing pre-expectation conditions in \cref{def:URS,def:LRS}, 
~\Cref{theorem:verify_RM} establishes 
two corresponding sufficient conditions, which are easier to check.

\begin{theorem}
\label{theorem:verify_RM}
Given an $M_{\mu}$ and a function $h:S\rightarrow\Rset$, 
we have 
$pre_h(s)\le h(s)$ for any state $s\in \overline{S_g}$ if the formula below 
\begin{align}\label{formula:verify_URS}
\expv_{\delta\sim \mu}[h(f(\tilde{s},\pi(\tilde{s}+\delta)))]\le h(\tilde{s})-\zeta
\end{align} holds 
for any state $\tilde{s}\in \tilde{S}\cap  \overline{S_g}$, where $\zeta=r_{\mathrm{max}}+ \tau\cdot L_h\cdot (1+L_f\cdot (1+L_\pi))$ with  $L_f,L_\pi,L_h$ being the Lipschitz constants of $f,\pi,h$, 
and $r_{\mathrm{max}}$ being the maximum value of $R$, respectively. \\
Analogously, we have  
$pre_h(s)\ge h(s)$ for any state $s\in \overline{S_g}$ if: 
\begin{align}\label{formula:verify_LRS}
\expv_{\delta\sim \mu}[h(f(\tilde{s},\pi(\tilde{s}+\delta)))]\ge h(\tilde{s})-\zeta'
\end{align} holds 
for any state $\tilde{s}\in \tilde{S}\cap \overline{S_g}$, where $\zeta'=r_{\mathrm{min}}-\tau\cdot L_h\cdot (1+L_f\cdot (1+L_\pi))$ with $r_{\mathrm{min}}$ being the minimum  value of $R$.
\end{theorem}

Similarly, any state violating  \cref{formula:verify_URS} (or \cref{formula:verify_LRS})  is treated as a counterexample and will be added to $\tilde{S}$ for training.

To check the satisfiablility of \cref{formula:verify_URS,formula:verify_LRS} in a state $\tilde{s}$, 
we need to compute the expected value  $\expv_{\delta\sim \mu}[h(f(\tilde{s},\pi(\tilde{s}+\delta)))]$. 
However, it is difficult to compute a closed form because $h$ is provided in the form of neural networks. 
We devise two strategies below depending on the training approaches of control policies.

\vskip 1mm
\paragraph{An Over-Approximation Approach}
For control policies that are trained on compact but infinitely continuous state space, we bound the  expected value $\expv_{\delta\sim \mu}[h(f(\tilde{s},\pi(\tilde{s}+\delta)))]$ via interval arithmetic \cite{Interval_Bound_Propagation,interval_pr} instead of computing it, which is inspired by the work~\citep{stab_martingales,reach_avoid_martingale}.
In particular, given the noise distribution $\mu$ and its support $\mathrm{N} =\{\delta \in \Rset^n \ | \ \mu(\delta) > 0 \}  $, we first partition $\mathrm{N}$ into finitely $k$
cells $\mathrm{cell(N)} = \{N_1,\cdots,N_k\}$, 
and use  $\mathrm{maxvol} = max_{N_i \in \mathrm{cell(N)}} \mathrm{vol}(N_i)$ (resp. $\mathrm{minvol} = min_{N_i \in \mathrm{cell(N)}} \mathrm{vol}(N_i)$) to denote the maximal (resp. minimal) volume with respect to the Lebesgue measure of any cell in the partition, respectively. For the expected value in \cref{formula:verify_URS}, we bound it from above: 
\begin{align}\label{eq:overes-up} 
\expv_{\delta\sim \mu}[h(f(\tilde{s},\pi(\tilde{s}+\delta)))]\le \sum\limits_{N_i\in \mathrm{cell(N)}}  \mathop{\mathrm{maxvol}} \cdot \mathop{\mathrm{sup}}_{\delta}F(\delta) 
\end{align}
\noindent where $F(\delta)=h(f(\tilde{s},\pi(\tilde{s}+\delta)))$.   
Similarly, for the expected value in \cref{formula:verify_LRS}, we bound it from below: 
\begin{align}\label{eq:overes-lo} 
\expv_{\delta\sim \mu}[h(f(\tilde{s},\pi(\tilde{s}+\delta)))]\ge \sum\limits_{N_i\in \mathrm{cell(N)}}  \mathop{\mathrm{minvol}} \cdot \mathop{\mathrm{inf}}_{\delta}F(\delta) 
\end{align}
Both supremum and  infimum  can be  calculated via interval arithmetic.
We refer interested readers to 
\citep{stab_martingales} and \citep{reach_avoid_martingale}  for more details.

By replacing the actual expected values with  their overestimated  upper bounds and lower bounds in \cref{eq:overes-up,eq:overes-lo}, the validation becomes pragmatically feasible without losing the soundness, i.e., a reward martingale candidate that is validated must be  valid. However, due to the overestimation, it may produce false positives and incur unnecessary further training or even timeout.

\vskip 1mm
\paragraph{An Analytic Approach}

Next, we propose an analytical approach for the control policies that are trained on discretized abstract states. In previous work~\cite{trainify,BBreach,li2022neural}, a compact but infinitely continuous 
state space $S$ was discretized to a  finite  set of abstract states, i.e., $S = \bigcup _{i=1}^{L}S^i $ and $\forall i\neq j, \ S^i \cap S^j=\emptyset $. Then a neural network policy $\pi$ was trained on the set of abstract states. After training, each abstract state $S^i$ corresponds to a constant action $a_i$, i.e., $\pi(s)=a_i$ for all $s \in S^i,i=1,\dots,L$. Based on the work, we present our analytic approach below.

For the finiteness of the abstract state space, we can calculate the probabilities of all possible actions for the perturbed state $\hat{s}=s+\delta$ with $\delta\sim \mu$ as follows:
\begin{linenomath*}
	\begin{equation}
        \label{equ:exp_1}
		\begin{split}
			\Delta^i:=p(\hat{a}=a_i)=p(s+\delta\in S^i)
		\end{split} 
	\end{equation}
\end{linenomath*}
for $i=1,\dots,L$. Refer to \cref{fig:dis_s} for an illustrative example.
Since the system dynamics $f$ is deterministic, we can have that the probability of the next state $s'_i$ is equivalent to the probability of the action $a_i$, i.e., $p(f(s,\hat{a})=s'_i)=\Delta^i$. Then we can compute the analytic solution:
\begin{linenomath*}
	\begin{equation}
		\label{equ:exp_2}
		\begin{split}
			\expv_{\delta\sim \mu}[h(f(s,\pi(s+\delta)))]=\textstyle\sum_{i=1}^{L} h(s'_i) \times \Delta^i,
		\end{split} 
	\end{equation}
\end{linenomath*}
where $\Delta^i=\textstyle \int_{S^i} \mu' \ d\hat{s}$
with $S^i$ being the set of all  states whose  action is $a_i$,  and $ \mu'$ being the distribution of the perturbed state $\hat{s}$ that obtained using the value of the actual state $s$ and the noise distribution $\mu$ ~\cite{williams1991}.

\begin{figure}[t]
	\centering
		\captionsetup{skip=2pt}
	\includegraphics[width=\columnwidth]{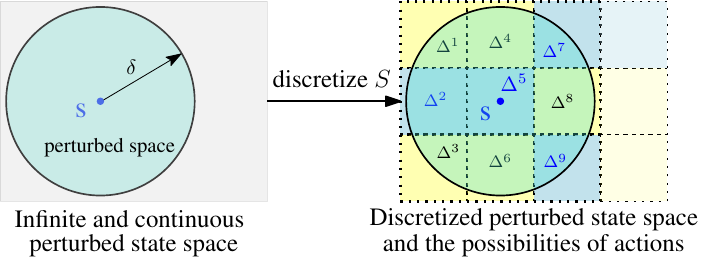}
	\caption{An example of state space discretization}
	\label{fig:dis_s}
	\vspace{-6mm}
\end{figure}

\begin{figure}[h!]
	\footnotesize 
	\centering
	\captionsetup[subfigure]{aboveskip=0pt,belowskip=0pt}
	\begin{subfigure}[b]{0.2\textwidth}
		\includegraphics[width=\textwidth]{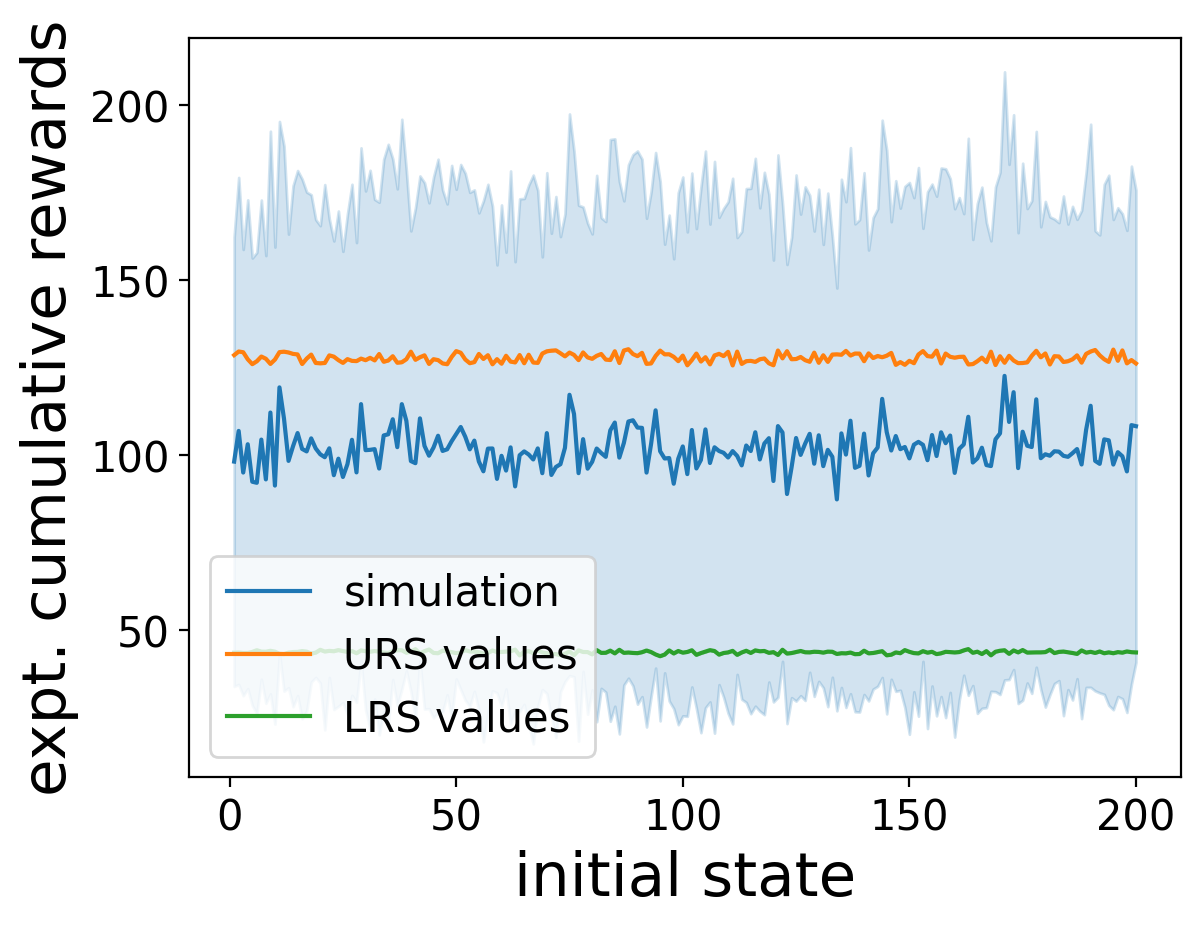}
		\caption{$\sigma=0.1$(A.S.)}
	\end{subfigure}\qquad 
	\begin{subfigure}[b]{0.2\textwidth}
		\includegraphics[width=\textwidth]{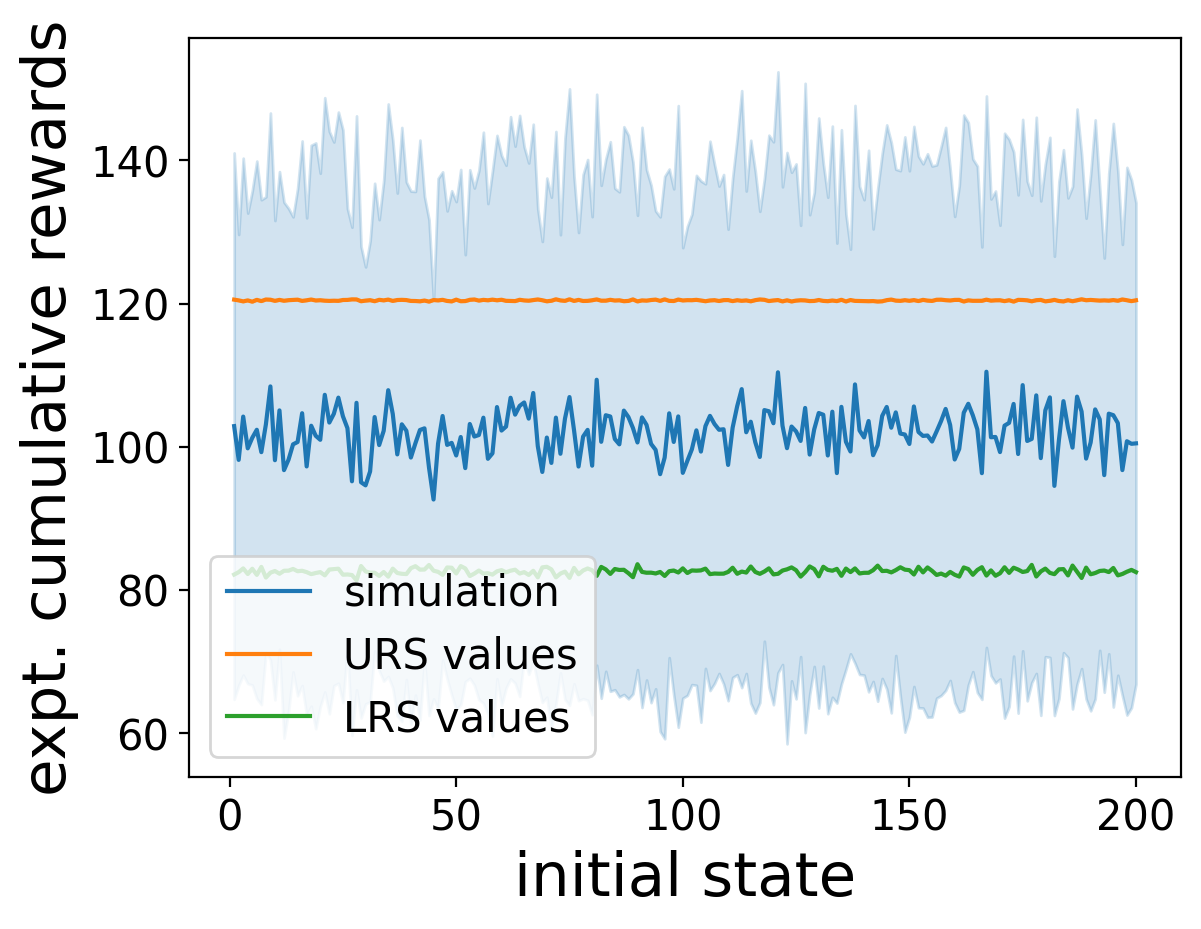}
		\caption{$\sigma=0.1$(C.S.)}
	\end{subfigure}\\ 
	\begin{subfigure}[b]{0.2\textwidth}
		\includegraphics[width=\textwidth]{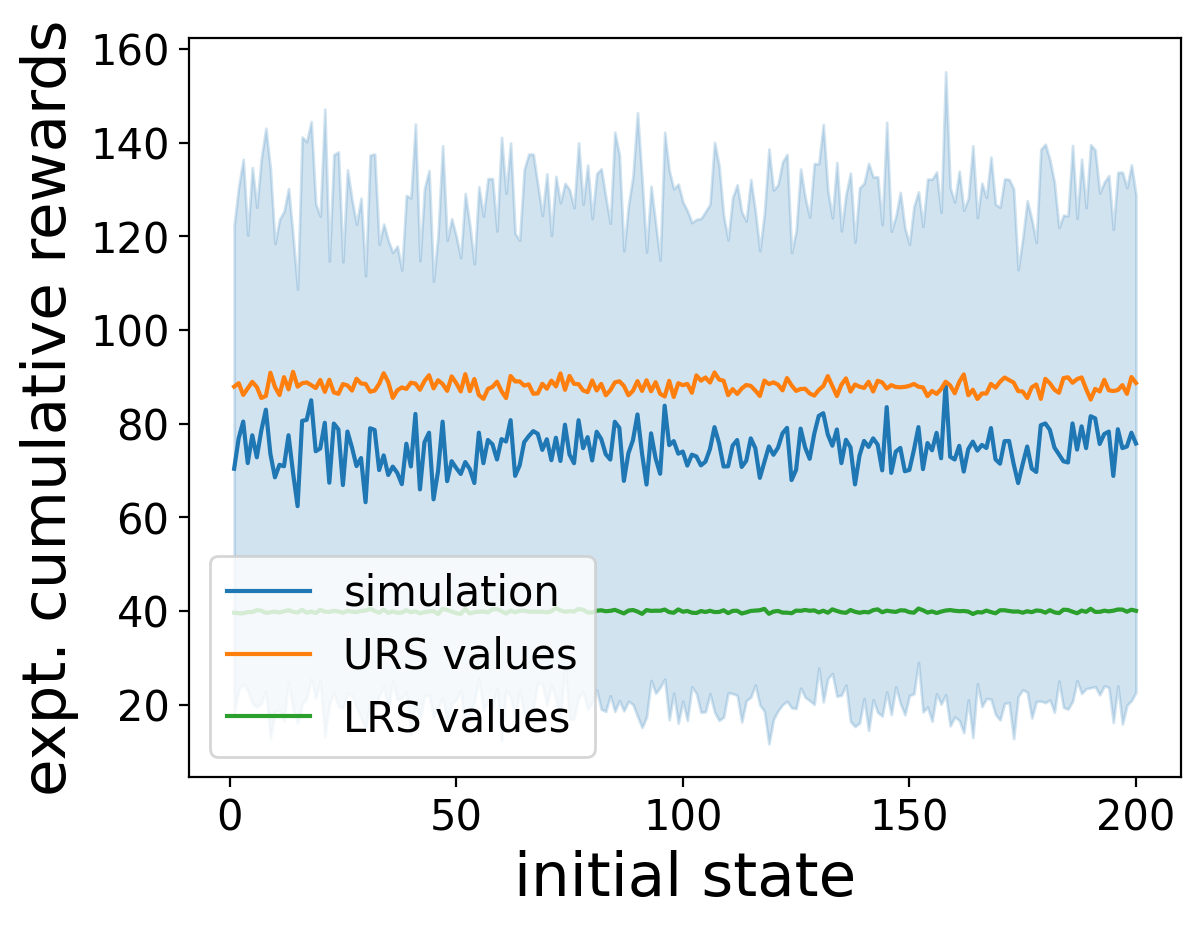}
		\caption{$r=0.2$(A.S.)}
	\end{subfigure}\qquad  
	\begin{subfigure}[b]{0.2\textwidth}
		\includegraphics[width=\textwidth]{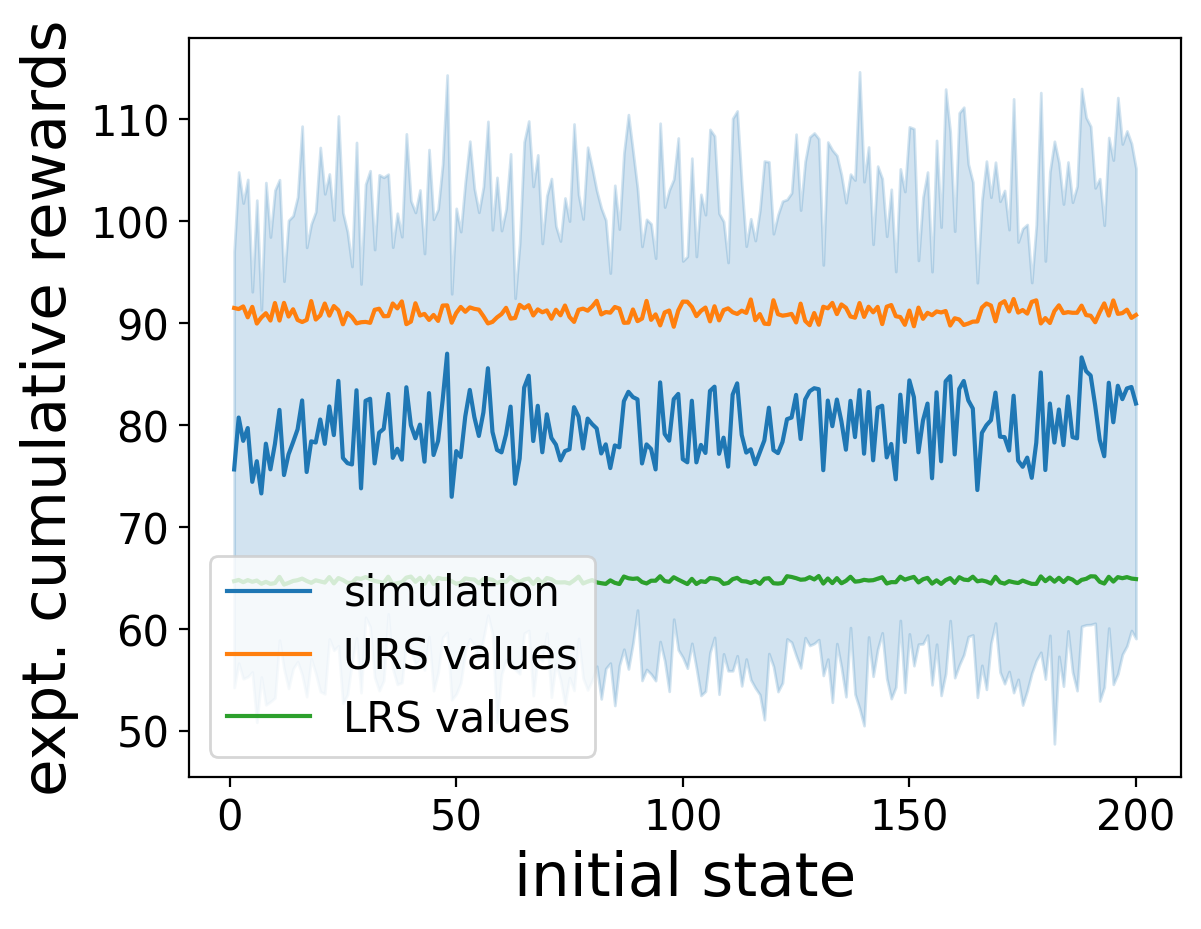}
		\caption{$r=0.2$(C.S.)}
	\end{subfigure}\\
	\begin{subfigure}[b]{0.2\textwidth}
		\includegraphics[width=\textwidth]{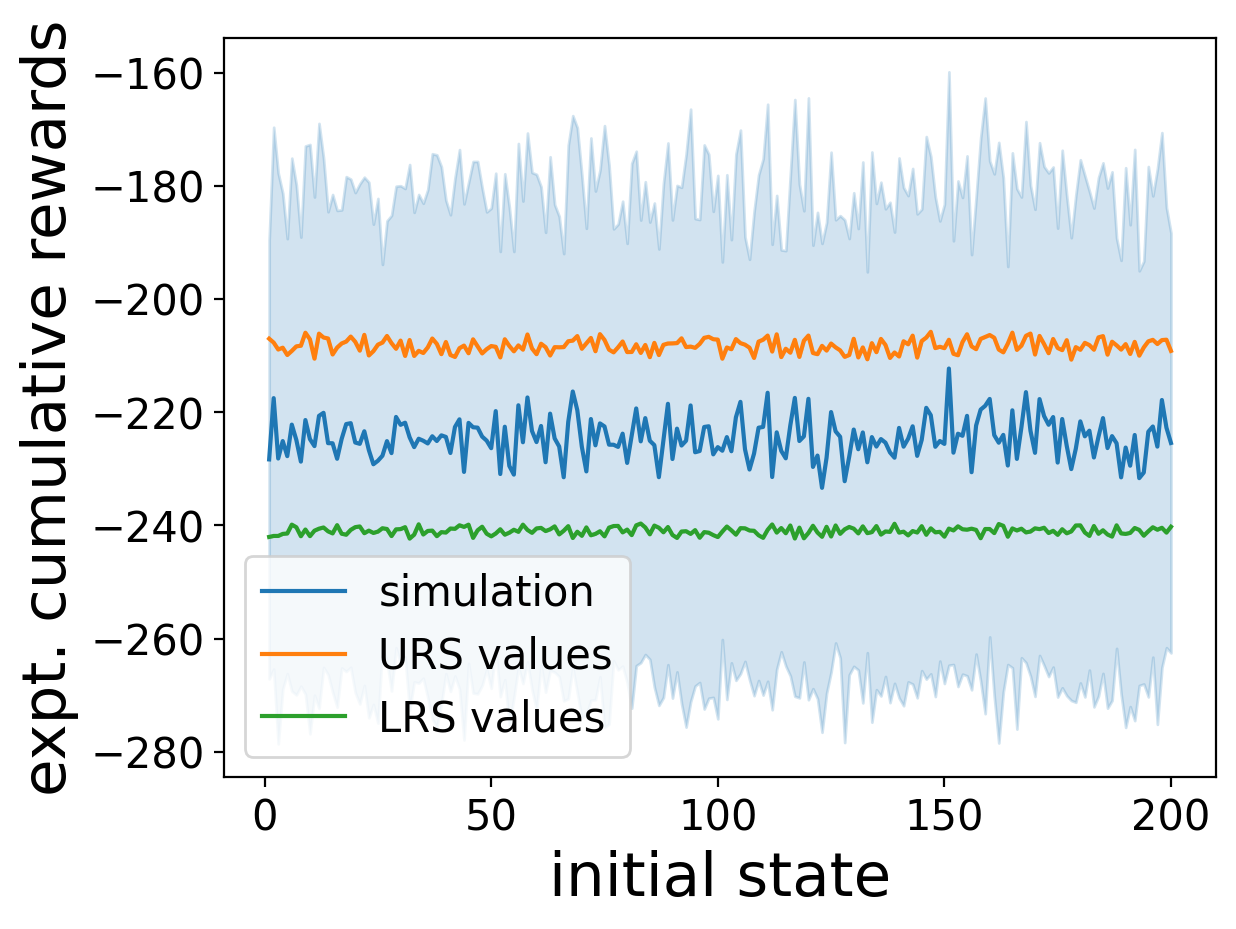}
		\caption{$\sigma=0.3$(A.S.)}
	\end{subfigure}\qquad 
	\begin{subfigure}[b]{0.2\textwidth}
		\includegraphics[width=\textwidth]{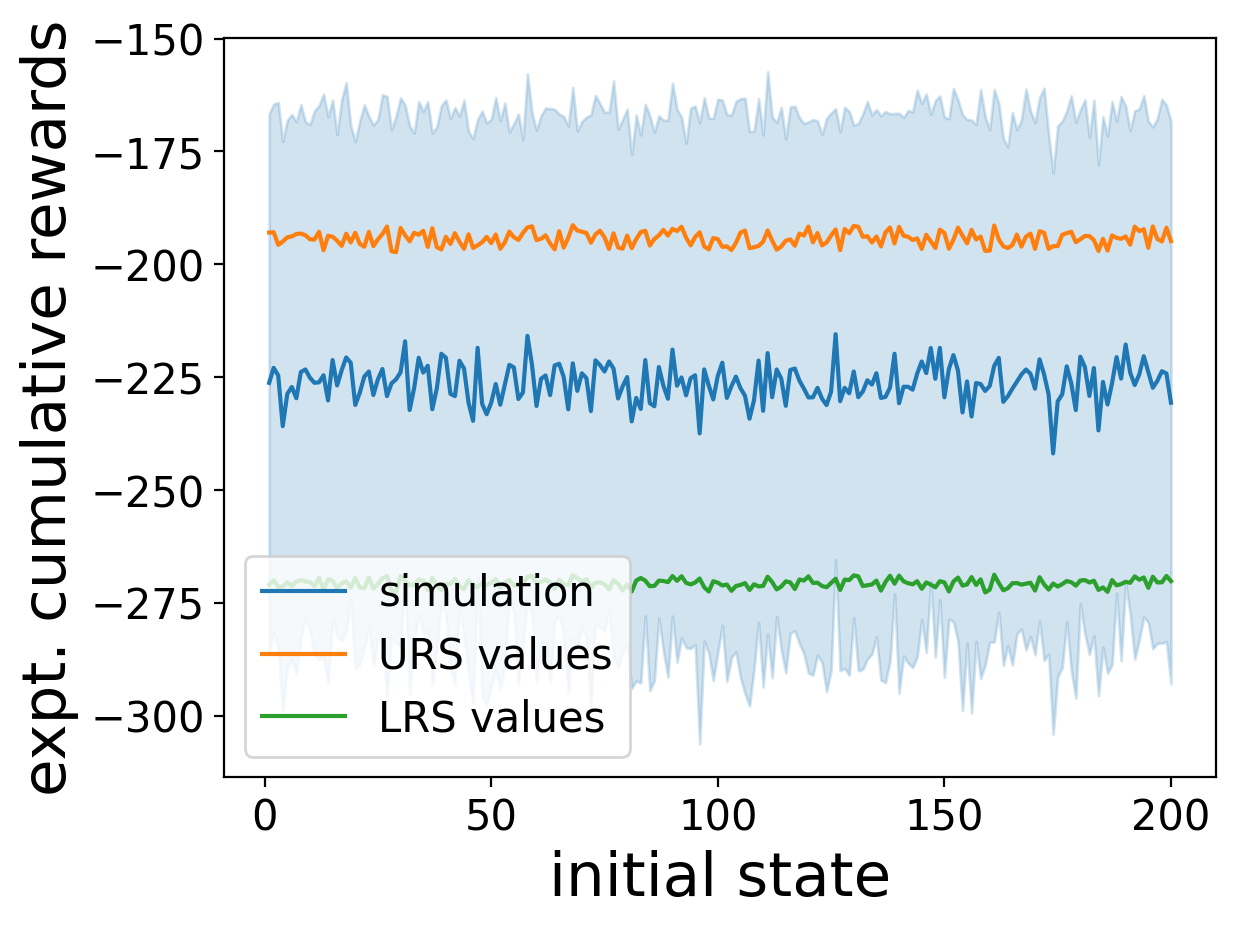}
		\caption{$\sigma=0.3$(C.S.)}
	\end{subfigure}\\ 
	\begin{subfigure}[b]{0.2\textwidth}
		\includegraphics[width=\textwidth]{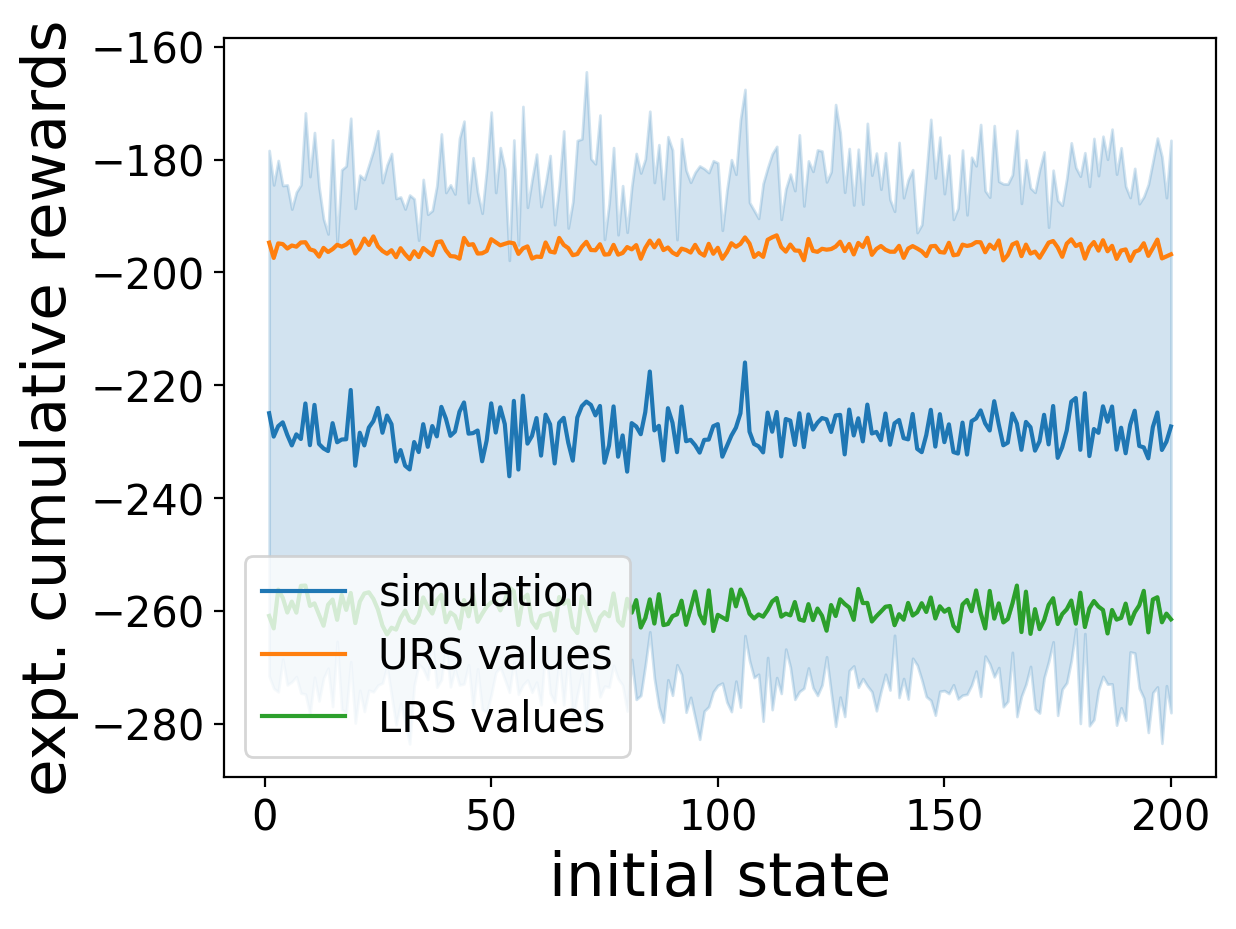}
		\caption{$r=0.5$(A.S.)}
	\end{subfigure}\qquad 
	\begin{subfigure}[b]{0.2\textwidth}
		\includegraphics[width=\textwidth]{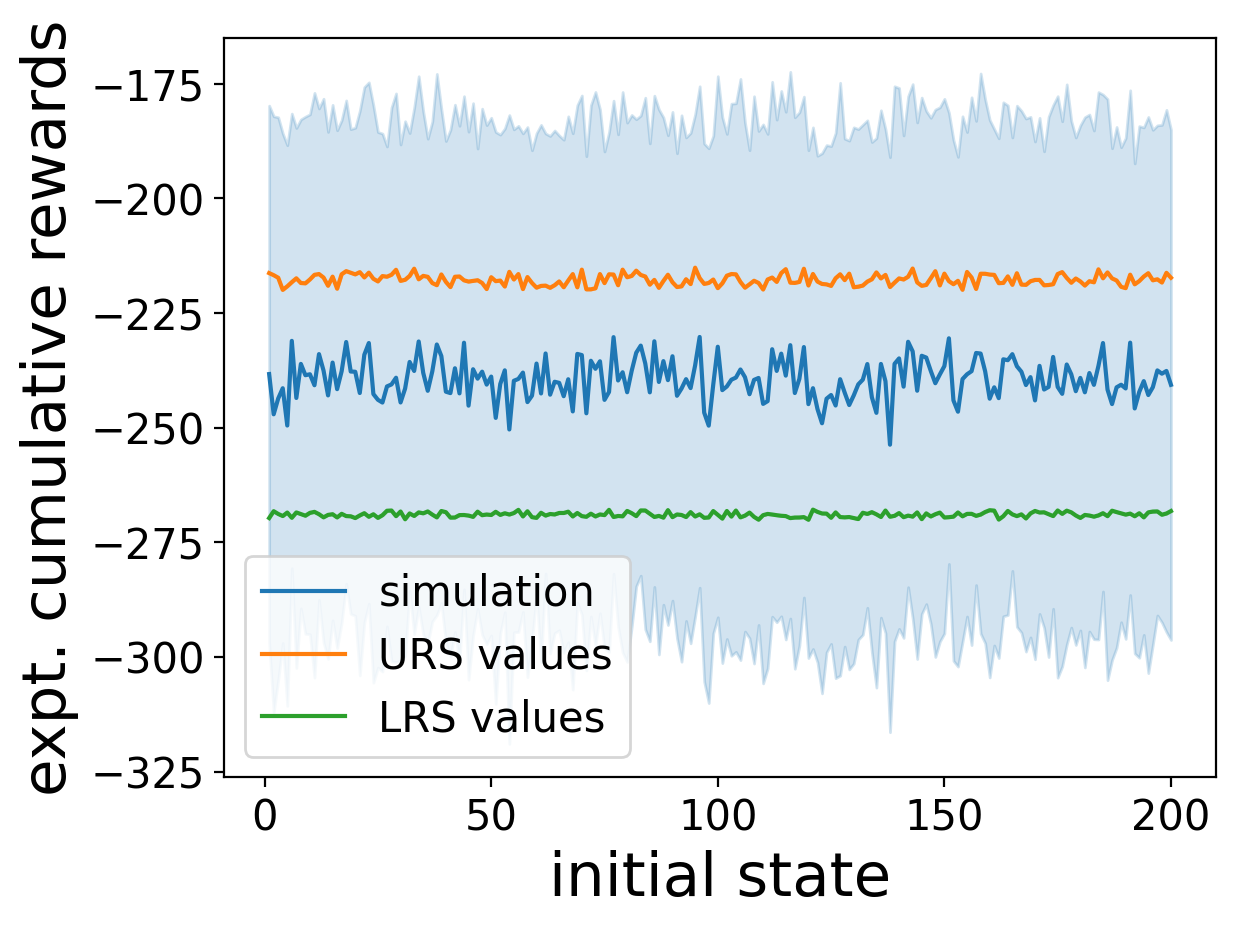}
		\caption{$r=0.5$(C.S.)}
	\end{subfigure}
		\captionsetup{skip=5pt}
	\caption{Certified bounds and simulation results for CartPole (a-d) and B1 (e-h).}
	\vspace{-2mm}
	\label{fig:CP_B1_RM}
\end{figure}

\section{Experimental Evaluations}
We first verified the termination of perturbed systems using existing methods in \citet{stab_martingales}, and then performed robustness verification for those 
finitely terminating systems.

\paragraph{Experiment Objectives.}
The goals of our experimental evaluations are to evaluate: (i)  the effectiveness of certified upper and lower bounds for expected cumulative rewards,  (ii) the effectiveness of certified tail bounds for cumulative rewards, and (iii) the efficiency of training and validating reward martingales.

\vspace{-2mm}
\paragraph{Experimental Settings.}
We consider four benchmarking problems: CartPole (CP),  MountainCar (MC), B1, and B2 from Gym ~\cite{GYM} and the benchmarks for reachability analysis ~\cite{Verisig}, respectively. 
To demonstrate the generality of our approach, we train systems with different activation functions and network structures of the planted NNs, using  different DRL algorithms such as DQN~\cite{DQN} and DDPG~\cite{DDPG}. 
\Cref{table:benchmarks_setting} gives the details of training settings. Besides, $\tau$ is 0.02 for CP, B1 and B2, and 0.01 for MC. $\xi$  is set to 0.002.

For the robustness verification, we consider two different state perturbations as follows:  
\begin{itemize}
\item Gaussian noises with zero means and 
different deviations. 

\item Uniform noises with different
radii. 
\end{itemize}

\begin{figure*}[t]
	\footnotesize 
	\centering
	\captionsetup[subfigure]{aboveskip=1pt,belowskip=1pt}
	\begin{subfigure}[b]{0.22\textwidth}
		\includegraphics[width=\textwidth]{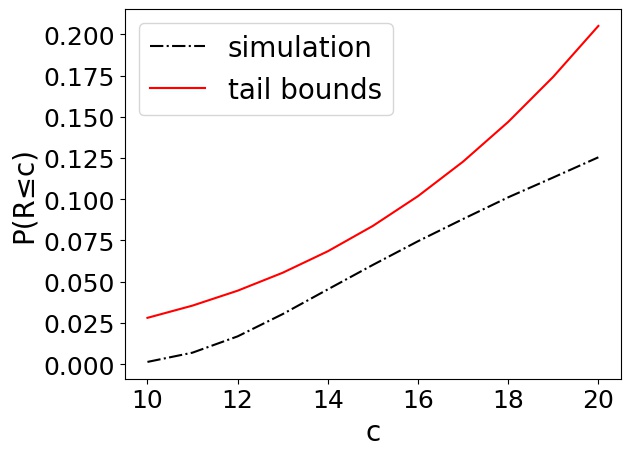}
		\caption{$r=0.2$(A.S.)}
	\end{subfigure}\qquad  
	\begin{subfigure}[b]{0.22\textwidth}
		\includegraphics[width=\textwidth]{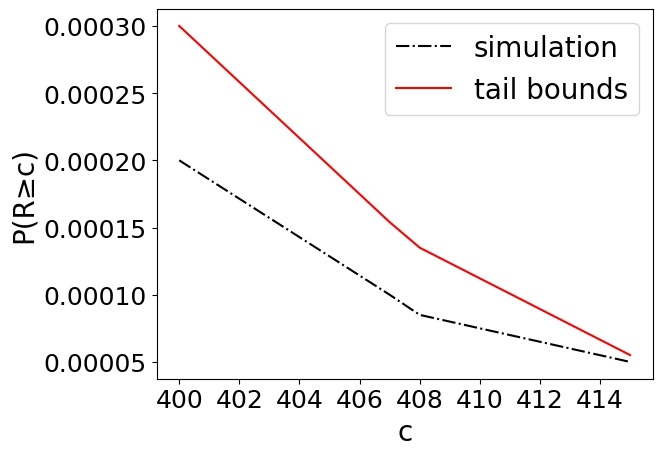}
		\caption{$r=0.2$(A.S.)}
	\end{subfigure} \qquad 
	\begin{subfigure}[b]{0.21\textwidth}
		\includegraphics[width=\textwidth]{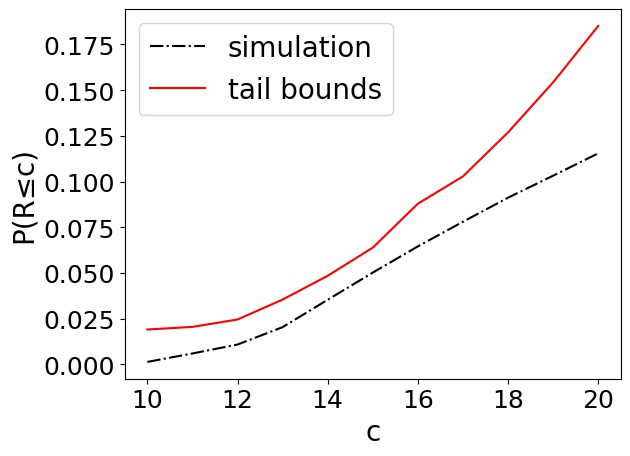}
		\caption{$r=0.2$(C.S.)}
	\end{subfigure}\qquad 
	\begin{subfigure}[b]{0.22\textwidth}
		\includegraphics[width=\textwidth]{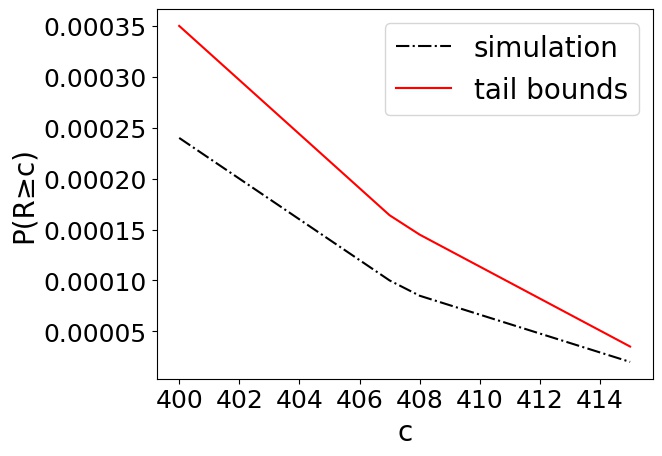}
		\caption{$r=0.2$(C.S.)}
	\end{subfigure}\\
	\begin{subfigure}[b]{0.22\textwidth}
		\includegraphics[width=\textwidth]{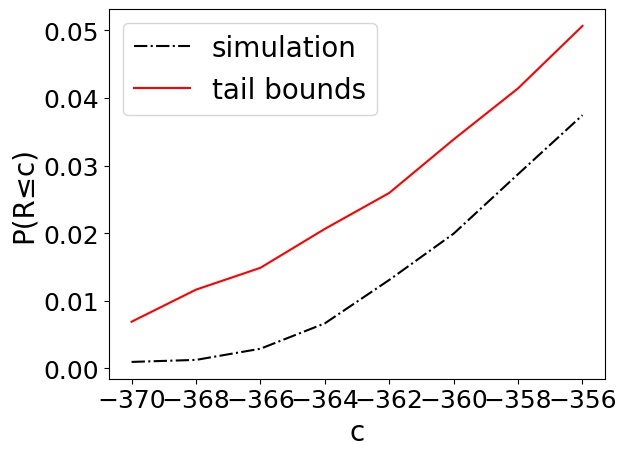}
		\caption{$r=0.5$(A.S.)}
	\end{subfigure}\qquad 
	\begin{subfigure}[b]{0.22\textwidth}
		\includegraphics[width=\textwidth]{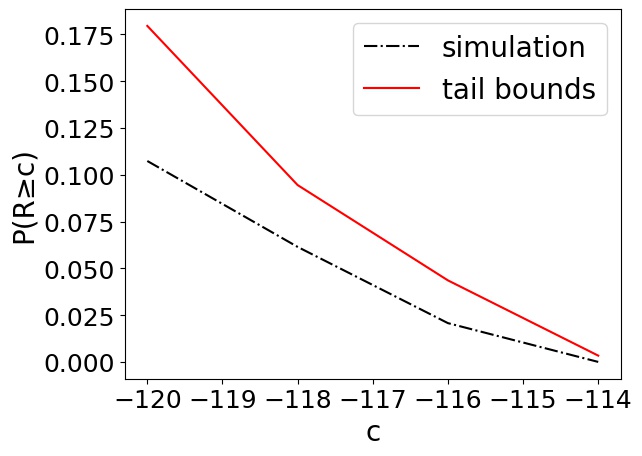}
		\caption{$r=0.5$(A.S.)}
	\end{subfigure} \qquad 
	\begin{subfigure}[b]{0.22\textwidth}
		\includegraphics[width=\textwidth]{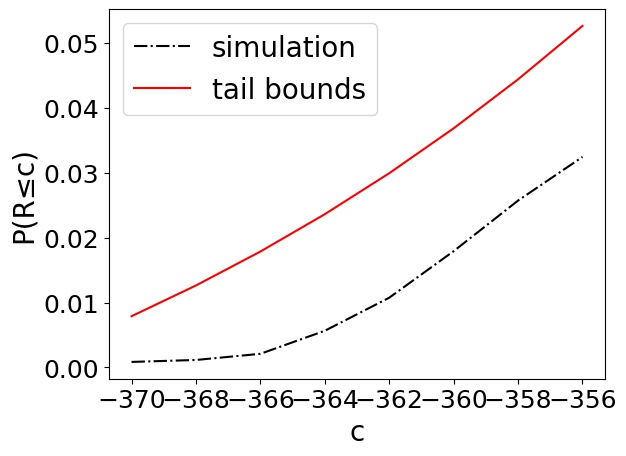}
		\caption{$r=0.5$(C.S.)}
	\end{subfigure}\qquad 
	\begin{subfigure}[b]{0.22\textwidth}
		\includegraphics[width=\textwidth]{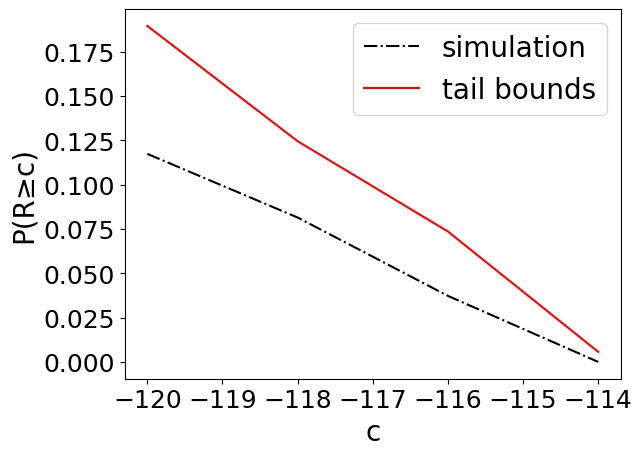}
		\caption{$r=0.5$(C.S.)}
	\end{subfigure}
	\vspace{-2mm}
	\caption{Certified tail bounds of expected cumulative rewards and simulation results for CartPole (a-d) and B1 (e-h).}
	\vspace{-5mm}
	\label{fig:CP_B1_Tail_Bound}
\end{figure*}

\begin{table}[h!]
	\vspace{-4mm}
	\centering
	\footnotesize
	\setlength{\tabcolsep}{2.7pt}
	\begin{tabular}{l|r r r r r r r}
		\hline
		\textbf{Task}&\centering \textbf{Dim.}&\centering \textbf{Alg.} &\centering \textbf{A.F.}&\centering \textbf{Size}&\textbf{A.T.} &\textbf{S.P.} &\textbf{Training} \\
		\hline
		CP & 4 & DQN & ReLU & $3\times200$  & Dis.  & Gym &C.S./A.S.\\
		MC & 2 & DQN & Sigmoid & $2\times200$  & Dis.  & Gym &C.S./A.S.\\
		B1 & 2 & DDPG & ReLU & $2\times100$  & Cont.  & R.A. &C.S./A.S.\\
		B2 & 2 &  DDPG & Tanh & $2\times300$  & Cont.  & R.A.  &C.S./A.S.\\
		\hline
	\end{tabular}
	\begin{tablenotes}
		\footnotesize \item \textbf{Remarks.}
		\textbf{Dim.}: dimension; 
		\textbf{Alg.}: DRL algorithm; 
		\textbf{A.F.}: activation function;   
		\textbf{A.T.}: action type; 
		\textbf{S.P.}: sources of problems;
		\textbf{Dis.}: discrete; 
		\textbf{Cont.}: continuous; 
		\textbf{R.A.}: reachability analysis;
		\textbf{C.S.}:  training on concrete states;
		\textbf{A.S.}: training on abstract states.
	\end{tablenotes}
	\captionsetup{skip=5pt}
	\caption{Experimental settings.}
	\vspace{-4mm}
	\label{table:benchmarks_setting}
\end{table}

\noindent Specifically, for each  state $s = (s_1,\ldots,s_n)$, we add noises $X_1,\ldots,X_n$ to each dimension of $s$ and obtain the perturbed state $(s_1 + X_1,\ldots,s_n + X_n)$, where $X_i \sim \mathbf{U}(-r, r)$ ($1\le i \le n$) is some uniformly distributed noise or  $X_i \sim \mathbf{G}(0, \sigma)$ ($1\le i \le n$) is some Gaussian distributed noise.

\paragraph{Effectiveness of Certified Upper and Lower Bounds.}

\vspace{-2mm}
\cref{fig:CP_B1_RM}  shows the certified bounds of expected cumulative rewards (\cref{thm:bounds}) and the simulation results for CP and B1 under different perturbations and policies, respectively. The $x$-axis indicates different initial states, while the $y$-axis means the corresponding expected cumulative rewards. The orange lines represent the upper bounds calculated by the trained URSs, the green lines represent the lower bounds computed by the trained LRSs, and the blue lines and shadows represent the means and standard deviations of the simulation results that are obtained by executing $200$ episodes for each initial state. 
We can find that the certified bounds tightly enclose the simulation outcomes,  demonstrating the effectiveness of our trained reward martingales.

We also observe that the tightness of the certified bounds depends on particular systems,  trained policies, and perturbations. It is worth further investigating how these factors affect the tightness to produce tighter bounds. 

\vspace{-2mm}
\paragraph{Effectiveness of Certified Tail Bounds.} 
\cref{fig:CP_B1_Tail_Bound} depicts the certified tail bounds 
and statistical results for CP and B1 under different perturbations and  policies. Due to the data sparsity of cumulative rewards, we choose 200 different initial states (instead of a single one) and execute the systems by 200 episodes for each initial state. We record the cumulative rewards, and
the statistical results of tail probabilities for different $c$'s are shown by the black dashed lines. 
For each initial state $s_0$ from the 200 initial states, we calculate $\probm_{s_0}(\calR\ge c)$ and $\probm_{s_0}(\calR\le c)$ according to \cref{thm:tail-bounds}. Their average values $P(\calR \ge c)$ and $P(\calR \le c)$ are shown by the red solid lines.
The results show that our calculated tail bounds tightly enclose the statistical outcomes. 
We also observe that the trend of the calculated tail bounds is exponential upward or downward, which is consistent with \Cref{thm:tail-bounds}. The same claims also hold for any single initial state.

\vspace{-2mm}
\paragraph{Efficiency Comparison.} \Cref{table:time} shows the time cost of  training and validating reward martingales under different policies and perturbations. In general,  training costs more time than validating,  and high-dimensional systems e.g., CP (4-dimensional state space)  cost more time than low-dimensional ones, e.g, B1 (2-dimensional state space). 
That is because the validation step suffers from the curse of high-dimensionality~\cite{Stability_Guarantees}.

\begin{table}[t]
	\centering
	\footnotesize
	\setlength{\tabcolsep}{4pt}
	\begin{tabular}{l|r| r| r r r| r r r}
	\hline
	&&& \multicolumn{3}{c|}{\textbf{URS}} & \multicolumn{3}{c}{\textbf{LRS}}\\ \hline
 
	\textbf{Task} &\centering \textbf{Pert.} &\centering \textbf{Policy}&\centering \textbf{T.T.} &\centering \textbf{V.T.}&\centering \textbf{To.T.}&\centering \textbf{T.T.}&\textbf{V.T.} &\textbf{To.T.} \\
        \hline
	\multirow{4}{*}{\textbf{CP}} & \multirow{2}{*}{\textbf{$\sigma=0.1$}} & A.S.  &912  &751 & 1663 & 1165 & 822 & 1987\\
	&   &C.S. & 1203  &  873 &2076  & 1384 & 918 & 2302 \\
	\cline{2-9}
	& \multirow{2}{*}{$r=0.2$} &A.S. & 901  & 409 & 1310  & 1081 & 535 &1616 \\
	& &   C.S. & 898  &  529 &1427 & 1009 & 649 & 1658\\
	\hline
	\multirow{4}{*}{\textbf{B1} } & \multirow{2}{*}{\textbf{$\sigma=0.3$}} & A.S.  & 501  &  45 & 546 & 629  &  41&  670\\
	&   & C.S. & 570  &  53 & 623 & 850 & 189 & 1036\\
	\cline{2-9}
	& \multirow{2}{*}{\textbf{$r=0.5$}} & A.S.  & 530 & 47 & 577  & 681  & 175 &856 \\
	&  & C.S. & 593  & 54 & 647 & 945 & 280 & 1225 \\
	\hline
	\end{tabular}
	\begin{tablenotes}
		\footnotesize \item 
		\textbf{T.T.}: training time; 
		\textbf{V.T.}: validating time; 
		\textbf{To.T.}: total time. 
		\vspace{-1mm}
	\end{tablenotes}
	\captionsetup{skip=2pt}
	\caption{Time on training and validating reward martingales.}
	\vspace{-7mm}
	\label{table:time}
\end{table}

We also observe that, the analytic  approach is more efficient in training and validating than the over-approximation-based approach with up to 35.32\% improvement. This result is consistent to the fact that the analytical approach can produce more precise results of the expected values (\Cref{theorem:verify_RM}) and consequently can reduce false positives and unnecessary refinement and training.

Note that the conclusions from the above results are also applicable to MC and B2. 
More experimental results, a detailed discussion of different hyperparameter settings, and all omitted proof can be found in the Appendix.


\vspace{-1mm}
\section{Concluding Remarks}
In this paper, we have introduced a groundbreaking quantitative robustness verification approach for perturbed DRL-based control systems, utilizing the innovative concept of reward martingales. Our work has established two fundamental theorems that serve as cornerstones: the certification of reward martingales as rigorous upper and lower bounds, as well as their role in tail bounds for system robustness. We have presented an algorithm that effectively trains reward martingales through the implementation of neural networks. Within this algorithm, we have devised two distinct methods for computing expected values, catering to control policies developed across diverse state space configurations. Through an extensive evaluation encompassing four classical control problems, we have convincingly showcased the versatility and efficacy of our proposed approach.

We believe this work would inspire several future studies to reduce the complexity in validating reward martingales for high dimensional systems. Besides, it is also  worth further investigation to advance the approach for calculating tighter certified bounds by training more precise reward martingales. 

\clearpage

\section{Acknowledgments}

The work has been supported by the NSFC-ISF Joint Program (62161146001,3420/21), NSFC Project (62372176), Huawei Technologies Co., Ltd., Shanghai International Joint Lab of Trustworthy Intelligent Software (Grant No. 22510750100),  Shanghai Trusted Industry Internet Software Collaborative Innovation Center, the Engineering and Physical Sciences Research Council (EP/T006579/1), National Research Foundation (NRF-RSS2022-009), Singapore, and Shanghai Jiao Tong University Postdoc Scholarship.
\bibliography{aaai24}
\clearpage
\begin{center}
	\Large\textbf{Appendix}
\end{center}

\section{Probability Theory}

We start by reviewing some notions from probability theory.

\noindent\textbf{Random Variables and Stochastic Processes.}
A probability space is a triple ($\Omega, {\cal F}, {\mathbb P}$), where $\Omega$ is a non-empty sample space, ${\cal F}$ is a
$\sigma$-algebra over $\Omega$, and  $\mathbb P(\cdot)$ is a probability measure over $\cal F$, i.e. a function $\mathbb P$: ${\cal F}  \rightarrow [0,1]$ that satisfies the following properties: (1) ${\mathbb P}(\emptyset)=0$, (2)${\mathbb P}(\Omega - A)=1-{\mathbb P}[A]$, and (3) ${\mathbb P} (\textstyle\cup_{i=0}^{\infty} A_i)= \textstyle\sum_{i=0}^{\infty}{\mathbb P}(A_i)$ for  any
sequence $\{A_i\}_{i=0}^{\infty}$ of pairwise disjoint sets in ${\cal F}$.

Given a probability space ($\Omega, {\cal F}, {\mathbb P}$), a
random variable is a function $X: \Omega \rightarrow {\mathbb R} \cup \{\infty\}$ that is $\cal F$-measurable, i.e., for each $a \in {\mathbb R}$ we have that $\{\omega \in \Omega | X(\omega) \leq a \} \in {\cal F}$.
Moreover, a discrete-time stochastic process is a sequence $\{X_n\}_{n=0}^{\infty}$ of random variables in ($\Omega, {\cal F}, {\mathbb P}$).

\noindent\textbf{Conditional Expectation.}
Let ($\Omega, {\cal F}, {\mathbb P}$) be a probability space and $X$ be a random variable in ($\Omega, {\cal F}, {\mathbb P}$). The expected value of the random variable $X$,
denoted by ${\mathbb E}[X]$, is the Lebesgue integral of $X$ wrt $\mathbb P$. If the range
of $X$ is a countable set $A$, then ${\mathbb E}[X]= \textstyle\sum_{\omega \in A}\omega \cdot {\mathbb P}(X=\omega)$. 
Given a sub-sigma-algebra ${\cal F}' \subseteq {\cal F}$, a conditional expectation of $X$ for the given ${\cal F}'$ is a ${\cal F}'$-measurable random variable $Y$ such
that, for any $A \in {\cal F}'$, we have:
\begin{linenomath*}
	\begin{equation}
	\begin{split}
     {\mathbb E}[X \cdot {\mathbb I}_{A}]={\mathbb E}[Y \cdot {\mathbb I}_{A}]
	\end{split}
	\end{equation}
\end{linenomath*}

\noindent Here, ${\mathbb I}_{A}: \Omega \rightarrow \{0,1\}$  is an indicator function of $A$, defined as ${\mathbb I}_{A}(\omega)=1$ if $\omega \in A$ and ${\mathbb I}_{A}(\omega)=0$ if $\omega \notin A$. 
Moreover, whenever the conditional expectation exists, it is also almost-surely unique, i.e., for any two ${\cal F}'$-measurable random variables $Y$ and $Y'$ which are conditional expectations of $X$ for given ${\cal F}'$, we have that ${\mathbb P}(Y=Y')=1$.

\noindent\textbf{Filtrations and Stopping Times.}
A filtration of the probability space ($\Omega, {\cal F}, {\mathbb P}$) is an infinite sequence $\{{{\cal F}_n}\}_{n=0}^{\infty}$ such that
for every $n$, the triple ($\Omega, {\cal F}_n, {\mathbb P}$) is a probability space and ${\cal F}_n \subseteq {\cal F}_{n+1} \subseteq {\cal F}$. 
A stopping time with respect to a filtration $\{{{\cal F}_n}\}_{n=0}^{\infty}$  is a random variable $T: \Omega \rightarrow {\mathbb N}_0 \cup \{\infty\}$ such that, for every $i \in {\mathbb N}_0$, it holds that $\{\omega \in \Omega |T(\omega) \leq i \} \in {\cal F}_i$. 
Intuitively, $T$ returns the time step at which some stochastic process shows a desired behavior and should be “stopped”.

A discrete-time stochastic process  $\{X_n\}_{n=0}^{\infty}$ in ($\Omega, {\cal F}, {\mathbb P}$) is adapted to a filtration $\{{{\cal F}_n}\}_{n=0}^{\infty}$, if
for all $n \geq 0$, $X_n$ is a random variable in ($\Omega, {\cal F}_n, {\mathbb P}$). 

\noindent\textbf{Martingales.}
A discrete-time stochastic process $\{X_n\}_{n=0}^\infty$ to a filtration $\{{{\cal F}_n}\}_{n=0}^{\infty}$ is a martingale (resp. supermartingale, submartingale) if for all $n \geq 0$, ${\mathbb E}[|X_n|] < \infty$ and it holds almost surely (i.e., with probability 1) that ${\mathbb E}[X_{n+1}|{\cal F}_n] = X_n$ (resp. ${\mathbb E}[X_{n+1}|{\cal F}_n] \leq X_n$, ${\mathbb E}[X_{n+1}|{\cal F}_n] \geq X_n$).

\noindent\textbf{Ranking Supermartingales.}
Let $T$ be a stopping time w.r.t. a filtration $\{{{\cal F}_n}\}_{n=0}^{\infty}$. 
A discrete-time stochastic process $\{X_n\}_{n=0}^\infty$ 
w.r.t. a stopping time $T$ 
is a ranking supermartingale (RSM) if for all $n \geq 0$, ${\mathbb E}[|X_n|] < \infty$ and there exists $\epsilon >0$ such that it holds almost surely (i.e., with probability 1) that $X_n \geq 0$ and ${\mathbb E}[X_{n+1}|{\cal F}_n] \leq X_n-\epsilon \cdot {\mathbb I}_{T>n}$.

\section{Proofs of Theorems}

Consider a perturbed DRL-based control system $M_\mu=(S,S_0,S_g,A,\pi,f,R,\mu)$. Fix an initial state $s_0\in S_0$ and its probability space $(\Omega_{s_0},\mathcal{F}_{s_0},\probm_{s_0})$. 

\subsection{Proofs of \cref{thm:bounds}}
Let $\{X_n\}_{n=0}^{\infty}$ be a stochastic process w.r.t. some filtration $\{\mathcal{F}_n\}_{n=0}^\infty$ over $(\Omega_{s_0},\mathcal{F}_{s_0},\probm_{s_0})$ such that $X_n:=h(\overline{s}_n)$ where $h:S\rightarrow \Rset$ is a function over $S$ and $\overline{s}_n$ is a random (vector) variable representing the value(s) of the state at the $n$-th step of an episode. 
Then we construct another stochastic process $\{Y_n\}_{n=0}^\infty$ such that $Y_n:=X_n+\sum_{i=0}^n r_i$ where $r_i$ is the reward at the $i$-th step of an episode.

\begin{proposition}\label{prop:URS}
   If $h$ is an URS (see \cref{def:URS}), then $\{Y_n\}_{n=0}^\infty$ is a supermartingale. 
\end{proposition}

\begin{proof}
By the definition of pre-expectation (see \cref{def:pre-ex}), we have that for all $n\in \Nset_0$, $\expv [X_{n+1}+r_{n}\mid \mathcal{F}_n]=pre_h(\overline{s}_n)$. Since $Y_{n}=X_n+\sum_{i=0}^{n} r_i$, we can obtain that
\begin{align*}
\expv [Y_{n+1}\mid \mathcal{F}_n]&= \expv [X_{n+1}+r_{n+1}+\sum_{i=0}^{n} r_i\mid \mathcal{F}_n] \\
 &= \expv [X_{n+1}+r_{n+1}+\sum_{i=0}^n r_i+ X_n-X_n\mid \mathcal{F}_n] \\
 &= \expv [X_{n+1}+r_{n+1}+Y_n-X_n \mid \mathcal{F}_n]  \\
 &= Y_n -X_n+ \expv [X_{n+1}+r_{n+1}\mid \mathcal{F}_n]\\
 &= Y_n -X_n+pre_h(\overline{s}_n) \\
 &\le Y_n
\end{align*}

where the last inequality is derived from the decreasing pre-expectation condition in \cref{def:URS}.
Hence, $\{Y_n\}_{n=0}^\infty$ is a supermartingale.
\end{proof}

\begin{proposition}\label{prop:LRS}
   If $h$ is a LRS (see \cref{def:LRS}), then $\{Y_n\}_{n=0}^\infty$ is a submartingale. 
\end{proposition}

\begin{proof}
By the definition of pre-expectation (see \cref{def:pre-ex}), we have that for all $n\in \Nset_0$, $\expv [X_{n+1}+r_{n}\mid \mathcal{F}_n]=pre_h(\overline{s}_n)$. Since $Y_{n}=X_{n}+\sum_{i=0}^{n} r_i$, we can obtain that
\begin{align*}
\expv [Y_{n+1}\mid \mathcal{F}_n]&= \expv [X_{n+1}+r_{n+1}+\sum_{i=0}^{n} r_i\mid \mathcal{F}_n] \\
 &= \expv [X_{n+1}+r_{n+1}+\sum_{i=0}^n r_i+ X_n-X_n \mid \mathcal{F}_n] \\
 &= \expv [X_{n+1}+r_{n+1}+Y_n-X_n \mid \mathcal{F}_n]  \\
 &= Y_n -X_n + \expv [X_{n+1}+r_{n+1}\mid \mathcal{F}_n]\\
 &= Y_n -X_n +pre_h(\overline{s}_n) \\
 &\ge Y_n
\end{align*}
where the last inequality is derived from the increasing pre-expectation condition in \cref{def:LRS}.
Hence, $\{Y_n\}_{n=0}^\infty$ is a submartingale.
\end{proof}

To prove \cref{thm:bounds}, we use the classical Optional Stopping Theorem as our mathematical foundation.
\begin{theorem}[Optional Stopping Theorem (OST) \cite{williams1991}]\label{thm:OST}
Let $\{X_n\}_{n\in\Nset_0}$ be a supermartingale (resp. submartingale) adapted to a filtration $\mathcal{F}=\{\mathcal{F}_n\}_{n\in\Nset_0}$, and $\kappa$ be a stopping time w.r.t. the filtration $\mathcal{F}$. Then the following condition is sufficient to ensure that $\expv(|X_\kappa|)<\infty$ and $\expv(X_\kappa)\le \expv(X_0)$ (resp. $\expv(X_\kappa)\ge \expv(X_0)$):
\begin{itemize}
\item $\expv(\kappa)<\infty$, and
\item there exists a constant $C>0$ such that for all $n\ge 0$, $|X_{n+1}-X_n|\le C$ holds almost surely.
\end{itemize}
\end{theorem}

Based on \cref{prop:URS}, \cref{prop:LRS} and \cref{thm:OST}, we can derive our theoretical results about upper and lower bounds for expected cumulative rewards.

\noindent\textbf{\cref{thm:bounds}} (Bounds for Expected Cumulative Rewards).
	Suppose an $M_\mu$ has a difference-bounded URS (resp. LRS) $h$ and $K,K'\in\Rset$ are the bounds of $h$. For each state $s_0\in S_0$, we have 
	\begin{align}
	&\expv_{s_0}[\calR]\le h(s_0)-K.\tag{Upper Bound}\\
	 (\textit{resp.}\  &\expv_{s_0}[\calR]\ge h(s_0)-K') \tag{Lower Bound}
	\end{align}

\begin{proof}
(Upper bounds). For any initial state $s_0\in S_0$, we construct a stochastic process $\{Y_n\}_{n=0}^\infty$ as above, i.e.. $Y_n:=h(\overline{s}_n)+\sum_{i=0}^n r_i$. As $h$ is an URS, by \cref{prop:URS}, $\{Y_n\}_{n=0}^\infty$ is a supermartingale. Since $M_\mu$ is finite terminating (see our model assumptions), we have that $\expv_{s_0}[T]<\infty$, and thus the first prerequisite of OST is satisfied. Then, from the difference-bounded property of $h$, we can derive that,
\begin{align*}
    |Y_{n+1}-Y_n| &= |h(\overline{s}_{n+1})+\sum_{i=0}^{n+1} r_i-h(\overline{s}_n)-\sum_{i=0}^n r_i| \\
    &= |h(\overline{s}_{n+1})-h(\overline{s}_n)+ r_{n+1}| \\
    &\le |h(\overline{s}_{n+1})-h(\overline{s}_n)| + |r_{n+1}| \\
    &\le m+ r_{\mathrm{max}} \\
    &=: C\\
\end{align*}
where the second inequality is obtained by the difference-boundedness condition in \cref{def:diff-bound}, and $r_{\mathrm{max}}$ is the maximal value of the reward. The second prerequisite of OST is thus satisfied. Therefore, by applying OST, we can have that $\expv [Y_T]\le \expv [Y_0]$. By definition,
\begin{align*}
& Y_{T}=h(\overline{s}_T)+\sum_{i=0}^T r_i    \\
\Leftrightarrow & \sum_{i=0}^T r_i=Y_{T}-h(\overline{s}_T)
\end{align*}
Thus, by the boundedness condition in \cref{def:URS}, $\expv_{s_0}[\calR]=\expv [Y_{T}]-h(\overline{s}_T)\le \expv [Y_0]-K=h(\overline{s}_0)-K$.

(Lower bounds). For any initial state $s_0\in S_0$, we construct a stochastic process $\{Y_n\}_{n=0}^\infty$ as above, i.e.. $Y_n:=h(\overline{s}_n)+\sum_{i=0}^n r_i$. As $h$ is a LRS, by \cref{prop:LRS}, $\{Y_n\}_{n=0}^\infty$ is a submartingale, which implies that $\{-Y_n\}_{n=0}^\infty$ is a supermartingale. Since $M_\mu$ is finite terminating (see our model assumptions), we have that $\expv_{s_0}[T]<\infty$, and thus the first prerequisite of OST is satisfied. Then, from the difference-bounded property of $h$, we can derive that,
\begin{align*}
    |-Y_{n+1}-(-Y_n)| &= |h(\overline{s}_{n+1})+\sum_{i=0}^{n+1} r_i-h(\overline{s}_n)-\sum_{i=0}^n r_i| \\
    &= |h(\overline{s}_{n+1})-h(\overline{s}_n)+ r_{n+1}| \\
    &\le |h(\overline{s}_{n+1})-h(\overline{s}_n)| + |r_{n+1}| \\
    &\le m+ r_{\mathrm{max}} \\
    &=: C\\
\end{align*}
where the second inequality is obtained by the difference-boundedness condition in \cref{def:diff-bound}, and $r_{\mathrm{max}}$ is the maximal value of the reward. The second prerequisite of OST is thus satisfied. Therefore, by applying OST, we can have that $\expv [-Y_T]\le \expv [-Y_0]$, so $\expv [Y_T]\ge \expv [Y_0]$. By definition,
\begin{align*}
& -Y_{T}=-h(\overline{s}_T)-\sum_{i=0}^T r_i    \\
\Leftrightarrow & \sum_{i=0}^T r_i=Y_{T}-h(\overline{s}_T)
\end{align*}
Thus, by the boundedness condition in \cref{def:LRS}, $\expv_{s_0}[\calR]=\expv [Y_{T}]-h(\overline{s}_T)\ge \expv [Y_0]-K'=h(\overline{s}_0)-K'$.

\end{proof}

\subsection{Proofs of \cref{thm:tail-bounds}}
To prove \cref{thm:tail-bounds}, we use Hoeffding's Inequality as our mathematical foundation.
\begin{theorem}[Hoeffding's Inequality on Martingales~\cite{hoeffding1994probability}]\label{thm:hoeffding}
     Let $\{X_n\}_{n\in\Nset_0}$ be a supermartingale w.r.t. a filtration $\{\mathcal{F}_n\}_{n\in\Nset_0}$, and $\{[a_n,b_n]\}_{n\in\Nset_0}$ be a sequence of non-empty intervals in $\Rset$. If $X_0$ is a constant random variable and $X_{n+1}-X_{n}\in [a_n,b_n]$ a.s. for all $n\in\Nset_0$, then 
 \[
 \probm(X_n-X_0\ge \lambda)\le \mathrm{exp}\left(-\frac{2\lambda^2}{\sum_{k=0}^{n-1}(b_k-a_k)^2} \right)
 \]
 for all $n\in\Nset_0$ and $\lambda>0$.
 And symmetrically when $\{X_n\}_{n\in\Nset_0}$ is a submartingale,
  \[
 \probm(X_n-X_0\le -\lambda)\le \mathrm{exp}\left(-\frac{2\lambda^2}{\sum_{k=0}^{n-1}(b_k-a_k)^2} \right)
 \]
\end{theorem}

We also need that $M_\mu$ has the concentration property, which can be ensured by the existence of a difference bounded ranking supermartingale map via existing work \cite{stab_martingales}.

\begin{definition}[Ranking Supermartingale Maps]\label{def:rsm-map}
  A \emph{ranking supermartingale map} (RSM-map) is a function $\eta:S\rightarrow \Rset$ such that $\eta(s)\ge 0$ for all $s\in S$, and there exists a constant $\epsilon>0$ satisfying that for all $s\in S\setminus S_g$, $\expv_{\delta\sim \mu}[\eta(f(s,\pi(s+\delta)))]\le \eta(s)-\epsilon$.
\end{definition}

\begin{definition}[Difference-bounded RSM-maps]\label{def:diff-rsm}
    A ranking supermartingale map $\eta$ is difference-bounded w.r.t. a non-empty interval $[a'',b'']\subseteq \Rset$ if for all $s\in S$, it holds that $a''\le \eta(f(s,\pi(s+\delta)))-\eta(s)\le b''$ with any $\delta\sim \mu$.
\end{definition}

Given an initial state $s_0\in S_0$, we define a stochastic process $\{X_n\}_{n=0}^{\infty}$ in $(\Omega_{s_0},\mathcal{F}_{s_0},\probm_{s_0})$ by:
\begin{equation*}
    X_n:=
    \begin{cases}
        \eta(\overline{s}_n) &  \text{if } $n< T$, \\
        \eta(\overline{s}_T) &   \text{otherwise}.\\
    \end{cases}
\end{equation*}
where $\eta$ is a RSM-map, $\overline{s}_n$ is a random (vector) variable representing the value(s) of the state at the $n$-th step of an episode and $T$ is the termination time. 
\begin{proposition}\label{prop:rsm}
    $\{ X_n\}_{n=0}^\infty$ is a ranking supermartingale w.r.t. the termination time $T$.
\end{proposition}

\begin{proof}
To prove \cref{prop:rsm}, we need to check the items below.
  \begin{itemize}
      \item $X_n\ge 0$ for all $n\in \Nset_0$. Since each $X_n$ is defined by $\eta$ (see \cref{def:rsm-map}) and $\eta(s)\ge 0$ for all states $s\in S$, it follows that $X_n\ge 0$ for all $n\Nset_0$.
      \item ${\mathbb E}[X_{n+1}|{\cal F}_n] \leq X_n-\epsilon \cdot {\mathbb I}_{T>n}$ for all $n\in\Nset_0$. To prove this inequality, we consider two cases. First, when $T>n$, we have that $X_n=\eta(\overline{s}_n)$. And we observe that $\expv [X_{n+1}\mid \mathcal{F}_n]=\expv_{\delta\sim\mu} [\eta(f(\overline{s}_n,\pi(\overline{s}_n+\delta)))]\le \eta(\overline{s}_n)-\epsilon=X_n-\epsilon$. Second, when $T\le n$, we have that $X_n=\eta(\overline{s}_T)$. Then we can obtain that $\expv [X_{n+1}\mid \mathcal{F}_n]=\eta(\overline{s}_T)=X_n$. According to the two cases, we can conclude that ${\mathbb E}[X_{n+1}|{\cal F}_n] \leq X_n-\epsilon \cdot {\mathbb I}_{T>n}$ for all $n\in\Nset_0$.
  \end{itemize}  
  Hence, we prove that $\{X_n\}_{n=0}^\infty$ is a ranking supermartingale w.r.t. $T$.
\end{proof}

We define another stochastic process $\{Y_n\}_{n=0}^{\infty}$ such that $Y_n:=X_n+\epsilon \cdot \mathrm{min}\{n,T\}$.

\begin{proposition}
    $\{Y_n\}_{n=0}^\infty$ is a supermartingale, and $Y_{n+1}-Y_n\in [ a''+\epsilon,b''+\epsilon]$ almost surely for all $n\in\Nset_0$.
\end{proposition}

\begin{proof}
  \begin{align*}
\expv [Y_{n+1}\mid \mathcal{F}_n] &= \expv [X_{n+1}+\epsilon \cdot \mathrm{min}\{n+1,T\}\mid \mathcal{F}_n]  \\
&= \expv [X_{n+1} \mid \mathcal{F}_n ]+\epsilon \cdot \expv [\mathrm{min}\{n+1,T\} 
 \mid\mathcal{F}_n]  \\
 &\le X_n-\epsilon\cdot {\mathbb I}_{T> n} +  \epsilon \cdot \expv [\mathrm{min}\{n+1,T\} \mid\mathcal{F}_n]  \\
\end{align*}
Then we have two cases:
\begin{itemize}
    \item If $T>n$, then $\mathrm{min}\{n+1,T\}=n+1$ and $\expv [\mathrm{min}\{n+1,T\} \mid\mathcal{F}_n]=n+1$. We can thus derive that
    \begin{align*}
\expv [Y_{n+1}\mid \mathcal{F}_n] &= \expv [X_{n+1}+\epsilon \cdot \mathrm{min}\{n+1,T\}\mid \mathcal{F}_n]  \\
 &\le X_n-\epsilon\cdot {\mathbb I}_{T>n} +  \epsilon \cdot \expv [\mathrm{min}\{n+1,T\} \mid\mathcal{F}_n]  \\
 &= X_n-\epsilon+\epsilon\cdot (n+1)  \\
 &= X_n+\epsilon \cdot n \\
 &=  Y_n
\end{align*}
    \item If $T\le n$, then $\mathrm{min}\{n+1,T\}=T=T\cdot {\mathbb I}_{T\le n} $. Since the event $T\cdot {\mathbb I}_{T\le n}$ is measurable in $\mathcal{F}_n$, $\expv [T\cdot {\mathbb I}_{T\le n}\mid \mathcal{F}_n]=T\cdot {\mathbb I}_{T\le n}$. We can thus derive that
    \begin{align*}
\expv [Y_{n+1}\mid \mathcal{F}_n] &= \expv [X_{n+1}+\epsilon \cdot \mathrm{min}\{n+1,T\}\mid \mathcal{F}_n]  \\
 &\le X_n-\epsilon\cdot {\mathbb I}_{T> n} +  \epsilon \cdot \expv [\mathrm{min}\{n+1,T\} \mid\mathcal{F}_n]  \\
 &= X_n-0+\epsilon\cdot T  \\
 &=  Y_n
\end{align*}    
\end{itemize}
Hence, $\{Y_n\}_{n=0}^\infty$ is a supermartingale. Moreover, since $T\le n$ implies $X_{n+1}=X_n$, we have that $X_{n+1}-X_n={\mathbb I}_{T> n}\cdot (X_{n+1}-X_n)$. Thus, we can obtain that
\begin{align*}
    Y_{n+1}-Y_n &= X_{n+1}-X_n+ \epsilon \cdot (\mathrm{min}\{n+1,T\}-\mathrm{min}\{n,T \}) \\
    &= X_{n+1}-X_n+ \epsilon \cdot {\mathbb I}_{T> n} \\
    &= {\mathbb I}_{T> n}\cdot (X_{n+1}-X_n+ \epsilon)
\end{align*}
where the second equality is derived by the fact that $\mathrm{min}\{n+1,T\}-\mathrm{min}\{n,T \}={\mathbb I}_{T> n}$. It follows that $Y_{n+1}-Y_n \in [a''+\epsilon, b''+\epsilon]$.  
\end{proof}

\begin{proposition}\label{prop:concentration}
 Consider a perturbed DRL system $M_\mu$ satisfying our model assumptions. If $M_\mu$ has a difference-bounded ranking supermartingale map $\eta$ w.r.t. $[a'',b'']$, then $M_\mu$ has the concentration property, i.e., there exist two constant $a,b>0$ such that for any initial state $s_0\in S_0$ and sufficiently large $n$, $\probm_{s_0}(T>n)\le a\cdot \mathrm{exp}(-b\cdot n)$.   
\end{proposition}

\begin{proof}
The proof is slightly different from that in \cite{stab_martingales} as we use Hoeffding's Inequality instead of Azuma's Inequality. Let $\nu=\epsilon\cdot n-X_0$ and $\hat{\nu}=\epsilon\cdot\mathrm{min}\{n,T\}-X_0$. Note that $\nu=\hat{\nu}$ whenever $T>n$. Then we have that
\begin{align*}
 \probm_{s_0}(T>n)&=\probm_{s_0}(X_n\ge 0\wedge T>n)  \\
 &=\probm_{s_0}((X_n+\hat{\nu}\ge \nu)\wedge T>n) \\
 &\le \probm_{s_0}(X_n+\hat{\nu}\ge \nu) \\
 &= \probm_{s_0}(Y_n-Y_0\ge \epsilon\cdot n-X_0) \\
 &\le \mathrm{exp}\left(\frac{-2\cdot (\epsilon\cdot n-\eta(\overline{s}_0))^2}{n\cdot (b''-a'')^2} \right) \\
 &= \mathrm{exp}\left(-\frac{2\cdot\epsilon^2}{(b''-a'')^2}\cdot n+ \frac{2\cdot\epsilon\cdot\eta(\overline{s}_0)}{(b''-a'')^2}-\frac{\eta^2(\overline{s}_0)}{n\cdot (b''-a'')^2}  \right)   \\
 &\le  \mathrm{exp}\left(-\frac{2\cdot\epsilon^2}{(b''-a'')^2}\cdot n+ \frac{2\cdot\epsilon\cdot\eta(\overline{s}_0)}{(b''-a'')^2} \right)   \\
 &= a\cdot \mathrm{exp}(-b\cdot n)\\
\end{align*}
for all $n>\frac{X_0}{\epsilon}$,
where $a:=\mathrm{exp}\left(\frac{2\cdot\epsilon\cdot\eta(\overline{s}_0)}{(b''-a'')^2}\right)$ and
$b:=\frac{2\epsilon^2}{(b''-a'')^2}$.

\end{proof}

\noindent\textbf{\cref{thm:tail-bounds}} (Tail Bounds for Cumulative Rewards).
Suppose that an $M_\mu$ has the concentration property and a difference-bounded URS (\emph{resp.} LRS) $h$ with bounds  $K,K'\in\Rset$. Given an initial state $s_0\in S_0$, if a reward $c>h(s_0)-K$ (resp. $c<h(s_0)-K'$), we have 
\begin{align}
&\probm_{s_0}(\calR\ge c)\le \alpha+ \beta\cdot \mathrm{exp}(-\gamma\cdot c^2 ) \label{eq:tail-1}\\ (\emph{resp.}\ & \probm_{s_0}(\calR\le c)\le \alpha+ \beta\cdot \mathrm{exp}(-\gamma\cdot c^2 ) \label{eq:tail-2}
\end{align} 
where, $\alpha,\beta,\gamma$ are positive constants derived from  $M_\mu$, the concentration property and $h$, respectively.

\begin{proof}
(Tail bound of $\probm_{s_0}(\calR\ge c)$).
We define a stochastic process $\{Y_n\}_{n=0}^{\infty}$ by $Y_n := h(\overline{s}_n) + \sum_{i=0}^{n}r_i$. Since $h$ is an URS, by \cref{prop:URS}, we have that $\{Y_n\}_{n=0}^{\infty}$ is a supermartingale. Then by the difference-bounded property of $h$ (\cref{def:diff-bound}), we can derive that
\begin{align*}
   Y_{n+1}-Y_n &= h(\overline{s}_{n+1})+\sum_{i=0}^{n+1} r_i - h(\overline{s}_n)-\sum_{i=0}^{n} r_i\\
   &= h(\overline{s}_{n+1})-h(\overline{s}_n) +r_{n+1} \\
   &\le h(\overline{s}_{n+1})-h(\overline{s}_n)+ r_{\mathrm{max}} \\
  &\in [a',b']
\end{align*}
where $a':=-m+ r_{\mathrm{max}}$, $b':=m+ r_{\mathrm{max}}$. 
For brevity, below we write $h_n$ for $h(\overline{s}_n)$.
Given any real number $\lambda>0$ and stopping time $T$, we can obtain that: 
\begin{linenomath*}
\begin{equation}
\begin{split}
	&Y_T-Y_0 \geq \lambda \\
	\iff & h_T + \textstyle\sum_{i=0}^{T}r_i -Y_0 \geq \lambda\\
	\iff & h_T + \calR -h_0 \geq \lambda\\
	\iff & \calR \geq h_0-h_T + \lambda\\
\end{split}
\nonumber 
\end{equation}
\end{linenomath*}

Let $c:=h_0-K+\lambda$ and thus $c>h_0-K$. Choose a sufficiently large real number $n^*$ satisfying the concentration property of $M_\mu$, i.e., $\probm_{s_0}(T> n^*)\le a\cdot \mathrm{exp}(-b\cdot n)$ with two constants $a,b>0$ (see \cref{prop:concentration}).
Then for any state $s_0\in S_0$, the tail bound of $\cal R$ w.r.t. $c$  can be deduced as follows:
\begin{align*}
  \probm_{s_0} (\calR\ge c) &=\probm_{s_0}(\calR\ge h_0-K+\lambda)  \\
  &\le \probm_{s_0}(\calR\ge h_0-h_T+\lambda) \\
  &=  \probm_{s_0}(Y_T-Y_0\ge \lambda )  \\
  &= \probm_{s_0}(Y_T-Y_0\ge\lambda\wedge T>n^* ) \\
  &~ +\probm_{s_0}(Y_T-Y_0\ge\lambda \wedge T\le n^*) \\
  &\le \probm_{s_0}(T>n^*)+\probm_{s_0}(Y_T-Y_0\ge\lambda \wedge T\le n^*) \\
    &\le \probm_{s_0}(T>n^*)+\probm_{s_0}(Y_{1}-Y_0\ge \lambda)  \\
  &\le a\cdot \mathrm{exp}(-b\cdot n^*)+\mathrm{exp}\left(-\frac{2\lambda^2}{ (b'-a')^2} \right) \\
  &= a\cdot \mathrm{exp}(-b\cdot n^*)+\mathrm{exp}\left(-\frac{2\cdot (c-h_0+K')^2}{(b'-a')^2} \right)  \\
  &\le  \alpha+\beta\cdot \mathrm{exp}(-\gamma\cdot c^2) \\
\end{align*}
where the third inequality is obtained by the concentration property and Hoeffding's Inequality on supermartingales (\cref{thm:hoeffding}), $\alpha=a\cdot \mathrm{exp}(-b\cdot n^*)$, $\beta=\mathrm{exp}\left( \frac{4(h_0-K)^2}{(b'-a')^2} \right)$, and $\gamma=\frac{2}{(b'-a')^2}$.

(Tail bound of $\probm_{s_0}(\calR\le c')$). The proof is similar to that above.
We define a stochastic process $\{Y_n\}_{n=0}^{\infty}$ by $Y_n := h(\overline{s}_n) + \sum_{i=0}^{n}r_i$ where $h$ is a LRS. By \cref{prop:LRS}, we have that $\{Y_n\}_{n=0}^{\infty}$ is a submartingale. Then by the difference-bounded property of $h$ (\cref{def:diff-bound}), we can derive that
\begin{align*}
   Y_{n+1}-Y_n &= h(\overline{s}_{n+1})+\sum_{i=0}^{n+1} r_i - h(\overline{s}_n)-\sum_{i=0}^{n} r_i\\
   &= h(\overline{s}_{n+1})-h(\overline{s}_n) +r_{n+1} \\
   &\le h(\overline{s}_{n+1})-h(\overline{s}_n)+ r_{\mathrm{max}} \\
  &\in [a',b']
\end{align*}
where $a':=-m+ r_{\mathrm{max}}$, $b':=m+ r_{\mathrm{max}}$. 
For brevity, below we write $h_n$ for $h(\overline{s}_n)$.
Given any real number $\lambda'<0$ and stopping time $T$, we can obtain that: 
\begin{linenomath*}
\begin{equation}
\begin{split}
	&Y_T-Y_0 \leq \lambda' \\
	\iff & h_T + \textstyle\sum_{i=0}^{T}r_i -Y_0 \leq \lambda' \\
	\iff & h_T + \calR -h_0 \leq \lambda'\\
	\iff & \calR \leq h_0-h_T + \lambda'\\
\end{split}
\nonumber 
\end{equation}
\end{linenomath*}

Let $c':=h_0-K'+\lambda'$ and thus $c<h_0-K'$. Choose a sufficiently large real number $n^*$ satisfying the concentration property of $M_\mu$, i.e., $\probm_{s_0}(T> n^*)\le a\cdot \mathrm{exp}(-b\cdot n)$ where $a,b>0$ are two positive constants (see \cref{prop:concentration}).
Then for any state $s_0\in S_0$, the tail bound of $\cal R$ w.r.t. $c'$  can be deduced as follows:
\begin{align*}
  \probm_{s_0} (\calR\le c') &=\probm_{s_0}(\calR\le h_0-K'+\lambda')  \\
  &\le \probm_{s_0}(\calR\le h_0-h_T+\lambda') \\
  &=  \probm_{s_0}(Y_T-Y_0\le \lambda' )  \\
  &= \probm_{s_0}(Y_T-Y_0\le\lambda'\wedge T>n^* ) \\
  &~ +\probm_{s_0}(Y_T-Y_0\le\lambda' \wedge T\le n^*) \\
  &\le \probm_{s_0}(T>n^*)+\probm_{s_0}(Y_T-Y_0\le\lambda' \wedge T\le n^*) \\
  &\le \probm_{s_0}(T>n^*)+\probm_{s_0}(Y_{1}-Y_0\le \lambda')  \\
 &\le a\cdot \mathrm{exp}(-b\cdot n) +\mathrm{exp}\left(-\frac{2\lambda'^2}{ (b'-a')^2} \right) \\
  &= a\cdot \mathrm{exp}(-b\cdot n^*)+\mathrm{exp}\left(-\frac{2\cdot (c'-h_0+K')^2}{(b'-a')^2} \right)  \\
  &\le  \alpha'+\beta'\cdot \mathrm{exp}(-\gamma'\cdot c'^2) \\
\end{align*}
where the third inequality is obtained by the concentration property and Hoeffding's Inequality on submartingales (\cref{thm:hoeffding}), $\alpha'=a\cdot \mathrm{exp}(-b\cdot n^*)$, $\beta'=\mathrm{exp}\left( \frac{4(h_0-K')^2}{(b'-a')^2} \right)$, and $\gamma'=\frac{2}{(b'-a')^2}$.

\end{proof}

\subsection{Proofs of \cref{theorem:verify_RM}}

\noindent\textbf{\cref{theorem:verify_RM}}. 
Given an $M_{\mu}$ and a function $h:S\rightarrow\Rset$, 
we have 
$pre_h(s)\le h(s)$ for any state $s\in \overline{S_g}$ if the formula below 
\begin{align*}
\expv_{\delta\sim \mu}[h(f(\tilde{s},\pi(\tilde{s}+\delta)))]\le h(\tilde{s})-\zeta \quad\quad\quad\quad\quad\quad (\ref{formula:verify_URS})
\end{align*} holds 
for any state $\tilde{s}\in \tilde{S}\cap  \overline{S_g}$, where $\zeta=r_{\mathrm{max}}+ \tau\cdot L_h\cdot (1+L_f\cdot (1+L_\pi))$ with  $L_f,L_\pi,L_h$ being the Lipschitz constants of $f,\pi,h$, 
and $r_{\mathrm{max}}$ being the maximum value of $R$, respectively. \\
Analogously, we have that 
$pre_h(s)\ge h(s)$ for any state $s\in \overline{S_g}$ if 
the formula below 
\begin{align*}
\expv_{\delta\sim \mu}[h(f(\tilde{s},\pi(\tilde{s}+\delta)))]\ge h(\tilde{s})-\zeta'
\quad\quad\quad\quad\quad\quad (\ref{formula:verify_LRS})
\end{align*} holds 
for any state $\tilde{s}\in \tilde{S}\cap \overline{S_g}$, where $\zeta'=r_{\mathrm{min}}-\tau\cdot L_h\cdot (1+L_f\cdot (1+L_\pi))$ with $r_{\mathrm{min}}$ being the minimum  value of $R$.

\begin{proof}
 Let $L_f,L_\pi,L_h$ be the Lipschitz constants for the system dynamics $f$, the trained policy $\pi$ and the neural network function $h$, respectively. 
    Given a state $s\in \overline{S_g}$, let $\tilde{s}$ be such that $||\tilde{s}-s||_1\le \tau$. 
    
    To prove \cref{formula:verify_URS}, by the Lipschitz continuities, we have that
\begin{align*}
pre_h(s)=& r+\expv_{\delta\sim \mu}[h(f(s,\pi(s+\delta)))] \\
\le & r_{\mathrm{max}}+\expv_{\delta\sim \mu}[h(f(\tilde{s},\pi(\tilde{s}+\delta)))]
\\ & +||f(\tilde{s},\pi(\tilde{s}+\delta))-f(s,\pi(s+\delta))||_1\cdot L_h \\
\le& r_{\mathrm{max}}+\expv_{\delta\sim \mu}[h(f(\tilde{s},\pi(\tilde{s}+\delta)))]
\\ &~+||(\tilde{s},\pi(\tilde{s}+\delta))-(s,\pi(s+\delta))||_1\cdot L_h\cdot L_f   \\
\le& r_{\mathrm{max}}+\expv_{\delta\sim \mu}[h(f(\tilde{s},\pi(\tilde{s}+\delta)))]
 \\ &~+ ||\tilde{s}-s||_1\cdot L_h\cdot L_f \cdot (1+L_\pi) \\
\le& r_{\mathrm{max}}+\expv_{\delta\sim \mu}[h(f(\tilde{s},\pi(\tilde{s}+\delta)))]
+ \tau\cdot L_h\cdot L_f \cdot (1+L_\pi)
\end{align*}
and 
\begin{align*}
-h(s)\le h(\tilde{s})+||\tilde{s}-s||_1\cdot L_h\le -h(\tilde{s})+\tau\cdot L_h,
\end{align*}
where $r_{\mathrm{max}}$ is the maximum value of the reward function $R$.
Thus, we can derive that
\begin{align*}
pre_h(s)-h(s)&\le \expv_{\delta\sim \mu}[h(f(\tilde{s},\pi(\tilde{s}+\delta)))]-h(\tilde{s})
\\ &~ +r_{\mathrm{max}} +\tau\cdot L_h\cdot L_f \cdot (1+L_\pi) +\tau\cdot L_h\\
&\le   \expv_{\delta\sim \mu}[h(f(\tilde{s},\pi(\tilde{s}+\delta)))]-h(\tilde{s}) +\zeta \\
&\le 0  
\end{align*}

To prove \cref{formula:verify_LRS}, by the Lipschitz continuities, we have that
\begin{align*}
pre_h(s)=& r+\expv_{\delta\sim \mu}[h(f(s,\pi(s+\delta)))] \\
\ge & r_{\mathrm{min}}+\expv_{\delta\sim \mu}[h(f(\tilde{s},\pi(\tilde{s}+\delta)))]
\\ & -||f(\tilde{s},\pi(\tilde{s}+\delta))-f(s,\pi(s+\delta))||_1\cdot L_h \\
\ge& r_{\mathrm{min}}+\expv_{\delta\sim \mu}[h(f(\tilde{s},\pi(\tilde{s}+\delta)))]
\\ &~-||(\tilde{s},\pi(\tilde{s}+\delta))-(s,\pi(s+\delta))||_1\cdot L_h\cdot L_f   \\
\ge& r_{\mathrm{min}}+\expv_{\delta\sim \mu}[h(f(\tilde{s},\pi(\tilde{s}+\delta)))]
 \\ &~- ||\tilde{s}-s||_1\cdot L_h\cdot L_f \cdot (1+L_\pi) \\
\ge& r_{\mathrm{min}}+\expv_{\delta\sim \mu}[h(f(\tilde{s},\pi(\tilde{s}+\delta)))]
- \tau\cdot L_h\cdot L_f \cdot (1+L_\pi)
\end{align*}
and 
\begin{align*}
-h(s)\ge -h(\tilde{s})-||s-\tilde{s}||_1\cdot L_h\ge -h(\tilde{s})-\tau\cdot L_h,
\end{align*}
where $r_{\mathrm{min}}$ is the minimum value of the reward function $R$.
Thus, we can derive that
\begin{align*}
pre_h(s)-h(s)&\ge \expv_{\delta\sim \mu}[h(f(\tilde{s},\pi(\tilde{s}+\delta)))]-h(\tilde{s})
\\ &~ +r_{\mathrm{min}} -\tau\cdot L_h\cdot L_f \cdot (1+L_\pi) -\tau\cdot L_h\\
&\ge   \expv_{\delta\sim \mu}[h(f(\tilde{s},\pi(\tilde{s}+\delta)))]-h(\tilde{s}) +\zeta' \\
&\ge 0  
\end{align*}

\end{proof}

\section{Implementation Details and Additional Experimental Results}

\subsection{Experimental Environment}

\begin{table}[h!]
	\vspace{-4mm}
	\centering
	\footnotesize
	\setlength{\tabcolsep}{2.7pt}
	\begin{tabular}{l|l }
		\hline
		\textbf{Hyperparameter} &  \textbf{Value} \\
		\hline
            Neural network size & $3 \times 200$ \\
            Activation function & $Relu$\\
		Learning rate  &  $1e-3$ \\
		Optimizer &  $Adam$\\
            Weight decay & $1.5e-3$\\ 
            Timeout threshold & $60 \  minutes$ \\
            Boundedness parameter $K$ & -0.01\\
            Boundedness parameter $K'$ & $ \ $0.01 \\
            Loss coefficient  $k_1, k_1'$ & 1\\
            Loss coefficient  $k_2, k_2'$ & \\
              & 0.01 $\ \ \ $ CP \\
              & 0.005 $\ \ $ B1\\
              & 0.007 $\ \ $ MC\\
              & 0.05 $\ \ \ $ B2\\
            Loss coefficient  $k_3, k_3'$ & 1\\
             Number of partition cells $k$ & 10 \\
		\hline
	\end{tabular}
	\captionsetup{skip=5pt}
	\caption{Hyperparametes  for training  and validating reward martingales.}
	\vspace{-4mm}
	\label{table:Hyperpar}
\end{table}

We conducted experiments on a workstation running Ubuntu  18.04 with a 32-core AMD Ryzen Threadripper CPU and 128GB RAM.  We show a  list of ordinary hyperparameters for training and validating reward martingales in \Cref{table:Hyperpar}, and discuss the effects of other hyperparameters of interest below. 
Moreover, the Lipschitz constants used in our theorems and algorithms are calculated using the same method as those in \cite{stab_martingales}.

\subsection{Analysis of Hyperparameters}
\paragraph{Neccessity of the third loss terms.}
The hyperparameters $\overline{u}$ and $ \underline{l}$ in the  third loss terms $\calL_{C3},\calL_{C3'}$ are used to enforce the upper and lower bounds calculated by reward martingales as tight as possible.
For each perturbation and policy, we execute the trained systems by 200 episodes starting from different initial states, and employ the best cumulative rewards plus a constant $0.1$ and the worst cumulative rewards subtraction a constant $0.1$ as $\overline{u}$ and $\underline{l}$, correspondingly. 

We train reward martingales with  and without $\calL_{C3}$   (resp. $\calL_{C3'}$),  and the comparison between them is shown in \Cref{fig:out_bound}. The policies of CP and B2 are trained on abstract states, and those of MC and  B1 are trained on concrete states,  correspondingly. As we can see, the upper and lower bounds in the left subfigures (i.e., the bounds trained without the third loss term) are looser than the bounds in the right subfigures (i.e., the bounds trained with the third loss term). This proves the necessity of the existence of the two heuristic loss terms. Although the first two terms in the loss functions can enforce the candidate reward martingales to be a URS or LRS, the lack of tightness may make the results trivial.

\begin{figure}[h!]
	\footnotesize 
		\begin{subfigure}[b]{0.23\textwidth}
		\includegraphics[width=\textwidth]{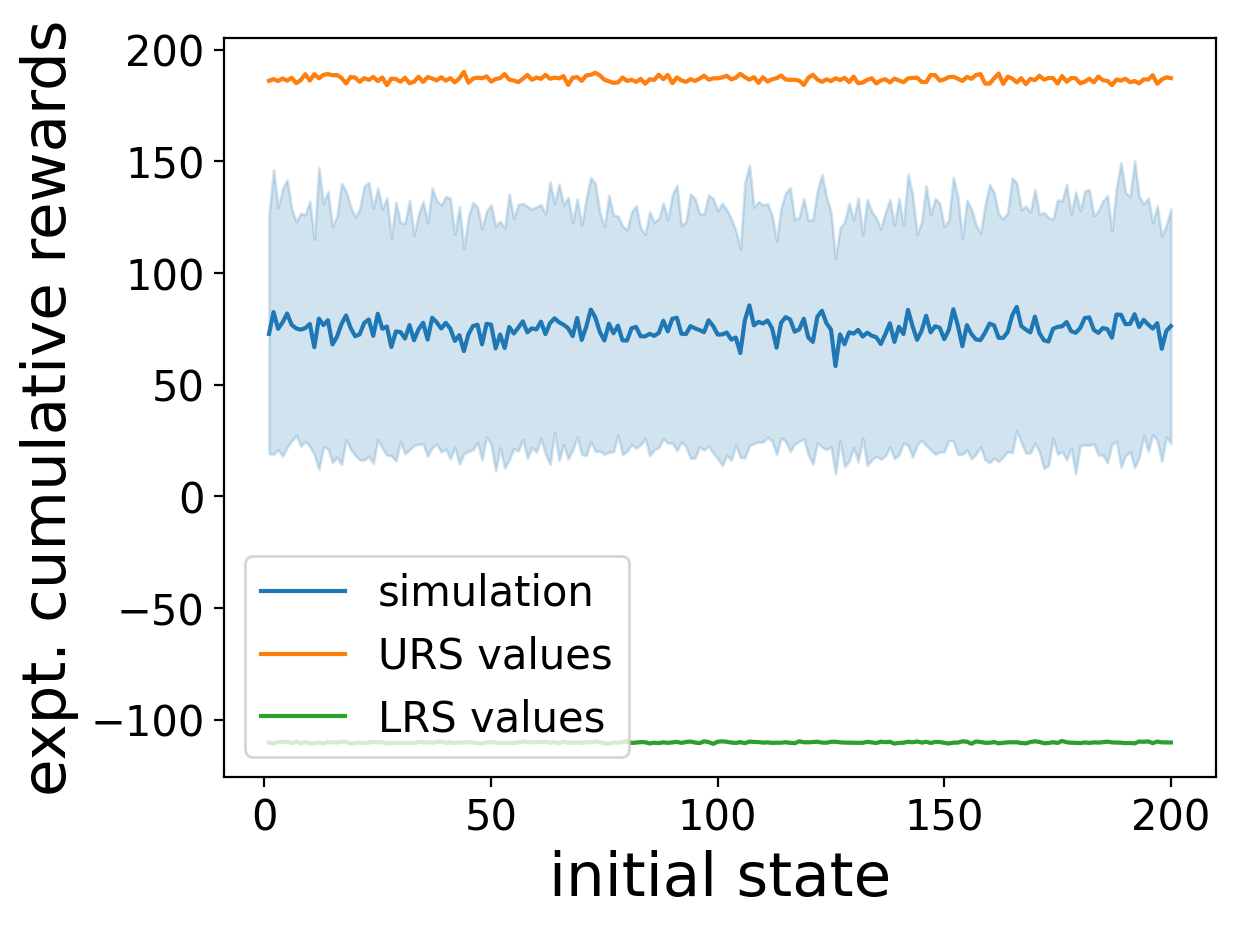}
		\caption{CP($r=0.2$ w/o $\calL_{C3/C3'}$)}
	\end{subfigure}\!
            \begin{subfigure}[b]{0.23\textwidth}
		\includegraphics[width=\textwidth]{imgs/CP_u_0.2_ok.png}
		\caption{CP($r=0.2$)}
	\end{subfigure}\\
	\begin{subfigure}[b]{0.23\textwidth}
		\includegraphics[width=\textwidth]{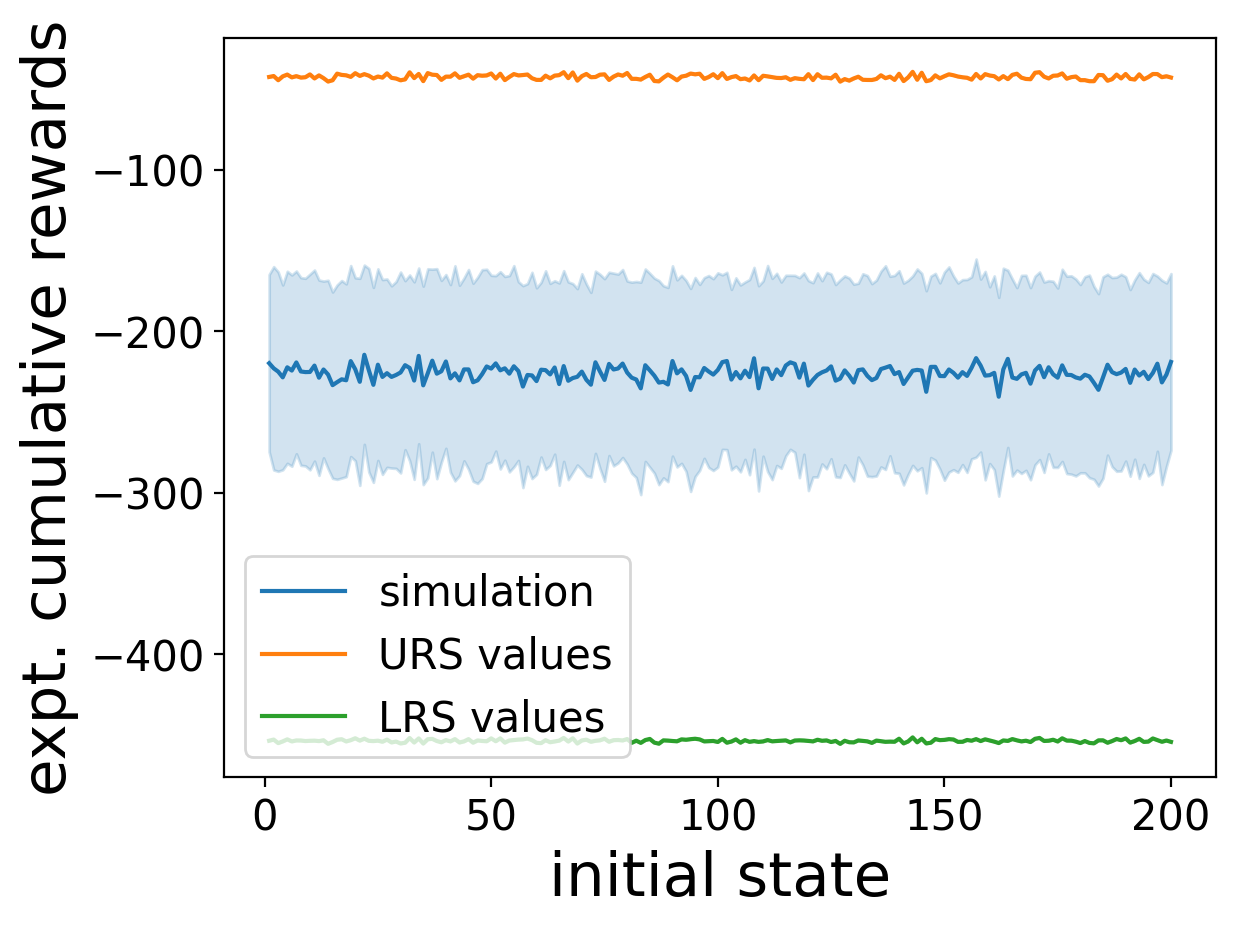}
		\caption{B1 ($\sigma=0.3$ w/o $\calL_{C3/C3'})$}
	\end{subfigure}\!
        \begin{subfigure}[b]{0.23\textwidth}
		\includegraphics[width=\textwidth]{imgs/B1_g_0.3_ok_DNN.png}
		\caption{B1 ($\sigma=0.3$)}
	\end{subfigure}\\
        \begin{subfigure}[b]{0.23\textwidth}
		\includegraphics[width=\textwidth]{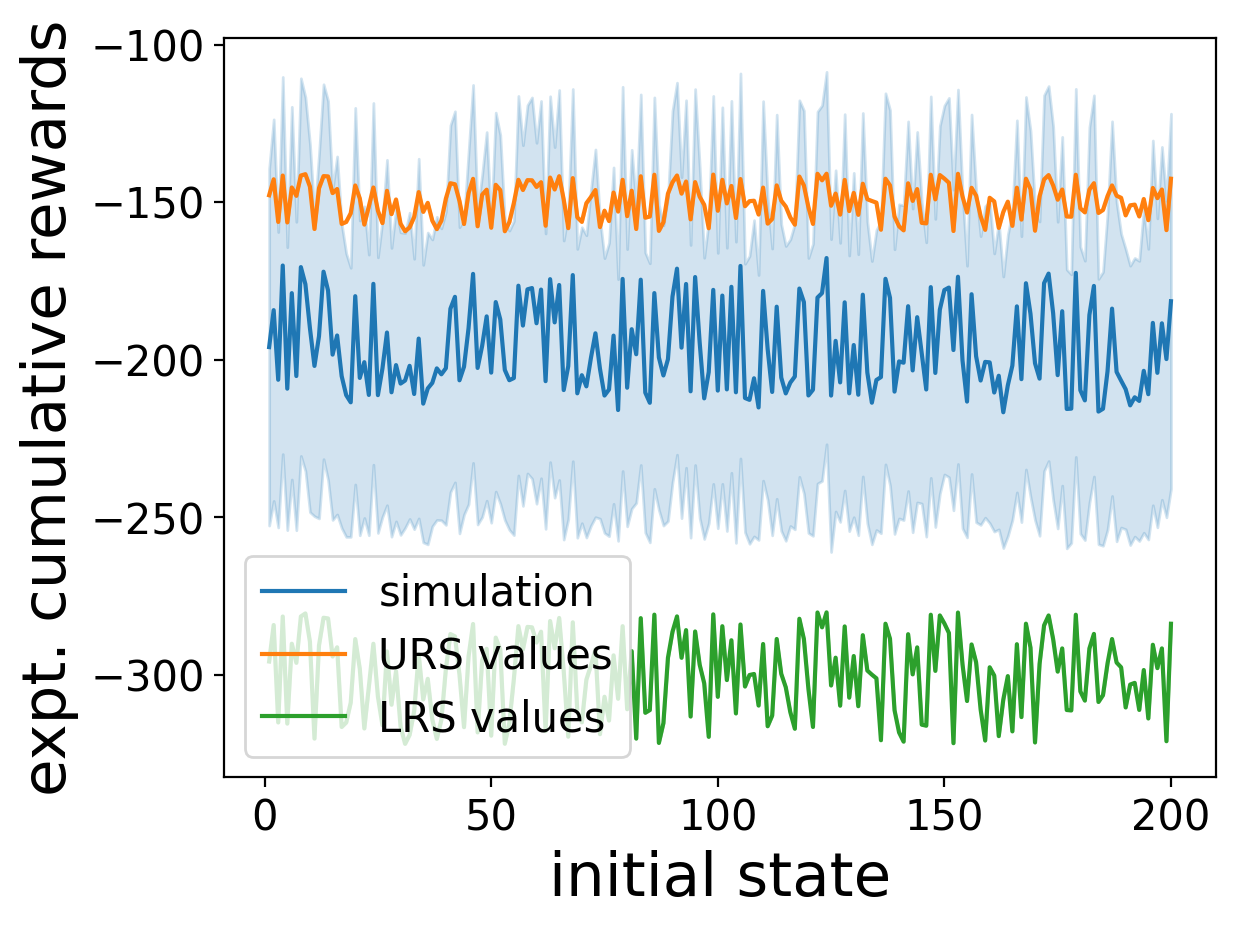}
		\caption{MC ($\sigma=0.08$  w/o $\calL_{C3/C3'})$}
	\end{subfigure}\!
        \begin{subfigure}[b]{0.23\textwidth}
		\includegraphics[width=\textwidth]{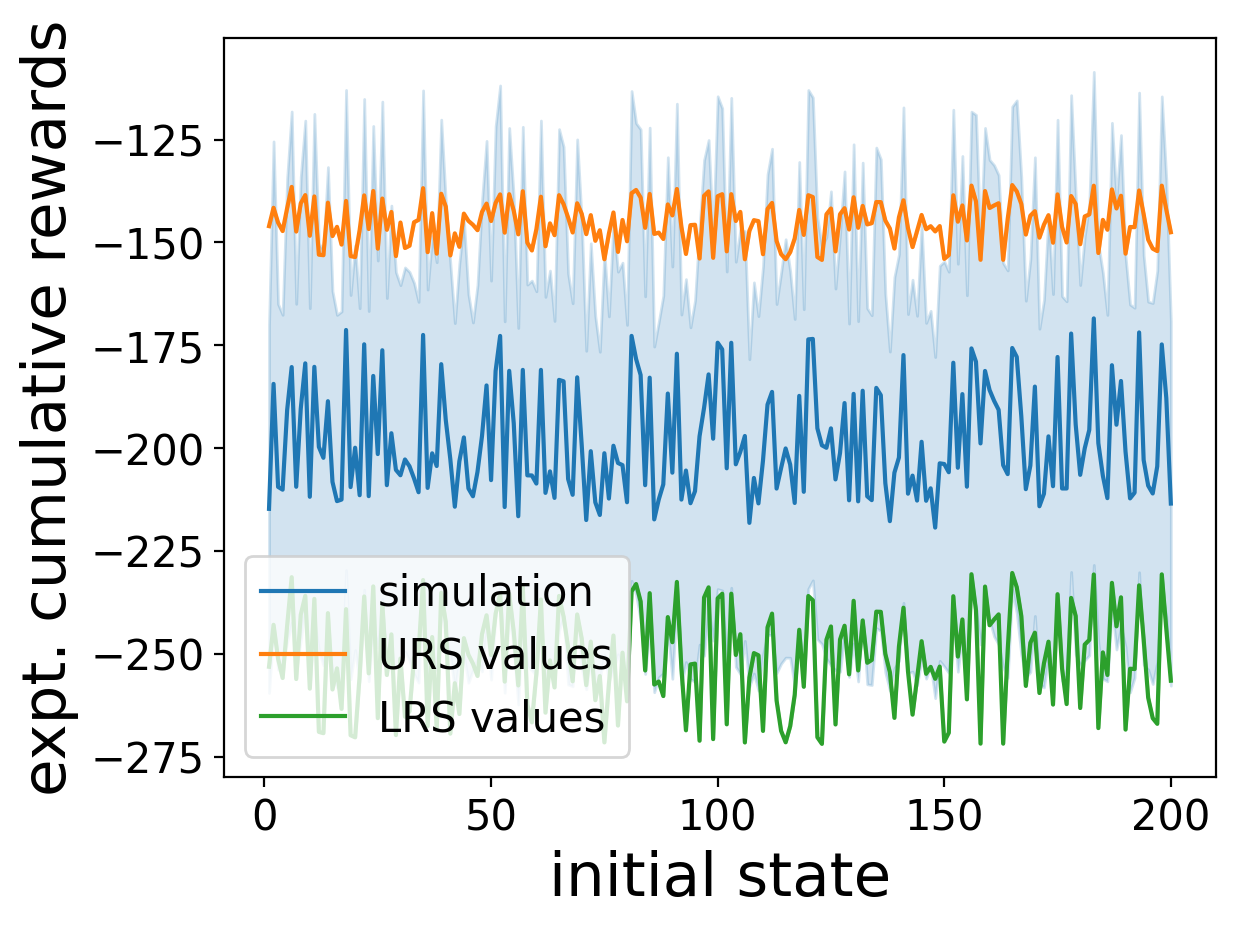}
		\caption{MC ($\sigma=0.08$ )}
	\end{subfigure}\\
	\begin{subfigure}[b]{0.23\textwidth}
		\includegraphics[width=\textwidth]{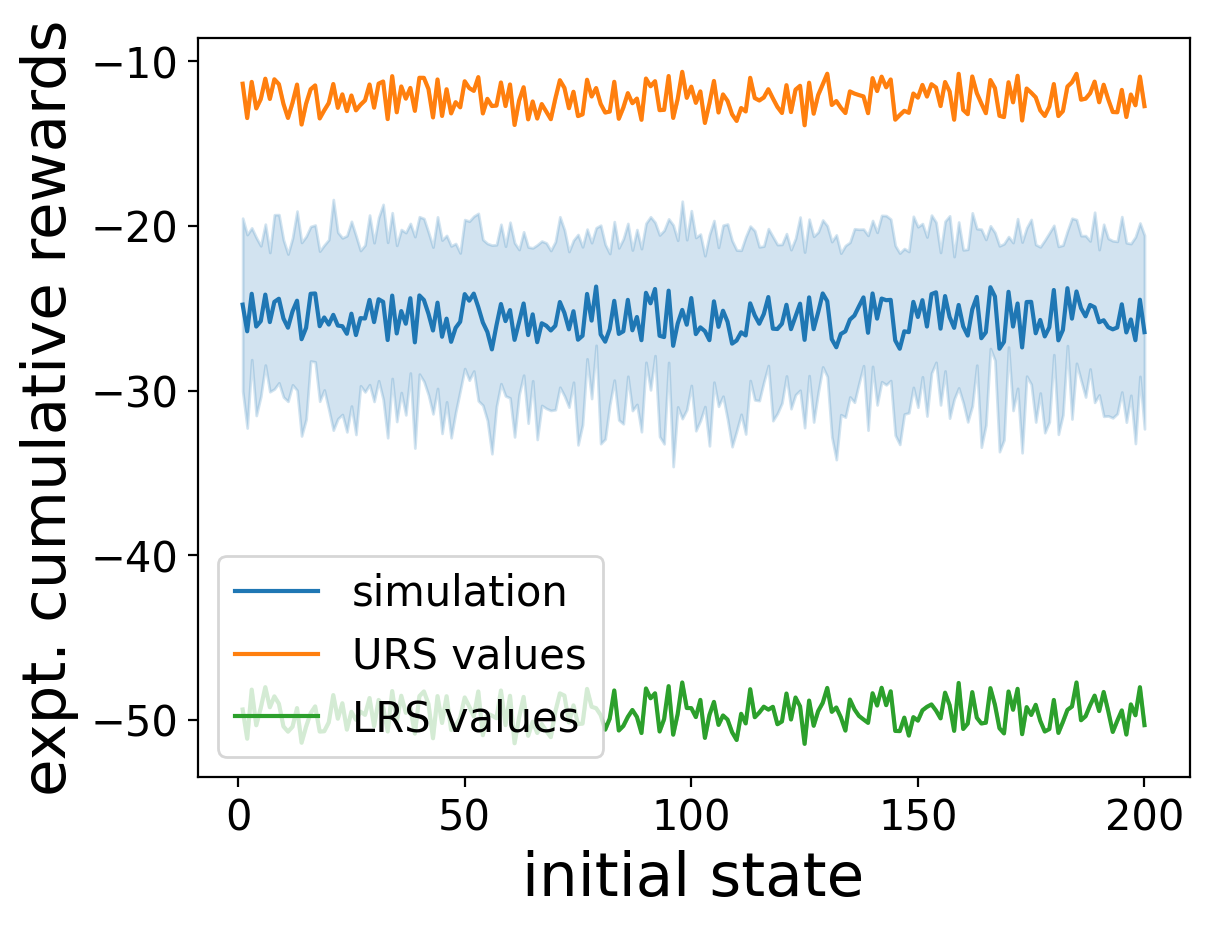}
		\caption{B2 ($r=0.1$  w/o $\calL_{C3/C3'}$)}
	\end{subfigure}\!
 \begin{subfigure}[b]{0.23\textwidth}
		\includegraphics[width=\textwidth]{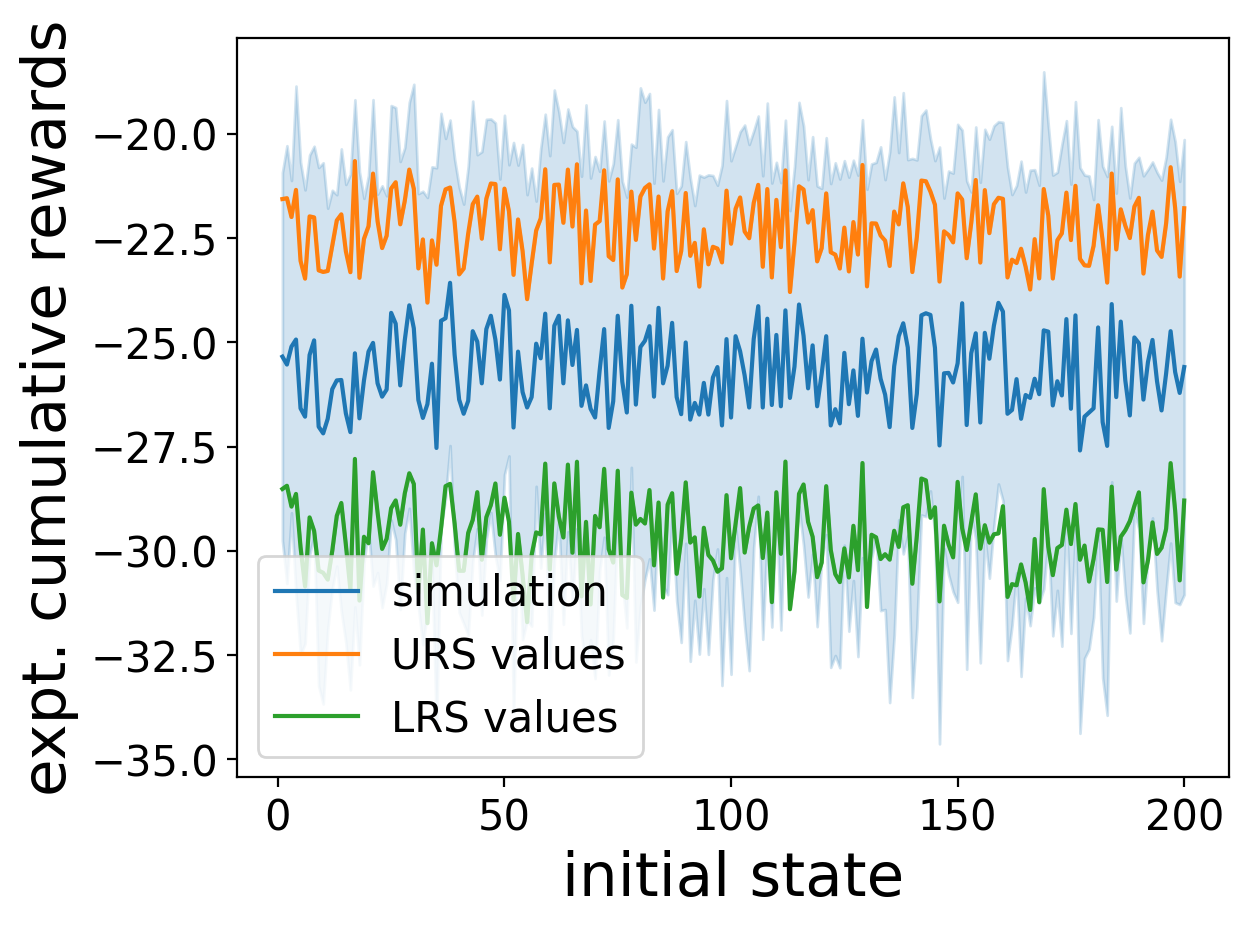}
		\caption{B2 ($r=0.1$)}
	\end{subfigure}
	\vspace{-2mm}
	\caption{ Comparison between certified bounds and simulation results under different perturbations  with and without  tightness constraint.}
	\vspace{-4mm}
	\label{fig:out_bound}
\end{figure}

\paragraph{Analysis of granularity $\tau$ and refinement step length $\xi$. }

\begin{table}[t]
	\centering
	\footnotesize
	\setlength{\tabcolsep}{4pt}
	\begin{tabular}{l|l| l| l| c | l l l}
	\hline
	\textbf{Task} & \centering \textbf{Pert.} & $\tau$ & $\xi$     & \textbf{Iters.} & \textbf{T.T.} & \textbf{V.T.} &     \textbf{To.T.} \\ \hline
        \multirow{2}{*}{\textbf{CP}} & \multirow{2}{*}{$\sigma=0.1$} & 0.04 & 0.001 & 2 &\multicolumn{3}{c}{$timeout$} \\
          &   & 0.02  &  0.002 & 1 & 1203&873 & 2076 \\
          \hline
        \multirow{2}{*}{\textbf{B1}} & \multirow{2}{*}{$\sigma=0.3$} & 0.04  & 0.0015 & 5 & \multicolumn{3}{c}{$timeout$} \\
          &   & 0.02  & 0.002 & 2 &898 & 529 & 1427\\
          \hline
        \multirow{2}{*}{\textbf{MC}} & \multirow{2}{*}{$r=0.1$} & 0.014 & 0.002& 3  &  1952 & 1532  &  3484 \\
          &   & 0.01  & 0.002 & 1 & 712 & 528 & 1240 \\
          \hline
        \multirow{2}{*}{\textbf{B2}} & \multirow{2}{*}{$r=0.1$} & 0.04 & 0.001 & 6 & \multicolumn{3}{c}{$timeout$} \\
          &   & 0.02  & 0.002 & 2 & 917 & 502 & 1419 \\
          \hline
	\end{tabular}
	\captionsetup{skip=2pt}
	\caption{Time on training and validating URS under different $\tau$ and $\xi$ for policies with C.S.}
	\vspace{-2mm}
	\label{table:time_tau_xi}
\end{table}

\Cref{table:time_tau_xi} shows the training and validating  time for URS with different $\tau$ and $\xi$ for policies trained on concrete states.

As we can see, the tasks with a smaller initial granularity  $\tau$ and a possibly larger refinement step length $\xi$
complete the training and validating in fewer iterations, e.g., 
B2 with $\tau =0.02,\xi=0.002$. 
On the contrary, the tasks with a larger initial granularity $\tau$ and a possibly smaller refinement step length $\xi$ will increase the iterations of the 'training-validating', which leads to an increase in total elapsed time (e.g., MC with $\tau=0.014,\xi=0.002$) and even timeout (e.g., CP with $\tau=0.04,\xi=0.001$).
This is because by using a smaller initial granularity $\tau$ and a larger refinement step length $\xi$, the algorithm can generate an increased amount of training data. As a result, it accelerates the training and validation process for the reward martingales.

\paragraph{The correlation between noises and tightness.}
We also study 
the tightness of the calculated bounds under the same type of noises with different magnitudes.
However, no general statement can be made. For instance, the results of MC under different uniform noises are shown in ~\Cref{fig:tight_noise}. Though the performance of the system decreases with the increase in noise (i.e., the blue line drops with the increase of the noise), the tightness of calculated reward martingales does not show a significant trend of change, i.e., the gaps between the upper bounds (the orange line) and the lower bounds (the green line) seem similar.
It is worth further investigating what factors affect the tightness to produce tighter bounds.

\begin{figure}[h!]
	\footnotesize 
		\begin{subfigure}[b]{0.15\textwidth}
		\includegraphics[width=\textwidth]{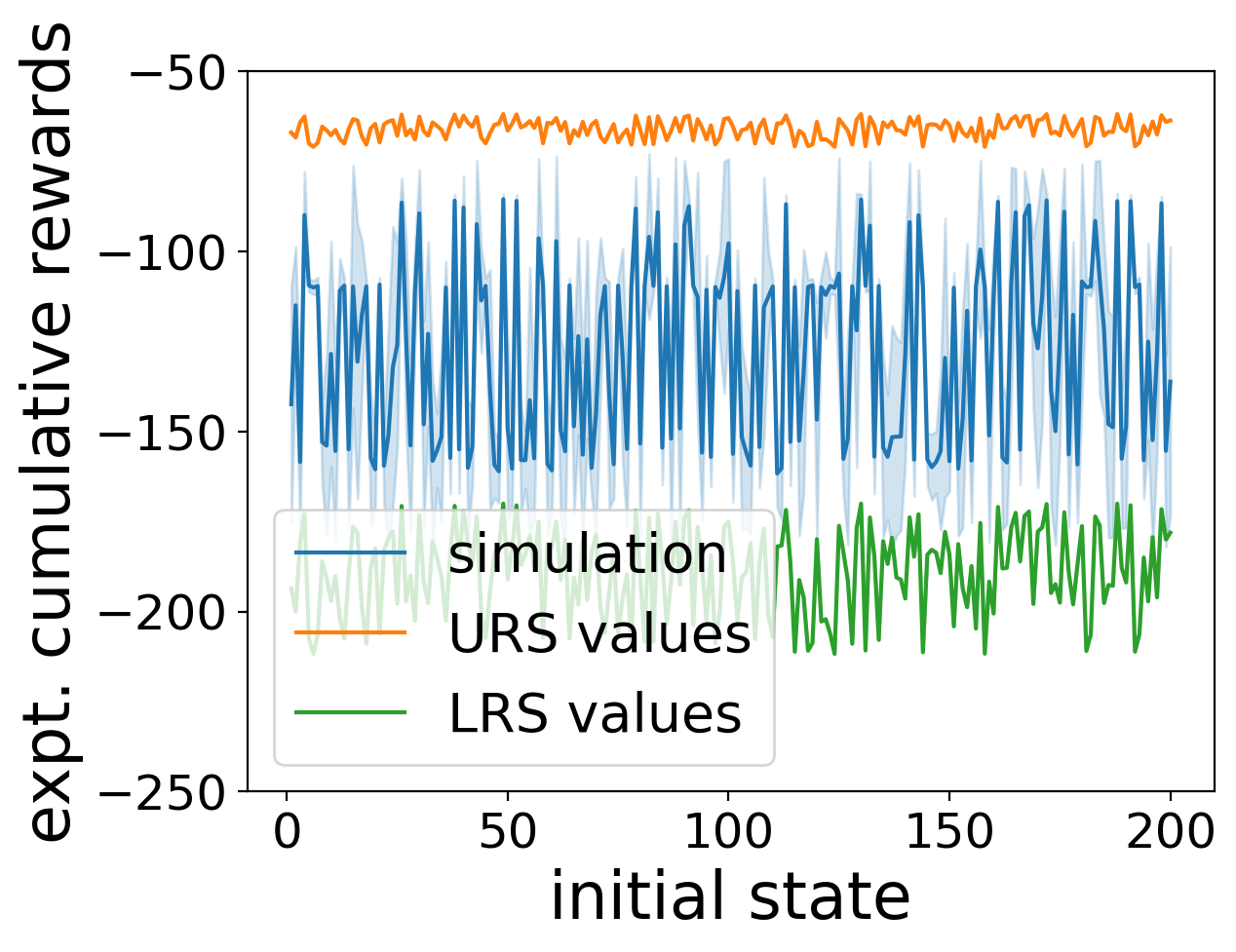}
		\caption{$r=0.06$}
	\end{subfigure}\!
            \begin{subfigure}[b]{0.15\textwidth}
		\includegraphics[width=\textwidth]{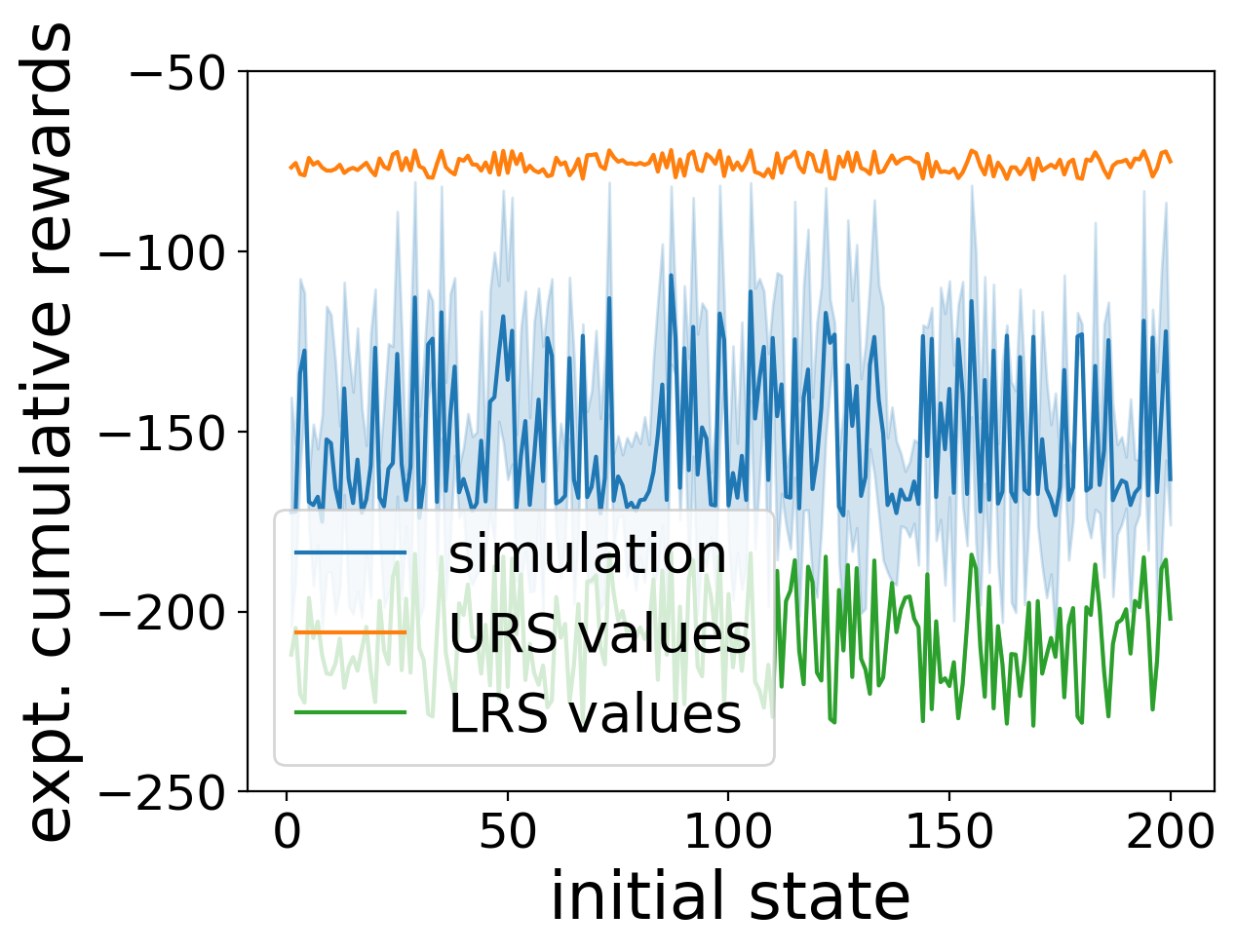}
		\caption{$r=0.08$}
	\end{subfigure}\!
	\begin{subfigure}[b]{0.15\textwidth}
		\includegraphics[width=\textwidth]{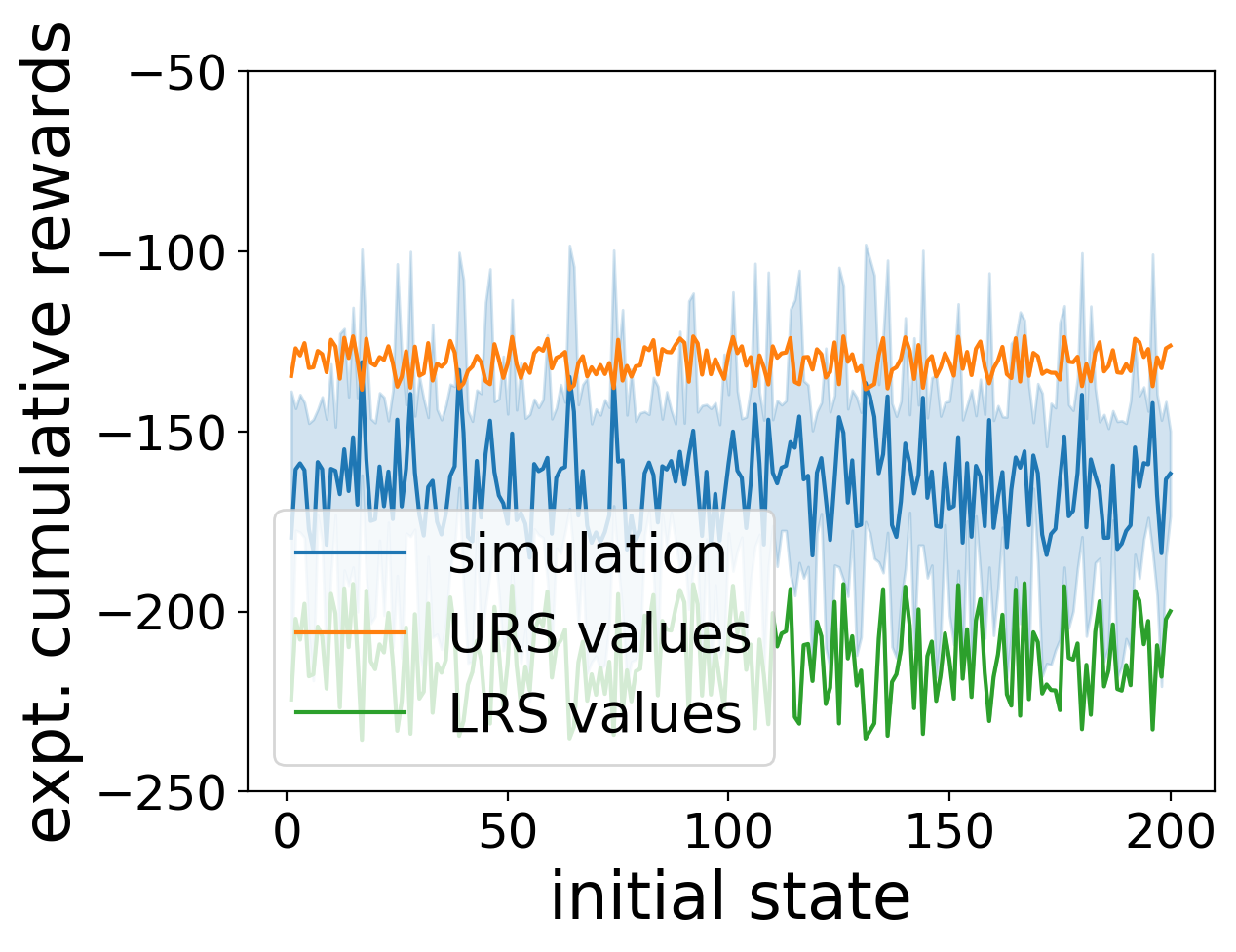}
		\caption{$r=0.1$}
	\end{subfigure}
	\caption{Calculated bounds and simulation results for MC trained on abstract states under different uniform noises.}
	\label{fig:tight_noise}
\end{figure}

\subsection{Additional Experimental  Results}

\begin{figure}[h!]
	\footnotesize 
		\begin{subfigure}[b]{0.23\textwidth}
		\includegraphics[width=\textwidth]{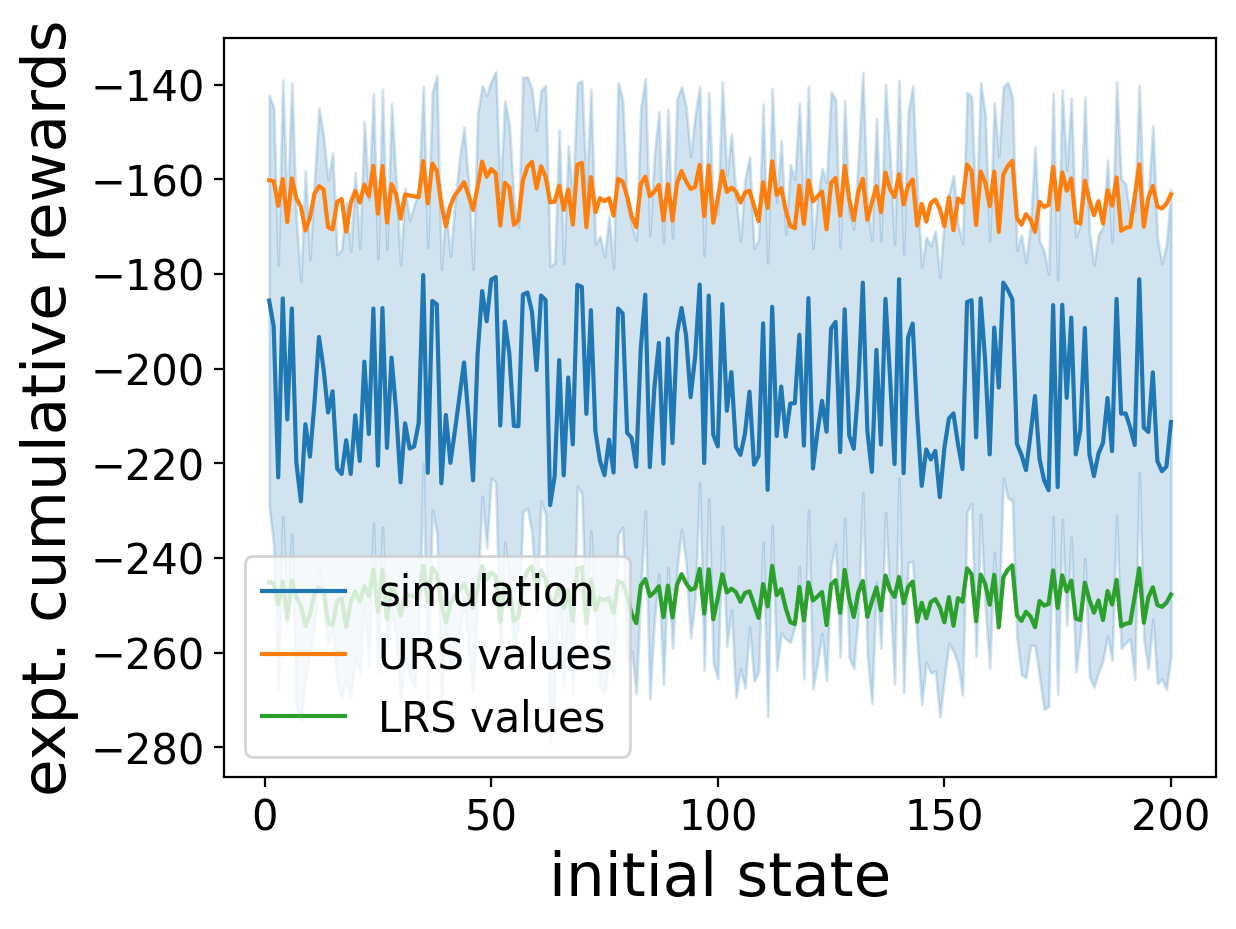}
		\caption{$\sigma=0.08$(A.S.)}
	\end{subfigure}\!
	\begin{subfigure}[b]{0.23\textwidth}
		\includegraphics[width=\textwidth]{imgs/MC_g_0.08_ok_DNN.png}
		\caption{$\sigma=0.08$(C.S.)}
	\end{subfigure}\\
        \begin{subfigure}[b]{0.23\textwidth}
		\includegraphics[width=\textwidth]{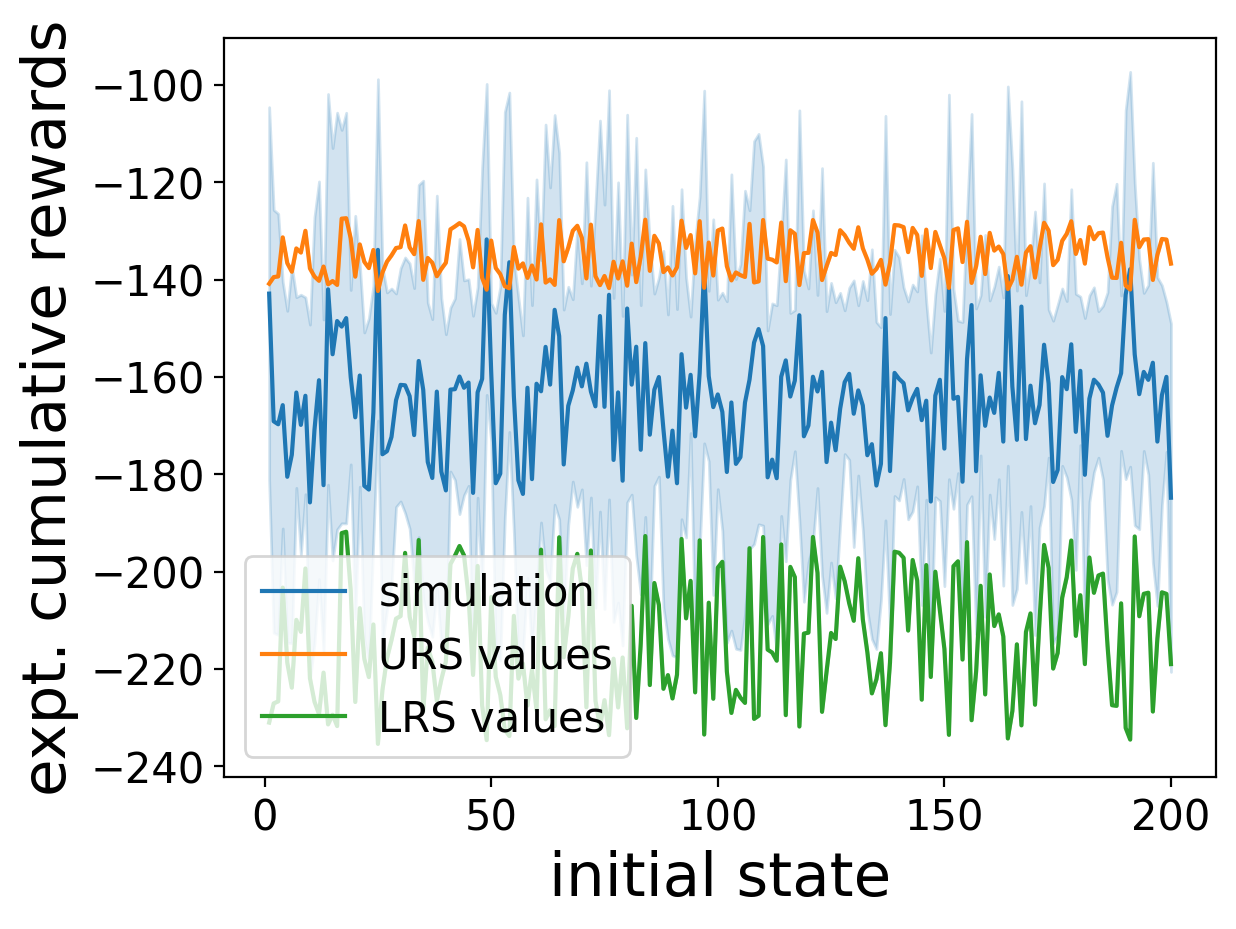}
		\caption{$r=0.1$(A.S.)}
	\end{subfigure}\!
	\begin{subfigure}[b]{0.23\textwidth}
		\includegraphics[width=\textwidth]{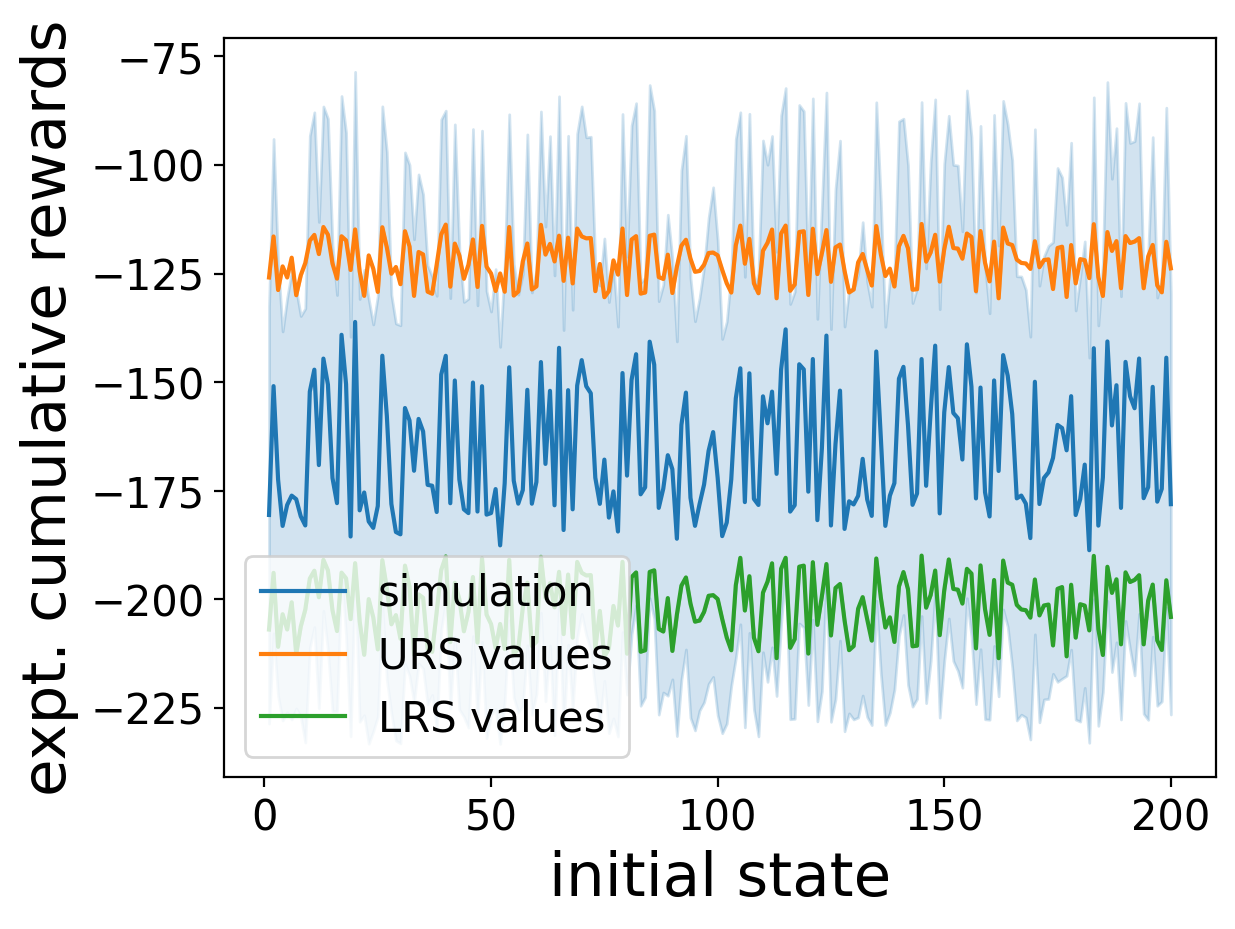}
		\caption{$r=0.1$(C.S.)}
	\end{subfigure}\\

        \begin{subfigure}[b]{0.23\textwidth}
		\includegraphics[width=\textwidth]{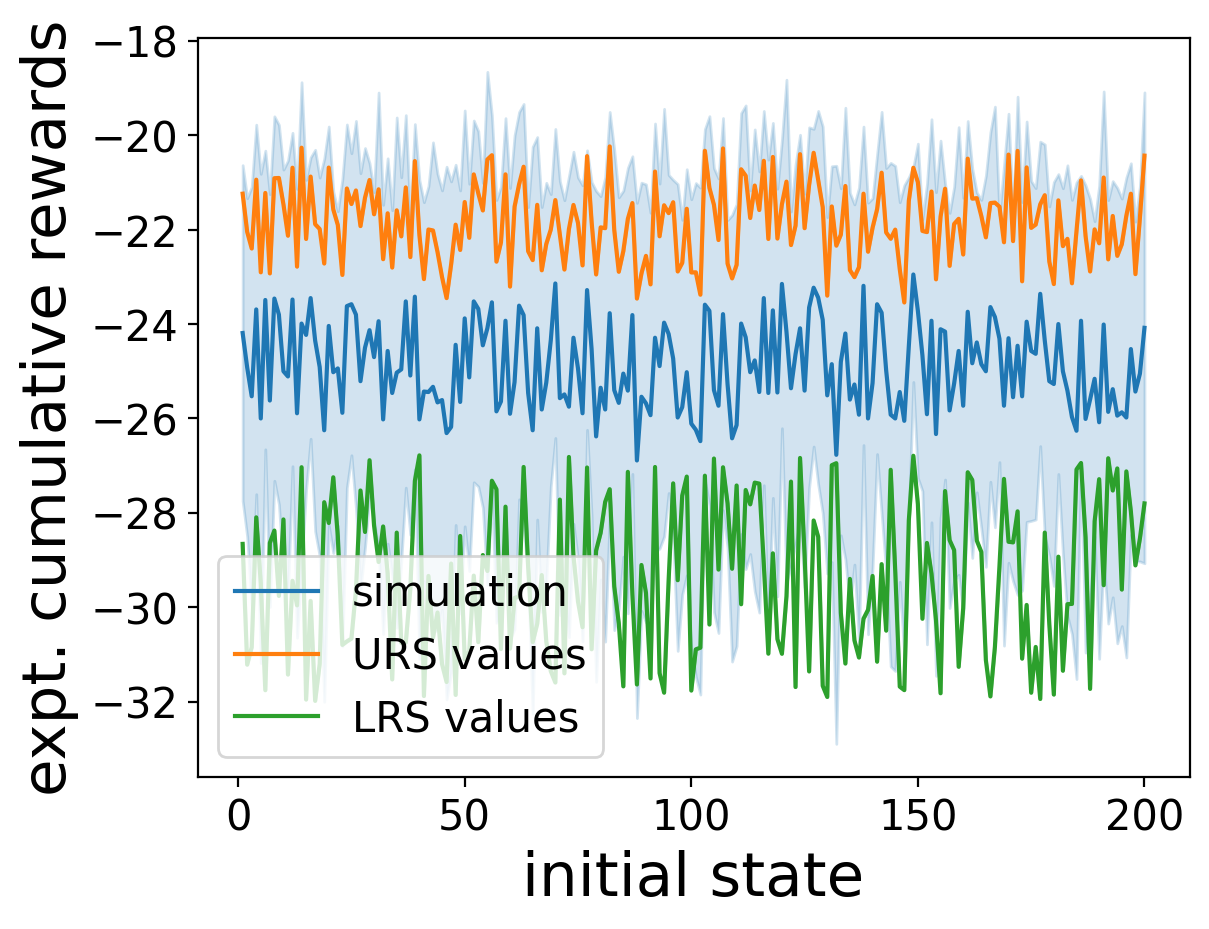}
		\caption{$\sigma=0.05$(A.S.)}
	\end{subfigure}\!
	\begin{subfigure}[b]{0.23\textwidth}
		\includegraphics[width=\textwidth]{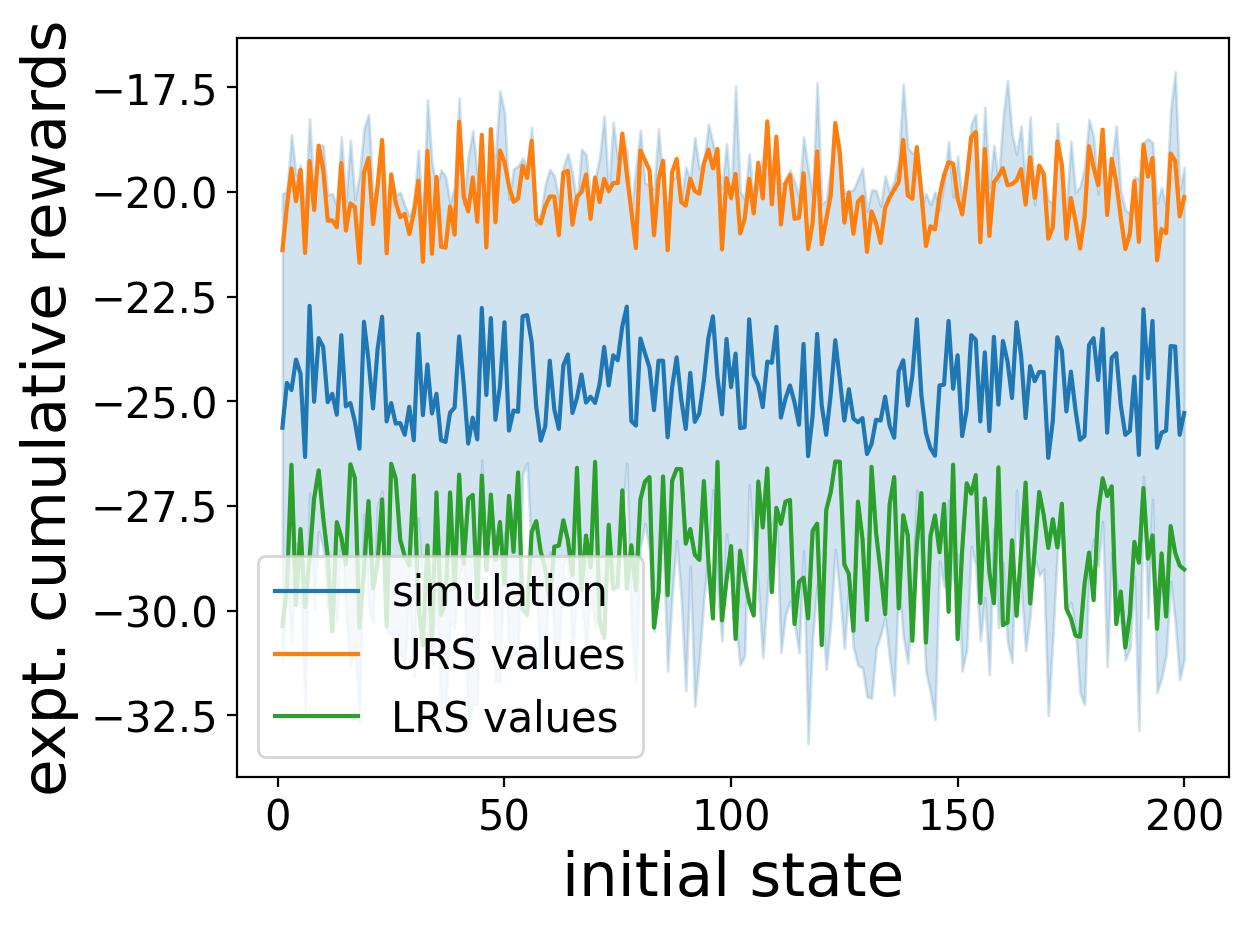}
		\caption{$\sigma=0.05$(C.S.)}
	\end{subfigure}\\
        \begin{subfigure}[b]{0.23\textwidth}
		\includegraphics[width=\textwidth]{imgs/B2_u_0.1_ok.png}
		\caption{$r=0.1$(A.S.)}
	\end{subfigure}\!
	\begin{subfigure}[b]{0.23\textwidth}
		\includegraphics[width=\textwidth]{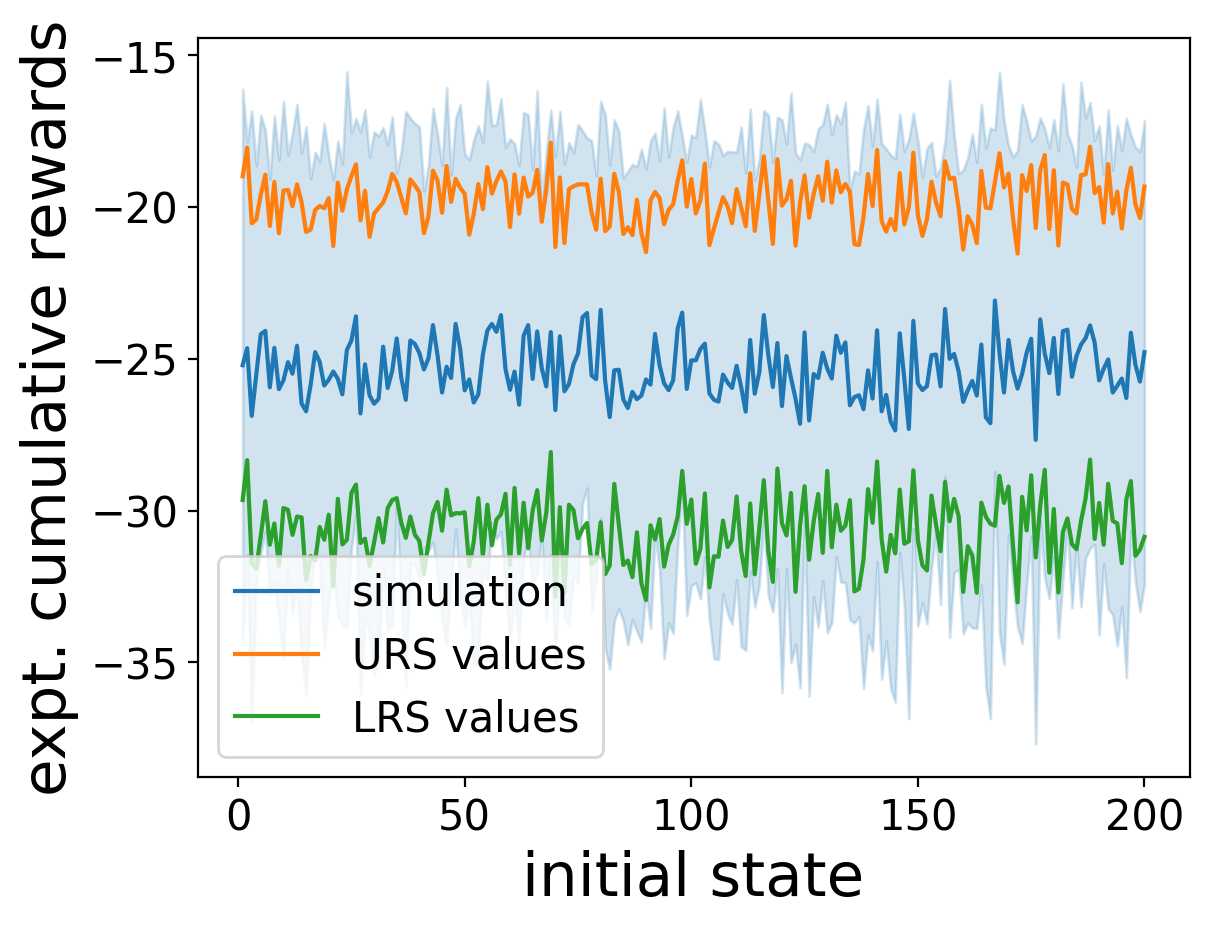}
		\caption{$r=0.1$(C.S.)}
	\end{subfigure}

	\caption{ Comparison between certified bounds and simulation results under different perturbations and policies for MountainCar (a-d) and B2 (e-h).}
	\label{fig:MC_B2_RM}
\end{figure}

\Cref{fig:MC_B2_RM,fig:MC_Tail_Bound_Gaus,fig:B2_Tail_Bound_Gaus} 
show the certified bounds of expected cumulative reward,  tail bounds of cumulative reward, and  the corresponding simulation results of MC and B2 under different perturbations and policies, respectively. 
Moreover, the comparison between the tail bounds of cumulative reward and  the corresponding simulation results for CP and B1 under different Gaussian noises is shown in ~\Cref{fig:CP_B1_Gau_Tail_Bound_Gaus}.
We can draw the same conclusions as those of CP and B1 in the main pages,  that is,  our calculated bounds tightly enclose the simulation and statistical outcomes, which demonstrates the effectiveness of our trained reward martingales.

\begin{figure}[h!]
	\footnotesize 

	\begin{subfigure}[b]{0.23\textwidth}
		\includegraphics[width=\textwidth]{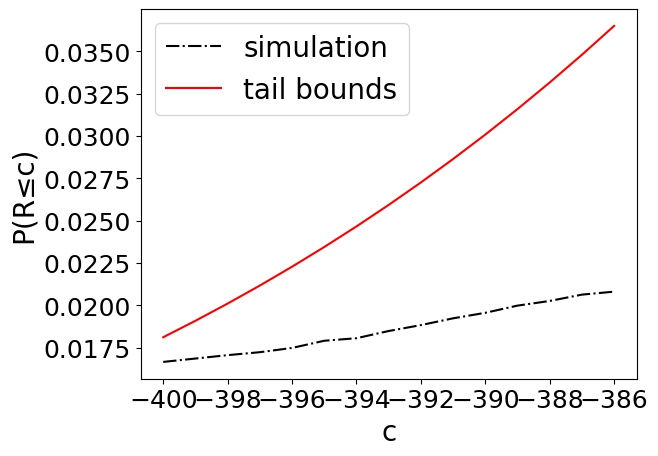}
		\caption{$\sigma=0.08$(A.S.)}
	\end{subfigure}\!
 		\begin{subfigure}[b]{0.23\textwidth}
		\includegraphics[width=\textwidth]{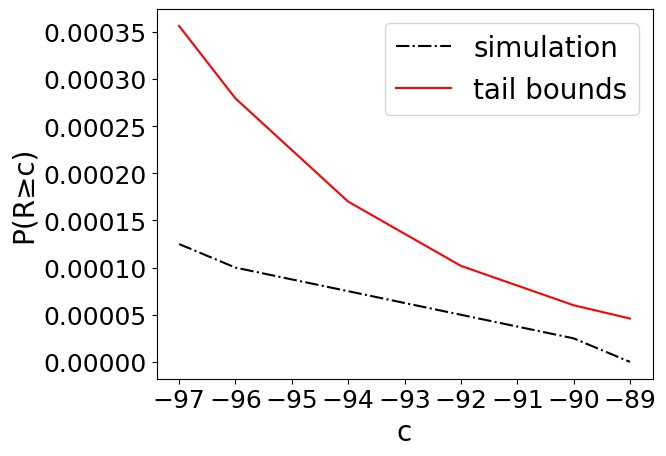}
		\caption{$\sigma=0.08$(A.S.)}
	\end{subfigure} \\

 \begin{subfigure}[b]{0.23\textwidth}
		\includegraphics[width=\textwidth]{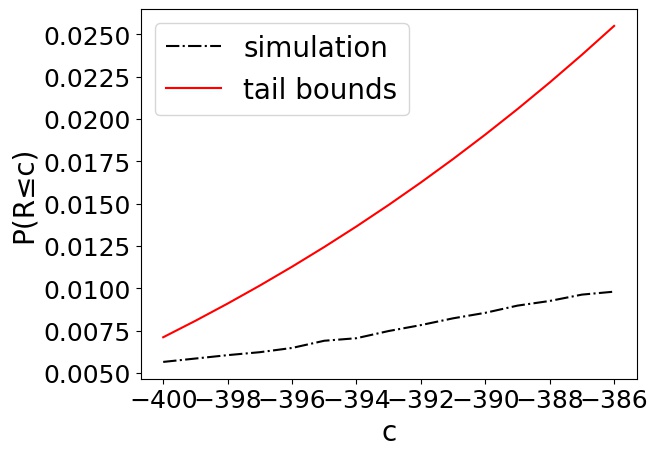}
		\caption{$\sigma=0.08$(C.S.)}
	\end{subfigure}\!
 		\begin{subfigure}[b]{0.23\textwidth}
		\includegraphics[width=\textwidth]{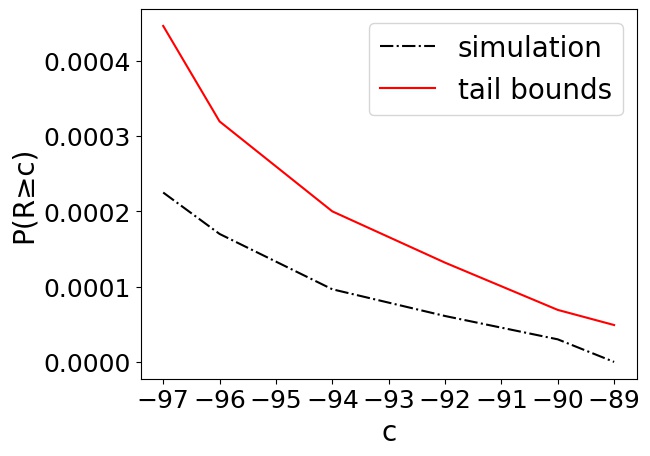}
		\caption{$\sigma=0.08$(C.S.)}
	\end{subfigure} \\
 \begin{subfigure}[b]{0.23\textwidth}
		\includegraphics[width=\textwidth]{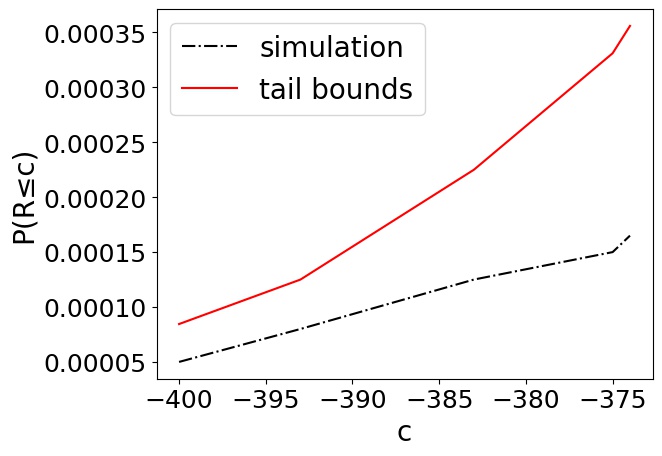}
		\caption{$r=0.1$(A.S.)}
	\end{subfigure}\!
 		\begin{subfigure}[b]{0.23\textwidth}
		\includegraphics[width=\textwidth]{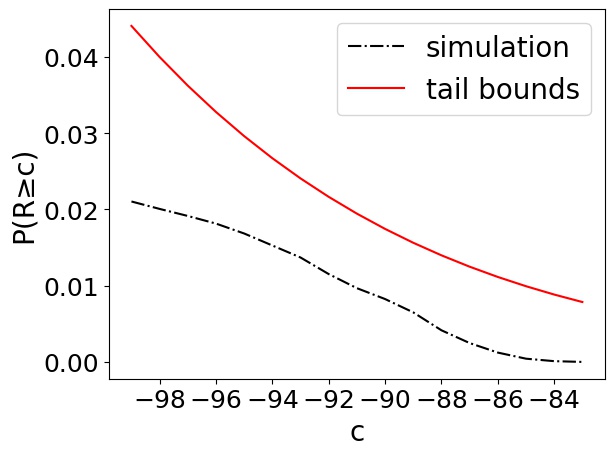}
		\caption{$r=0.1$(A.S.)}
	\end{subfigure} \\
        \begin{subfigure}[b]{0.23\textwidth}
		\includegraphics[width=\textwidth]{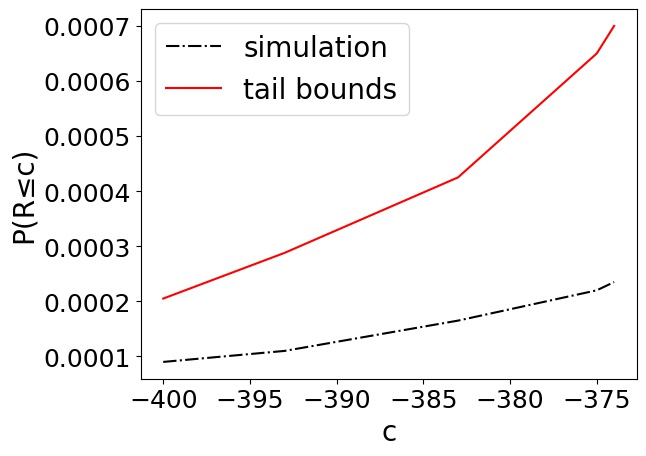}
		\caption{$r=0.1$(C.S.)}
	\end{subfigure}\!
	\begin{subfigure}[b]{0.23\textwidth}
		\includegraphics[width=\textwidth]{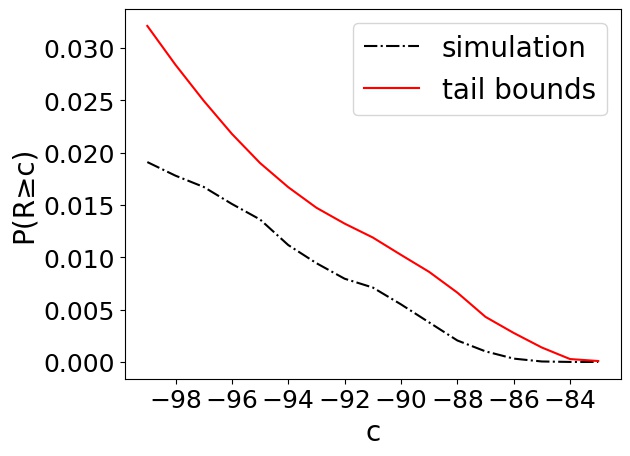}
		\caption{$r=0.1$(C.S.)}
	\end{subfigure}
	\caption{Certified tail bounds of expected cumulative rewards and simulation results under different perturbations for MountainCar.}
	\label{fig:MC_Tail_Bound_Gaus}
\end{figure}

\begin{figure}[h!]
	\footnotesize 

	\begin{subfigure}[b]{0.23\textwidth}
		\includegraphics[width=\textwidth]{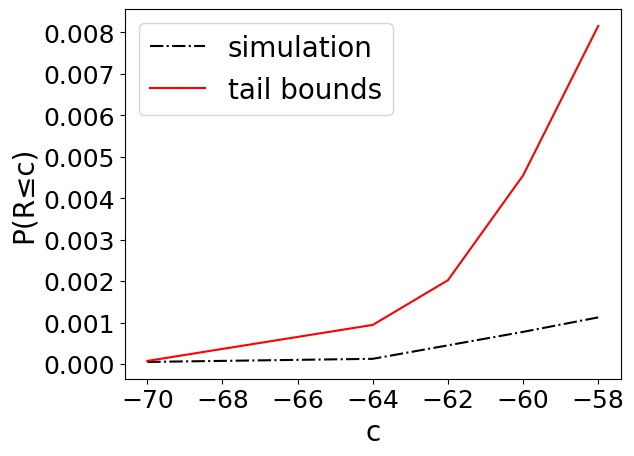}
		\caption{$\sigma=0.05$(A.S.)}
	\end{subfigure}\!
 		\begin{subfigure}[b]{0.23\textwidth}
		\includegraphics[width=\textwidth]{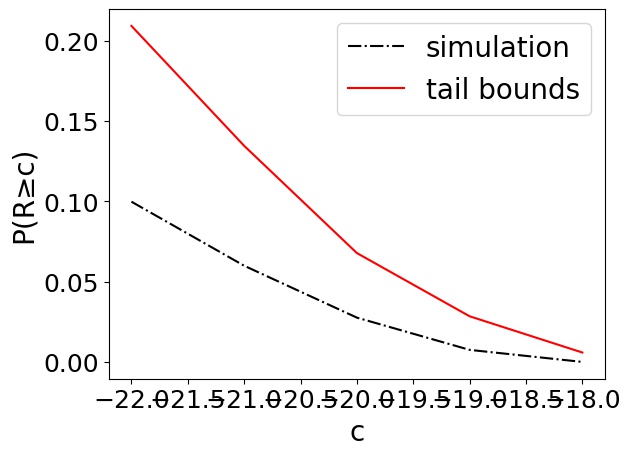}
		\caption{$\sigma=0.05$(A.S.)}
	\end{subfigure} \\

 \begin{subfigure}[b]{0.23\textwidth}
		\includegraphics[width=\textwidth]{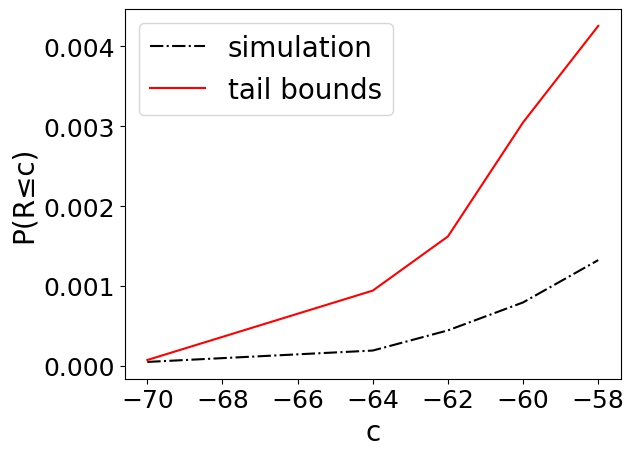}
		\caption{$\sigma=0.05$(C.S.)}
	\end{subfigure}\!
 		\begin{subfigure}[b]{0.23\textwidth}
		\includegraphics[width=\textwidth]{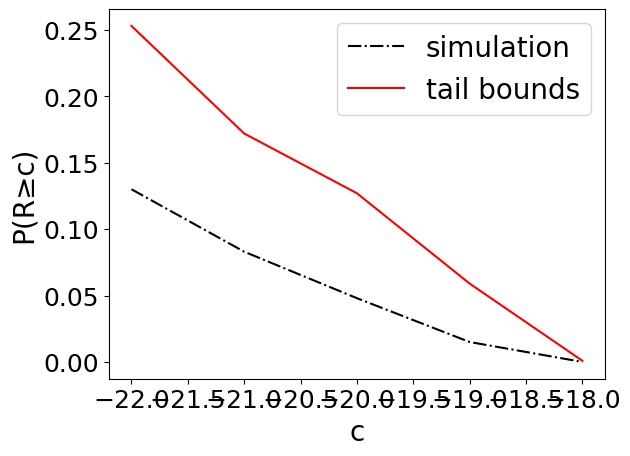}
		\caption{$\sigma=0.05$(C.S.)}
	\end{subfigure} \\
 \begin{subfigure}[b]{0.23\textwidth}
		\includegraphics[width=\textwidth]{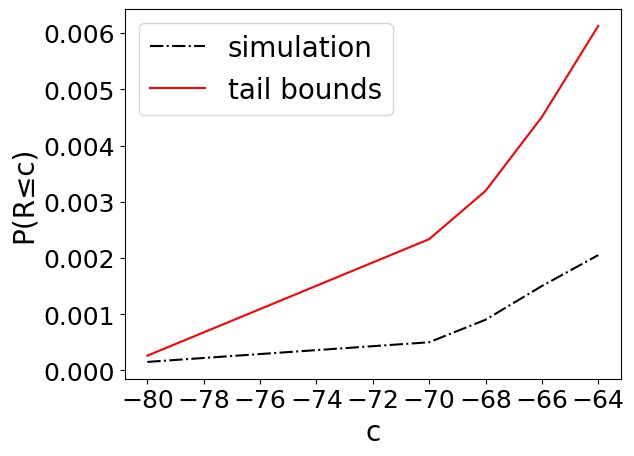}
		\caption{$r=0.1$(A.S.)}
	\end{subfigure}\!
 		\begin{subfigure}[b]{0.23\textwidth}
		\includegraphics[width=\textwidth]{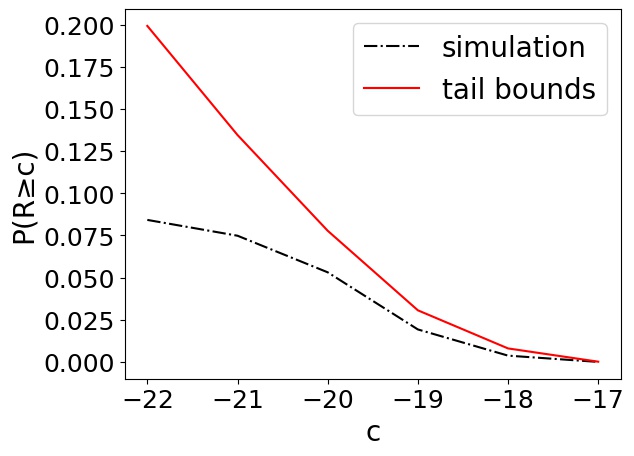}
		\caption{$r=0.1$(A.S.)}
	\end{subfigure} \\
        \begin{subfigure}[b]{0.23\textwidth}
		\includegraphics[width=\textwidth]{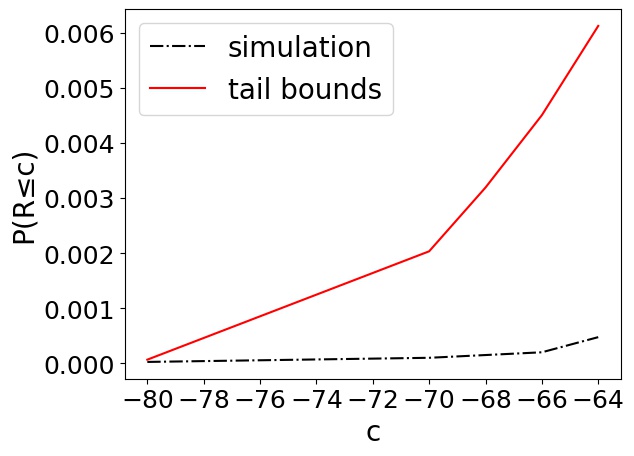}
		\caption{$r=0.1$(C.S.)}
	\end{subfigure}\!
	\begin{subfigure}[b]{0.23\textwidth}
		\includegraphics[width=\textwidth]{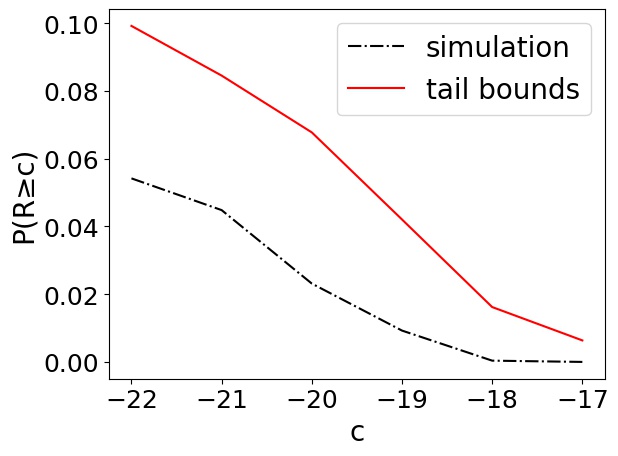}
		\caption{$r=0.1$(C.S.)}
	\end{subfigure}
	\caption{Certified tail bounds of expected cumulative rewards and simulation results under different perturbations for B2.}
	\label{fig:B2_Tail_Bound_Gaus}
\end{figure}

\begin{figure}[h!]
	\footnotesize 

	\begin{subfigure}[b]{0.23\textwidth}
		\includegraphics[width=\textwidth]{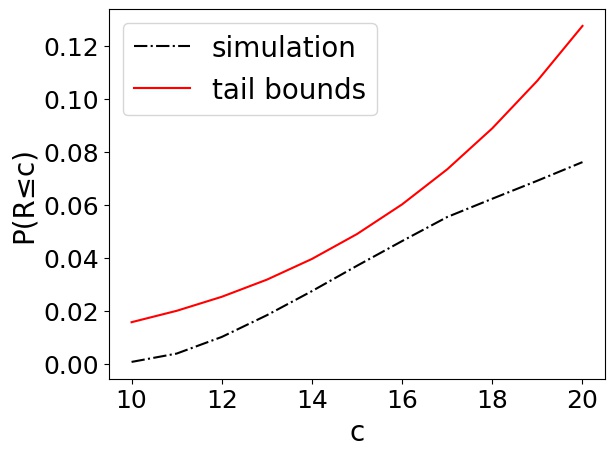}
		\caption{$\sigma=0.1$(A.S.)}
	\end{subfigure}\!
 		\begin{subfigure}[b]{0.23\textwidth}
		\includegraphics[width=\textwidth]{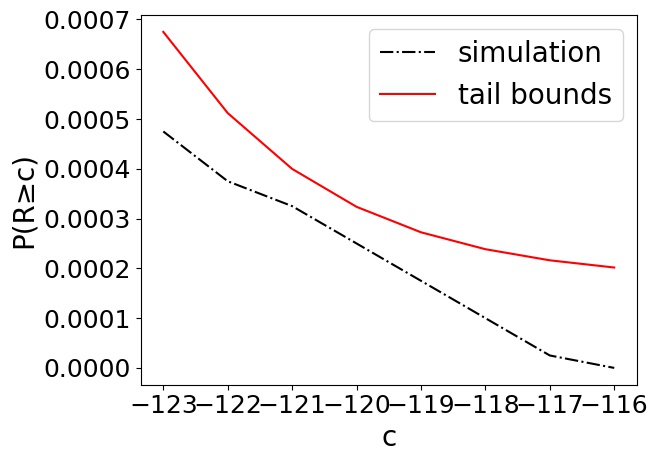}
		\caption{$\sigma=0.1$(A.S.)}
	\end{subfigure} \\

 \begin{subfigure}[b]{0.23\textwidth}
		\includegraphics[width=\textwidth]{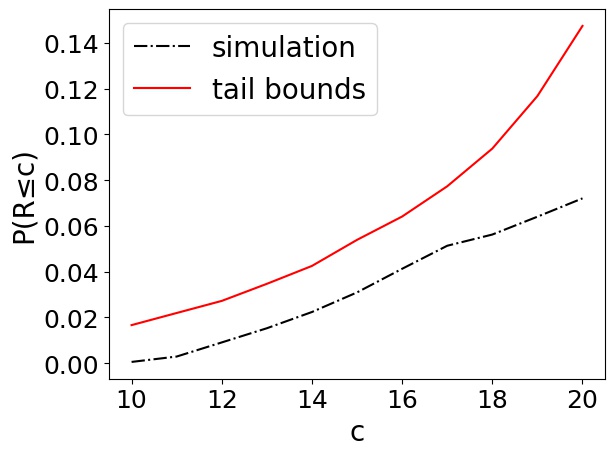}
		\caption{$\sigma=0.1$(C.S.)}
	\end{subfigure}\!
 		\begin{subfigure}[b]{0.23\textwidth}
		\includegraphics[width=\textwidth]{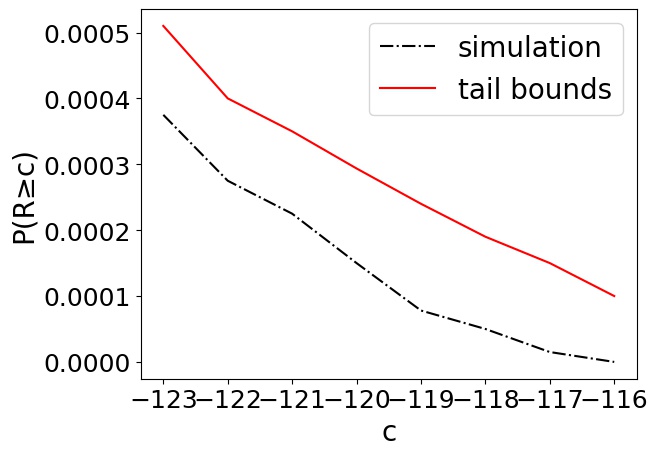}
		\caption{$\sigma=0.1$(C.S.)}
	\end{subfigure} \\
 \begin{subfigure}[b]{0.23\textwidth}
		\includegraphics[width=\textwidth]{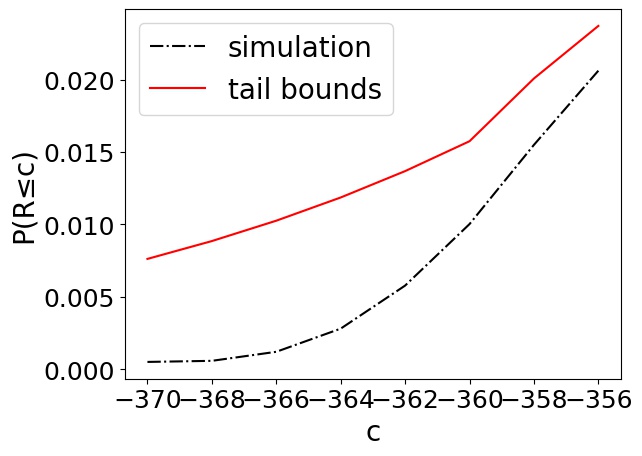}
		\caption{$r=0.3$(A.S.)}
	\end{subfigure}\!
 		\begin{subfigure}[b]{0.23\textwidth}
		\includegraphics[width=\textwidth]{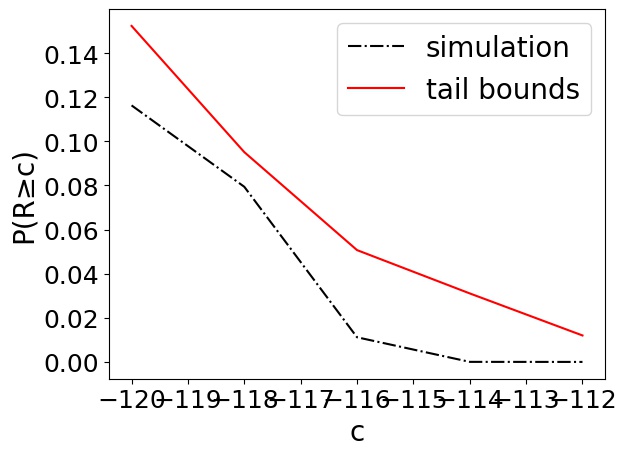}
		\caption{$r=0.3$(A.S.)}
	\end{subfigure} \\
        \begin{subfigure}[b]{0.23\textwidth}
		\includegraphics[width=\textwidth]{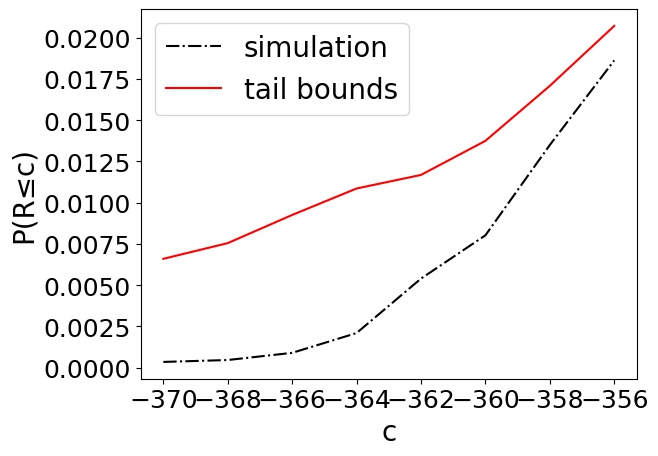}
		\caption{$r=0.3$(C.S.)}
	\end{subfigure}\!
	\begin{subfigure}[b]{0.23\textwidth}
		\includegraphics[width=\textwidth]{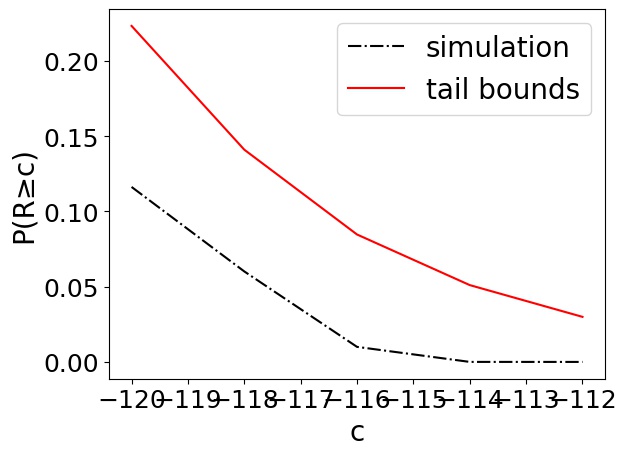}
		\caption{$r=0.3$(C.S.)}
	\end{subfigure}
	\caption{Certified tail bounds of expected cumulative rewards and simulation results under Gaussian perturbations for CP(a-d) and B1(e-h).}
	\label{fig:CP_B1_Gau_Tail_Bound_Gaus}
\end{figure}

\end{document}